\documentclass[12pt,reqno]{amsart}
\usepackage{tabularray}

\usepackage{amsthm,amsmath,amssymb,amsfonts}
\usepackage{natbib}
\usepackage{txfonts}
\usepackage{fullpage}

\usepackage{float}
\usepackage{graphicx}
\usepackage{graphics}
\usepackage{ascmac}
\usepackage{bm}
\usepackage{accents}
\usepackage[colorlinks=true, allcolors=blue]{hyperref}
\usepackage{url}
\usepackage{algorithm}
\usepackage{algorithmic}
\usepackage{enumitem}

\usepackage{fullpage}

\usepackage[utf8]{inputenc} % allow utf-8 input
\usepackage[T1]{fontenc}    % use 8-bit T1 fonts
\usepackage{hyperref}       % hyperlinks
\usepackage{url}            % simple URL typesetting
\usepackage{booktabs}       % professional-quality tables
\usepackage{amsfonts}       % blackboard math symbols
\usepackage{nicefrac}       % compact symbols for 1/2, etc.
\usepackage{microtype}      % microtypography
\usepackage{xcolor}         % colors
\usepackage{algorithm}
\usepackage{algorithmic}

\usepackage{amsmath,amssymb,amsthm,mathtools}
\usepackage{hyperref}
\usepackage{enumitem}
\usepackage{microtype}
\usepackage{natbib}
\usepackage{color}
\usepackage{longtable}          % multi-page per-problem tables in Appendix~\ref{app:llm}
\usepackage{multirow}           % Model/Dataset row spans in Table~\ref{tab:llm-main-case-a}
\usepackage{pifont}             % \ding glyphs for \cmark / \xmark (per-problem tables)
\usepackage{tabularx}           % URL column wrapping in Table~\ref{tab:links}
\newcommand{\cmark}{\ding{51}}  % mode = gold in per-problem tables
\newcommand{\xmark}{\ding{55}}  % mode != gold

\newcommand{\cA}{\mathcal{A}}
\newcommand{\cB}{\mathcal{B}}
\newcommand{\cE}{\mathcal{E}}
\newcommand{\cF}{\mathcal{F}}
\newcommand{\cH}{\mathcal{H}}
\newcommand{\cP}{\mathcal{P}}
\newcommand{\cQ}{\mathcal{Q}}
\newcommand{\cT}{\mathcal{T}}

\newcommand{\PP}{\mathbb{P}}
\newcommand{\EE}{\mathbb{E}}
\newcommand{\ind}{\mathbf{1}}

\newtheorem{theorem}{Theorem}[section]
\newtheorem{proposition}[theorem]{Proposition}
\newtheorem{lemma}[theorem]{Lemma}
\newtheorem{corollary}[theorem]{Corollary}
\newtheorem{remark}[theorem]{Remark}
\newtheorem{definition}[theorem]{Definition}
\newtheorem{assumption}[theorem]{Assumption}
\newtheorem{condition}[theorem]{Condition}

\usepackage{tikz}
\usetikzlibrary{arrows.meta,positioning,calc,decorations.pathreplacing}

\definecolor{cTarget}{RGB}{36,72,128}      % deep navy outer bar
\definecolor{cTargetIn}{RGB}{20,46,86}     % darker inner LCB block
\definecolor{cRunner}{RGB}{214,143,60}     % muted amber runner-up
\definecolor{cObs}{RGB}{180,180,180}       % warm-gray observed others
\definecolor{cObsLine}{RGB}{130,130,130}
\definecolor{cUnseen}{RGB}{120,120,120}    % dashed gray unseen
\definecolor{cThresh}{RGB}{197,76,90}      % refined coral for U_t / threshold
\definecolor{cAxis}{RGB}{110,110,110}
\definecolor{cInk}{RGB}{38,38,42}
\definecolor{cMute}{RGB}{105,105,112}

\usepackage{color}

\title{CITE: Anytime-Valid Statistical Inference in LLM Self-Consistency
}

    \author[H. Ota]{Hirofumi Ota}
    \address[H. Ota]{Komaba Institute for Science, Graduate School of Arts and Sciences, The University of Tokyo.}
    \email{hirofumi-ota@g.ecc.u-tokyo.ac.jp}
    \author[N. Iwase]{Naoto Iwase}
    \address[N. Iwase]{Nagoya University}
    \email{iwase.naoto.h6@s.mail.nagoya-u.ac.jp}
\author[Y. Ichihara]{Yuki Ichihara}
    \address[Y. Ichihara]{NARA Institute of Science and Technology (NAIST) / Mohamed bin Zayed University of Artificial Intelligence (MBZUAI)}
    \email{yuki.ichihara@mbzuai.ac.ae}
    \author[J. Komiyama]{Junpei Komiyama}
    \address[J. Komiyama]{Mohamed bin Zayed University of Artificial Intelligence (MBZUAI) / RIKEN AIP}
    \email{}
    \author[M. Imaizumi]{Masaaki Imaizumi}
    \address[M. Imaizumi]{Graduate School of Arts and Sciences, The University of Tokyo / RIKEN AIP / Kyoto University}
        \email{imaizumi@g.ecc.u-tokyo.ac.jp}
	\date{\today\  (first version)}

\allowdisplaybreaks

\setlist{itemsep=2pt,topsep=2pt,partopsep=0pt,parsep=0pt}

\allowdisplaybreaks

\begin{document}

\maketitle

\begin{abstract}
    Large language models often improve reasoning by sampling multiple outputs and aggregating their final answers, but precise and efficient control of error levels remains a challenging task. In particular, deciding when to stop sampling remains difficult when the stopping rule is data-dependent and the set of possible response labels is not known in advance. We study anytime-valid certification of a prespecified target answer as the unique mode of the model’s response distribution, a guarantee distinct from answer correctness. We propose the Certification by Intersection-union Testing with E-processes (CITE) algorithm, which provably controls false certification at any prescribed level under arbitrary data-driven stopping, without requiring prior knowledge of the answer category set. We also prove a category-set-size-free stopping-time rate, establish matching minimax lower bounds up to constants in the main regime, and extend the construction to confidence-weighted voting. Simulations and LLM self-consistency experiments show empirical error control and improved certification in diffuse-tail settings.
\end{abstract}

\section{Introduction}

\subsection{Background and Motivation}
Large language models (LLMs) have demonstrated increasingly strong performance on a wide range of  tasks \cite{brown2020language,chowdhery2023palm,achiam2023gpt,touvron2023llama}, especially from the use of \emph{test-time compute}. Prompting methods such as chain-of-thought and least-to-most decomposition encourage models to expose intermediate reasoning steps, while inference-time strategies such as self-consistency and tree-of-thought style search improve robustness by sampling, exploring, and aggregating multiple candidate reasoning trajectories \citep{wei2022chain,zhou2023least,wang2023selfconsistency,yao2023tree,brown2025monkeys,snell2025scaling}.

A notion of \textit{adaptive self-consistency} has gained attention to reduce computational cost of sampling growing approximately linearly with the number of sampled trajectories.
A seminal approach is \textit{majority voting} \citep{wang2023selfconsistency}, which is simple, model-agnostic, and does not require a separate verifier, reward model, or additional training. 
In particular, the method adapts to the case that 
easy questions often stabilize after only a few samples, whereas hard questions remain ambiguous and may benefit from continued sampling. As extensions, \cite{aggarwal2023let} studies a stopping method based on posterior agreement among sampled answers; \cite{li2024escape} uses low-entropy windows as a practical stopping signal; \cite{wang2025penny} and \cite{wang2025dynscaling} allocate budgets using difficulty estimates or bandit-style uncertainty across sets of queries; and confidence-aware variants use response-level confidence to improve aggregation and stopping \cite{taubenfeld2025confidence,aghazadeh2025cges}.

One challenge of the literature is to handle the situation where we must precisely control error level $\varepsilon$ of certifying a modal response, those that arise when the risk of failure is high.

A typical real example is the emergency-room triage on data such as ER-REASON \citep{mehandru2025erreason}, where each LLM-output acuity score determines the time-to-physician for an arriving patient and each disposition determines admission-versus-discharge.
In this case, if a modal-response certificate is used operationally, the false-certification probability must be bounded at a prescribed level $\varepsilon$. 
At present, it is difficult to control the certification errors with precision; therefore, keeping errors below a certain threshold requires performing an excessive number of inferences, which directly leads to a massive increase in computational costs.

Based on the situation, we have the following question: \textit{can we derive a data-driven stopping rule for adaptive self-consistency that controls the error level of certifying a modal response by LLMs?}

\begin{figure}
    \centering
    \caption{Overview of the stopping rule by CITE. With the countably infinite set of categories $\cA$,  CITE tests the hypothesis of whether the target category $r \in \cA$ is a mode or not.}
    \includegraphics[width=0.8\linewidth]{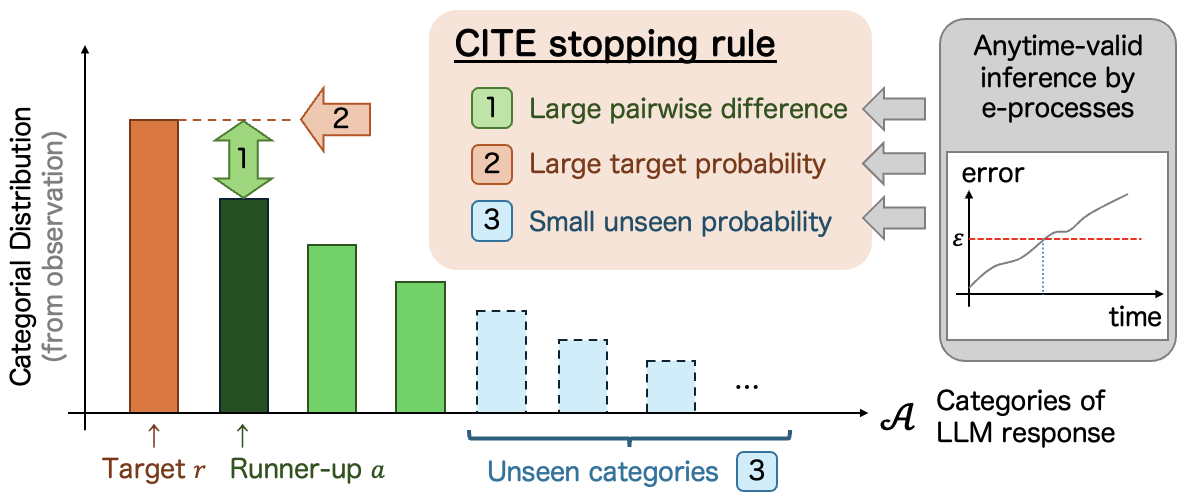}
    \label{fig:graph_abst}
    \vskip -0.15in
\end{figure}

\subsection{Our Study}

In this study, we propose the \emph{Certification by Intersection-union Testing with E-processes} (CITE), which is an \textit{anytime-valid} and \textit{category-agnostic} sequential test for controlling the Type-I error of falsely certifying a prespecified response label as the unique mode.
We note the important features: 
(i) \textit{Anytime validity}: CITE works with data-driven stopping time and the error guarantee must hold uniformly over the  time, not at a single prespecified sample size.
(ii) \textit{Category-agnostic}, CITE works even when the response category is unknown a priori and grows as sampling proceeds, which can adapt to new responses from LLMs. CITE gives a stopping time based on the three conditions on empirical probabilities: (i) the target is larger than the runner-up, (ii) the target itself is large, and (iii) the unobserved category is small, using anytime-valid e-process machinery
\citep{ville1939,grunwald2024safe,ramdas2025hypothesis}. Figure \ref{fig:graph_abst} gives a graphical abstract.

Our contributions are as follows.
\begin{itemize}[leftmargin=1.5em,itemsep=2pt,topsep=2pt]
\item \textbf{Anytime-valid category-agnostic error guarantee}: 
 CITE controls Type-I error of certifying a label as the unique mode uniformly over all
  stopping times at any prescribed level $\varepsilon$
  (Theorem~\ref{thm:main-iut}).
  The error guarantee is valid without the knowledge of answer categories.
\item \textbf{Category-size-free stopping-time}:
  We show that the stopping time by CITE has an order which is free from category size (Theorem~\ref{thm:power-iut-clean}), which is 

  contrast to several baseline methods.
We also derive a matching
minimax lower bound
(Theorem~\ref{thm:minimax-combined-lower-bound}). 
\item \textbf{Extension to weighted case}:
  We extend CITE to handle a case with weighted observations and gaps. We also derive an anytime-valid Type-I control and yield the same
  rate of the stopping time (Section~\ref{sec:weighted}).
  \item \textbf{Empirical performance}: Simulations and LLM experiments
  (Section~\ref{sec:experiments}) confirm that CITE controls
  Type-I error empirically, matches or improves on baselines on three LLM benchmarks. 

\end{itemize}

\section{Problem Formulation}
\label{sec:problem}

\subsection{Setup and Observation}

Let $\cA$ be a countable (possibly infinite) set of response
categories.  We consider an
unknown categorical distribution $P$ on $\cA$, and write
$p_a := P(X = a)$ for its probability mass function for $a \in \cA$.  Let
$p_{(1)} \geq p_{(2)} \geq \cdots$ denote the decreasing rearrangement
of $(p_a)_{a \in \cA}$, and let
$\cF_t := \sigma(X_1, \ldots, X_t)$ be the natural filtration.

Suppose that we observe an i.i.d.\ sequence $X_1, X_2, \ldots$ from $P$:
\begin{assumption}[i.i.d.\ categorical data]
\label{asn:dgp}
$(X_t)_{t \geq 1}$ are i.i.d.\ draws from a categorical distribution
$P$ on a countable set $\cA$.
\end{assumption}
We define some notations for statistics from the observations. For each $a \in \cA$ and $t \geq 1$, let
\[
  N_t(a) := \sum_{i=1}^t \ind\{X_i = a\},
  \qquad
  \widehat p_t(a) := \frac{N_t(a)}{t},
  \qquad
  \cA_t := \{a \in \cA : N_t(a) \geq 1\}
\]
denote the count, empirical frequency, and set of observed
categories by time~$t$.

Our definition of $P$ has a distinctive feature: A may contain a countably infinite number of categories, and since $\cA$ is unknown, information about the set of categories cannot be used in decision-making. This setup is effective for representing the complex response patterns of LLMs.

\subsection{Statistical Testing}

We study the fixed-target unique-mode testing problem.
For a target category $r \in \cA$ fixed \emph{a priori}, i.e. independent
of the data, we design a sequential test of the
unique-mode hypothesis that $r$ strictly dominates every other category.

\begin{definition}[Null and alternative hypothesis]
\label{def:exact-mode}
For a fixed target category $r \in \cA$, set
\[
  H_{0,r} := \bigl\{P : p_r \leq \sup_{a \neq r} p_a\bigr\},
  \qquad
  H_{1,r} := \bigl\{P : p_r > \sup_{a \neq r} p_a\bigr\}.
\]
$H_{1,r}$ is the unique-mode alternative, i.e. $r$ strictly dominates
every other category. 
Under any $P\in H_{0,r}$ there exists a deterministic witness $a^\star\neq r$ such that $p_{a^\star}\ge p_r$ holds.
A procedure \emph{certifies} $r$ whenever it
rejects $H_{0,r}$ in favor of $H_{1,r}$.
\end{definition}

To handle this testing, we consider certifying the mode through a sequential procedure with the observations using the following criteria.

{
\begin{definition}[Sequential certification procedure]
A sequential certification procedure for target $r$ is a sequence
$(D_t)_{t\ge1}$ of $\mathcal F_t$-measurable decision rules, where
$D_t=1$ means that, after observing $X_1,\ldots,X_t$, the procedure
certifies $r$ as the unique mode. Its stopping time is
\[
\tau := \inf\{t\ge1: D_t=1\},
\]
with $\inf\emptyset=\infty$. 
The procedure has anytime-valid Type-I error control at level $\varepsilon$ if it holds that 
\begin{align} \label{eq:type1}
\sup_{P\in H_{0,r}} \PP_P(\tau<\infty)\le \varepsilon.    
\end{align}
If $\tau<\infty$, the procedure certifies $r$ as the unique mode.
\end{definition}

}

Importantly, the sampling time $t$ is not fixed in advance, hence The bound~\eqref{eq:type1} is time-uniform. At each time $t$, the
procedure either continues sampling or stops and certifies, using only the data observed so far. 
The Type-I guarantee controls the probability that such a data-dependent rule ever falsely certifies $r$.

\begin{remark}
    We note the originality of our setting. While similar settings have been considered in statistics \citep{good1953population,Painsky02012025} and LLM test-time compute \citep{aggarwal2023let,li2024escape,komiyama2026bestofinfinity}, to the best of our knowledge, an anytime valid sequential testing for the infinite and unknown category set has not been investigated. This setting describes with the nature of LLMs and represents a novel approach in terms of hypothesis testing.
\end{remark}

\section{CITE: Certification by Intersection-union Testing with E-processes}
\label{sec:iut}

We present our proposed method, CITE.
Here, the target $r \in \cA$ and the error level of certification $\varepsilon > 0$ is fixed {before} observing data as formalized in \cite{kim2025locally}.
Our strategy involves using three data-driven stochastic processes to estimate the components of the distribution uniformly over time.
In particular, we consider $\cF_t$-measurable stochastic processes $E_t, L_t$ and $U_t$ with their forms will be derived in the following sections, which satisfies the following conditions under the null hypothesis $H_{0,r}$:
\begin{enumerate}
    \setlength{\parskip}{0cm} % 段落間
    \setlength{\itemsep}{0cm} % 項目間
    \item \textbf{Pairwise evaluation}: For $a \in \cA$, $E_t$ evaluates an relative size of $p_r$ against $p_a$, that is, for any $\alpha \in (0,1)$, $ \sup_{t \geq 1}E_t \leq 1 / \alpha$ holds with probability at least $1- \alpha$.
    \item \textbf{Lower bound of the target probability $p_r$}: the second process $L_t$ bounds $p_r$ below, that is, $\sup_{t \geq 1} L_t \leq p_r$ holds.
    \item \textbf{Upper bound of unseen categories}: the third process $U_t$ is an upper bound of probability of unseen categories, that is,  $\inf_{t \geq 1}  U_t\geq \sup_{a \in \cA : N_t(a) = 0} p_a$ holds.
\end{enumerate}
Note that these stochastic processes are calculated from observed values $X_1,...,X_t$ without the knowledge of $P$.

Using these processes, we define our stopping rule:  if any of the conditions above satisfied by these stochastic processes is violated at time $t$, we reject the null hypothesis $H_{0,r}$ and stop as $\tau = t$.

Rigorously, the stopping rule, with $P \in H_{0,r}$ and $a^\star \neq r$ satisfying
$p_{a^\star} \geq p_r$, provides the following  inclusion:
\begin{equation}\label{eq:type1-decomp}
  \{\tau < \infty\}
   \subseteq 
  \underbrace{\bigl\{\exists t:   E_t \geq 3/\varepsilon\bigr\}}_{\text{(i) pairwise}}
  \cup
  \underbrace{\bigl\{\exists t : L_t(r) > p_r\bigr\}}_{\text{(ii) target probability}}
  \cup
  \underbrace{\bigl\{\exists a, t : N_t(a) = 0,  p_a > U_t\bigr\}}_{\text{(iii) unseen categories}}.
\end{equation}
In the following sections, we will define the stochastic processes $E_t$, $L_t$, and $U_t$ in detail and evaluate these three events.

\subsection{Component (i): Pairwise E-Values}
\label{sec:pairwise}

We give a specific form of the stochastic process $E_t$ by the notion of e-process in the testing-by-betting style
of~\citet{ramdas2023gametheoretic}.
For the fixed target $r \in \cA$ and a competing category $a \neq r$, we consider a centered indicator
$Z_i^{(r,a)} := \ind\{X_i = r\} - \ind\{X_i = a\} \in \{-1,0,+1\}$ with a betting parameter $\lambda \in (0,1)$, and define the
\emph{pairwise e-process} as
\begin{equation}\label{def:pairwise-e-process}
  {E}_t^{(r,a)}(\lambda)
  := \prod_{i=1}^{t}\bigl(1 + \lambda Z_i^{(r,a)}\bigr)
  = (1+\lambda)^{N_t(r)}(1-\lambda)^{N_t(a)}.
\end{equation}
The closed form depends on the data only through $(N_t(r), N_t(a))$.
Under the null hypothesis, i.e. $p_r \leq p_a$, it holds that $\sup_{t \geq 1} E_t^{(r,a)}(\lambda) \leq 1/\alpha$ with probability at least $1-\alpha$; namely, it is a Nonnegative Super-Martingale (NSM) process.

We define the stochastic process $E_t$ as the pairwise mixture e-process against the
empirical runner-up. 
We extend $E_t^{(r,a)}(\lambda)$ by integrating $\lambda$ out.
Since the optimal $\lambda$ is $(p_r-p_a)/(p_r+p_a)$ (Proposition~\ref{prop:pairwise-oracle}) depends on unknown values, 

we therefore mix over a finite grid $\Lambda_{\mathrm{pw}}
\subset (0,1)$ with positive weights $(w_\lambda)$ summing to at most
$1$.
Namely, we define the pairwise mixture e-process as follows:
\begin{definition}[pairwise mixture e-process]    \label{def:mixture-e}
For each competitor $a \neq r$, define
\begin{align}    
  E_t = {E}_t^{(r,\widehat{a}_t)}
  \coloneqq \sum_{\lambda \in \Lambda_{\mathrm{pw}}}
     w_\lambda  E_t^{(r,\widehat{a}_t)}(\lambda), \quad \widehat a_t := \arg\max_{a \in \cA_t \setminus \{r\}} N_t(a).
\end{align}
\end{definition}
While $E_t^{(r,\widehat{a}_t)}$ is not a NSM process, $E_t^{(r,a)}$ is NSM since it is a convex combination of NSM processes:
\begin{corollary}
\label{cor:mixture-nsm}
Under $p_r \leq p_a$, for any $\alpha \in (0,1)$, we have $\sup_{t \geq 1} {E}_t^{(r,a)} \leq 1/\alpha$  with probability at least $1-\alpha$.
\end{corollary}

\subsection{Component (ii): Lower Confidence Bound on the Target Probability $p_r$}
\label{sec:lcb}

We develop a time-uniform lower confidence bound (LCB) on the target probability $p_r$ to handle \emph{hidden competitors}, i.e., categories $h \in \cA$ with $p_h > p_r$ that have never appeared in the sample.

For each candidate value $q \in (0,1]$, we test a hypothesis ``$p_r \leq q$'' by
betting on the centered indicator $\ind\{X_i = r\} - q$, which is an analogous of 
$\ind\{X_i=r\} - \ind\{X_i=a\}$ in Component~(i).  Fix a finite grid
$\Lambda_r \subset (0, \infty)$ with positive weights $(v_\lambda)$
summing to at most $1$, and for $q \in (0, 1]$, we consider
$
  \Lambda_r(q) := \{\lambda \in \Lambda_r : \lambda < 1/q\},
$
thus each factor $1 + \lambda(\ind\{X_i=r\} - q)$ is positive.  
Then, we define
the mixture e-process
\begin{equation}\label{def:lcb}
  {M}_t(q)
  := \sum_{\lambda \in \Lambda_r(q)}
     v_\lambda \prod_{i=1}^{t}
     \bigl(1 + \lambda(\ind\{X_i = r\} - q)\bigr),
\end{equation}
and it follows $\PP(\sup_t  M_t(q) \geq 1/\alpha) \leq \alpha$ for any  $\alpha \in (0,1)$ under $p_r \leq q$ (Corollary~\ref{cor:mixture-nsm}). 

Moreover, $q \mapsto M_t(q)$ is
nonincreasing (Lemma~\ref{lem:lcb-monotonicity}), so the set of
$q$ at which the wealth exceeds $1/\alpha$ is a downset, making the following definition well-posed.

\begin{definition}[Lower confidence bound on $p_r$]
\label{def:lcb-formal}
Given $\varepsilon > 0$, at time $t$, we define the lower confidence bound on $p_r$ as
\[
  L_t = L_t(r)
  := \sup\Bigl(\{0\} \cup
     \bigl\{q \in (0, \widehat p_t(r)] :
        M_t(q) \geq 3/\varepsilon\bigr\}\Bigr).
\]
\end{definition}

\begin{proposition}[LCB validity]
\label{prop:lcb-valid}
Under Assumption~\ref{asn:dgp}, it holds that
$\PP\bigl(\exists t \geq 1 : L_t(r) > p_r\bigr) \leq \varepsilon/3$.
\end{proposition}

\subsection{Component (iii): Unseen Upper Bound}
\label{sec:unseen}

We develop an upper bound for the probability of any currently unseen category, simultaneously over all categories and
times. 
A category of mass $u \in (0,1)$ remains unseen after $t$ i.i.d.\ draws with probability
$(1-u)^t$, and a \emph{probability-weighted} union bound turns this
fact into a time-uniform \emph{per-category} upper bound, bypassing
the need to estimate the aggregate missing mass.

\begin{definition}[Unseen upper bound]
\label{def:unseen}
An upper bound of the unseen probabilities is defined as
\[
  U_t := \min\bigl\{u \in (0,1] :
    u^{-1}(1 - u)^{t} \leq \varepsilon/3\bigr\}.
\]
\end{definition}

The threshold balances two forces: $(1 - u)^t$ decays exponentially in
$u$, while the $u^{-1}$ factor
pays the union-bound cost over rare categories. As a result, we obtain the following guarantee:

\begin{proposition}[Simultaneous validity of the unseen bound]
\label{prop:unseen-valid}
Under Assumption~\ref{asn:dgp}, we obtain
\[
  \PP\bigl(
    \exists a \in \cA,\ \exists t \geq 1 :
    N_t(a) = 0 \ \text{and}\ p_a > U_t
  \bigr)
  \leq \varepsilon/3.
\]
\end{proposition}

\subsection{Stopping Rule}
\label{sec:stopping-rule}

We define the CITE stopping rule by combining these three stochastic processes above, which are uniform in time.
In particular, we design the stopping rule to follow the inclusion of the events as \eqref{eq:type1-decomp}.
Algorithm~\ref{alg:iut} shows the pseudo-code of the CITE stopping rule.

\begin{definition}[CITE stopping rule]
\label{def:stopping}
With given $\varepsilon > 0$, CITE stops at the first time both
% \begin{enumerate}[label=(\alph*),leftmargin=2.2em,itemsep=0pt]
%   \item 
  \textbf{(a) Pairwise rejection:}
        $E_t \geq 3/\varepsilon$
        at the empirical runner-up,
  and \textbf{(b) LCB exceeds unseen bound:} $L_t > U_t$,
% \end{enumerate}
hold simultaneously, i.e.\ we define the stopping time as
\begin{equation}\label{eq:tau-def}
  \tau := \inf\bigl\{t \geq 1 :
     E_t \geq 3/ \varepsilon 
    \text{ and } L_t > U_t \bigr\}.
\end{equation}
If $\tau < \infty$, the procedure certifies $r$ as the unique mode.
\end{definition}

\begin{algorithm}[htbp]
\caption{Stopping rule by CITE}
\label{alg:iut}
\hrule\smallskip
\textbf{Input:}
  target label~$r \in \cA$;
  parameter grid $\Lambda_{\mathrm{pw}} \subset (0,1)$ with weights $(w_\lambda)$;
  LCB grid $\Lambda_r \subset (0,\infty)$ with weights $(v_\lambda)$;
  error level $\varepsilon > 0$.
  %$\alpha_{\mathrm{pw}},\alpha_r,\alpha_u > 0$
  %with $\alpha_{\mathrm{pw}} + \alpha_r + \alpha_u \leq \varepsilon$.
  \\[2pt]
\textbf{Output:}
  stopping time~$\tau$ and certification decision.
\smallskip\hrule\smallskip
\begin{enumerate}[leftmargin=1.5em,label=\arabic*.]
  \item Initialise
        $N_0(a) \leftarrow 0$ for all $a \in \cA$;
        $E_0^{(r,a)}(\lambda) \leftarrow 1$ for all $a \neq r$,
        $\lambda \in \Lambda_{\mathrm{pw}}$.
  \item \textbf{For} $t = 1, 2, \ldots$ \textbf{do}
  \begin{enumerate}[leftmargin=1.5em,label=2\alph*.]
    \item Draw $X_t$ and set
          $N_t(a) \leftarrow N_{t-1}(a) + \ind\{X_t = a\}$
          for all $a \in \cA$.
    \item \textsc{Step(i): Pairwise e-processes.}\\
          If $\cA_t \setminus \{r\} = \emptyset$,  set
          $P_t^{\mathrm{pw}} \leftarrow \texttt{true}$ (the unseen
          condition $U_t$ then governs whether $\tau = t$) and skip
          to Step~(II); \\
          otherwise, set
          $\widehat a_t \leftarrow \arg\max_{a \in \cA_t \setminus \{r\}} N_t(a)$
          and, for each $\lambda \in \Lambda_{\mathrm{pw}}$, compute
          \[
            E_t =  E_t^{(r,\widehat a_t)}
            \leftarrow
            {\textstyle\sum_{\lambda}}
            w_\lambda
            E_t^{(r,\widehat a_t)}(\lambda) ,\qquad
            E_t^{(r,\widehat a_t)}(\lambda)
            \leftarrow
            (1+\lambda)^{N_t(r)}(1-\lambda)^{N_t(\widehat a_t)},
          \]
          and set
          $P_t^{\mathrm{pw}} \leftarrow \big\{  E_t \geq 3/\varepsilon \big\}$.
    \item \textsc{Step(ii): Lower confidence bound on $p_r$.}
          Set $\widehat p_t(r) \leftarrow N_t(r)/t$ and
          \[
            L_t = L_t(r) \leftarrow \sup\bigl(\{0\} \cup
            \{q \in (0, \widehat p_t(r)] :
             M_t(q) \geq 3/\varepsilon\}\bigr).
          \]
    \item \textsc{Step(iii): Unseen upper bound.}
          Set $U_t \leftarrow \min\{u \in (0,1] :
            u^{-1}(1-u)^t \leq \varepsilon/3\}$.
    \item \textbf{If} $P_t^{\mathrm{pw}}$ \textbf{and} $L_t > U_t$
          \textbf{then}
          return $(\tau \leftarrow t$, ``certify $r$ as unique mode'').
  \end{enumerate}
\end{enumerate}
\smallskip\hrule
\end{algorithm}

\section{Theoretical Guarantees}

We establish three theoretical guarantees for CITE.  First, the false
certification probability is controlled at the prescribed anytime-valid
level for the fixed target $r$ (Theorem~\ref{thm:main-iut}).  Second,
under the unique-mode alternative, the expected stopping time is bounded
by a category-set-size-free rate depending on $(\delta,p_r)$ but not on
$|\cA|$ (Theorem~\ref{thm:power-iut-clean}). Third, a combined
information-theoretic lower bound shows that this rate is optimal up to
constants and logarithmic terms
(Theorem~\ref{thm:minimax-combined-lower-bound}).

\subsection{Main Theorems}
\label{sec:main-theorems}

We first provide the sequential version of Type-I error control on the stopping time $\tau$ given in Definition \ref{def:stopping}.
\begin{theorem}[Type-I error control]
\label{thm:main-iut} Let $\varepsilon \in (0,1)$ be a prespecified significance level. 
Under Assumption~\ref{asn:dgp}, the stopping time $\tau$ of
Definition~\ref{def:stopping} satisfies
\[
\sup_{P \in H_{0,r}} \PP_P(\tau < \infty)
\leq \varepsilon.
\]
\end{theorem}
The proof, deferred to Appendix~\ref{app:proof-thm:main-iut}, fixes a
null witness $a^\star\neq r$ satisfying $p_{a^\star}\ge p_r$.  If this
witness has been observed by the stopping time, runner-up monotonicity
implies that the fixed-competitor e-process
$ E_t^{(r,a^\star)}$ must have crossed $3/\varepsilon$.  If
the witness is still unseen, then the inequality $L_t>U_t$ can hold
only if either the LCB or unseen bound has failed.  This yields
\eqref{eq:type1-decomp} and a union bound over the three
$\varepsilon/3$ components.

The above result on the Type-I error control does not provide any guarantee regarding the time required to certify the true mode. Under the alternative, write the
modal gap as $\delta := p_r - p_{(2)} > 0$. We now show that
$\EE[\tau]$ is bounded by a function of $(\delta, p_r)$ alone,
 independent of $|\cA|$. 

\begin{condition}
\label{cond:grid}
$\Lambda_{\mathrm{pw}}$ contains some $\lambda_{\mathrm{pw}} \in [\delta/8, \delta/4]$,
and $\Lambda_r$ contains some $\lambda_r \in [1/32, 1/16]$.
\end{condition}

\begin{theorem}[Stopping-time bound]
\label{thm:power-iut-clean}
Suppose Assumption~\ref{asn:dgp} holds and
Condition~\ref{cond:grid} with $\delta > 0$, and the selected grid weights
$w_{\lambda_{\mathrm{pw}}}, v_{\lambda_r}$ are bounded below by a
constant independent of $\delta, p_r, \varepsilon$. Then, we have
$\PP(\tau < \infty) = 1$ and
\begin{equation}\label{eq:headline-rate}
  \EE[\tau]
  =
  O \left(
    \frac{\log(1/\varepsilon) + \log(1/\delta)}{\delta^2}
     +
    \frac{\log(1/\varepsilon) + \log(1/p_r)}{p_r}
  \right),
\end{equation}
with universal constants independent of $\delta$, $p_r$, $|\cA|$.
\end{theorem}

We next state the matching lower bound used to interpret
Theorem~\ref{thm:power-iut-clean}.  For $p\in(0,1/2]$ and
$\delta\in(0,p)$, define
\[
  \cP(p,\delta)
  :=
  \left\{
    P:
    p_r=p,\quad p_{(2)}\le p-\delta
  \right\}.
\]

\begin{theorem}[Combined minimax lower bound]
\label{thm:minimax-combined-lower-bound}
Fix $p\in(0,1/2]$ and $\delta\in(0,p)$.  Assume either that $\cA$ is
countably infinite, or that it contains at least $\left\lceil({1-p})/({p-\delta})\right\rceil+2$
distinct non-target categories.  Let $\tau'$ be any
level-$\varepsilon$ sequential certification procedure for the fixed
target $r$, i.e., $\sup_{Q\in H_{0,r}}\PP_Q(\tau'<\infty)\le\varepsilon$.
Assume that, for every $P\in\cP(p,\delta)$, $\PP_P(\tau'<\infty)=1$ and $\EE_P[\tau']<\infty$ hold.
Then, there is a universal constant $c>0$ such that 
\[
  \sup_{P\in\cP(p,\delta)}
  \EE_P[\tau']
  \ge
  c\log(1/\varepsilon)
  \left(
    \frac{p}{\delta^2}
    +
    \frac{1}{p}
  \right).
\]
\end{theorem}

Comparing the lower bound of Theorem~\ref{thm:minimax-combined-lower-bound} with the upper bound \eqref{eq:headline-rate}, we show that the CITE stopping time achieves the minimax optimal order up to logarithmic factors whenever $p_r$ is bounded away from zero. 

In the very small-$p_r$ regime, the current upper bound can be worse in the pairwise term; closing this gap is left open.

In Appendix~\ref{app:theory-extras}, we give extensions of the analysis: sharper bounds with refined constant (Theorem~\ref{thm:sharp-ratio}), %gives . %constant-level statement for an oracle pairwise benchmark.  It shows that, if the truerunner-up and the oracle tuning parameter are known, the pairwise e-process attains the leading constant determined by the KL distance to the closest null distribution.  This should be viewed as a benchmark for the pairwise component, not as a sharp-constant theorem for the full non-oracle CITE stopping time.
% contains three companion results: a
%growth-rate optimality result, in the sense of
%\citet{grunwald2024safe}, showing that 
a growth rate of the e-process
%the oracle fixed-$\lambda^\star$ pairwise bet maximizes expected log-growth within predictable linear pairwise bets, 
%while the finite mixture differs by the additive log-weight term $\log(1/w_{\lambda^\star})$ when $\lambda^\star$ is on the grid 
(Theorem~\ref{thm:grow}); 
a stopping time bound with an
geometric grid
%with a stopping-time bound uniform over$\delta\ge\delta_0$ and an extra $\log\log(1/\delta_0)$ term
(Theorem~\ref{thm:adaptive-grid}); and a top-$k$ extension for the Type-I error and the stopping-time bound 
%certifying a prespecified set of $k$ categories as the top-$k$ set
(Theorem~\ref{thm:top-k}).

\section{Extension to Confidence-Weighted Voting}
\label{sec:weighted}

\subsection{Setup and Weighted CITE}

We extend CITE to the confidence-weighted case: suppose we observe a pair $(X_i, W_i)$  with a
confidence score $W_i \in [0,1]$, e.g. the score  from logits or a verifier;
see \citep{taubenfeld2025confidence,fu2025deep}. 
In this case, our purpose is to certify the \emph{weighted
unique mode} $\arg\max_a \mu_a$ with $\mu_a := \EE[W_i \ind\{X_i = a\}]$
and weighted modal gap $\Delta_w := \mu_r - \mu_\star$, where
$\mu_\star := \sup_{a \neq r} \mu_a$.

\begin{assumption}[i.i.d. weighted categories]
\label{asn:weighted-dgp}
$(X_i, W_i)_{i \geq 1}$ are i.i.d.\ with $X_i \in \cA$ and $W_i \in [0,1]$.
\end{assumption}

We develop \textit{Weighted CITE} (W-CITE) for this setting.  The
weighted null and alternative are
\begin{align}
    H_{0,r}^{(w)} := \{P : \mu_r \leq \mu_\star\},  \mbox{
    ~~and~~} H_{1,r}^{(w)} := \{P : \mu_r > \mu_\star\}.
\end{align}
W-CITE extends the three components of Section~\ref{sec:iut} to the
weighted masses $(\mu_a)_{a\in\cA}$.

\paragraph{Component (i): Pairwise E-Values.}
We consider the weighted indicator
$\widetilde Z_i^{(r,a)} := W_i(\ind\{X_i = r\} - \ind\{X_i = a\})$ and a corresponding weighted pairwise e-process and its mixture as
\begin{equation}\label{def:weighted-pairwise-e-process}
  \widetilde E_t^{(r,a)}
  := \sum_{\lambda \in \Lambda_{\mathrm{pw}}}
     w_\lambda \widetilde E_t^{(r,a)}(\lambda), \qquad \widetilde E_t^{(r,a)}(\lambda)
  := \prod_{i=1}^{t}\bigl(1 + \lambda \widetilde Z_i^{(r,a)}\bigr).
\end{equation}
% With The weighted empirical runner-up is
% $
%   \widetilde a_t
%   := \arg\max_{a \in \cA_t \setminus \{r\}} \widehat\mu_t(a)$ with $
%   \widehat\mu_t(a) := t^{-1}\sum_{i=1}^t W_i\ind\{X_i=a\}$.
Because runner-up monotonicity need not hold with heterogeneous weights,
W-CITE uses the all-competitors check
$\min_{a \in \cA_t \setminus \{r\}} \widetilde E_t^{(r,a)} \geq 3/\varepsilon$.

\paragraph{Component (ii): LCB on Target Weighted Mass.}
The LCB targets $\mu_r$ via the increment
$W_i\ind\{X_i = r\}-q$.  Let $\widehat\mu_t(r)
  :=
  \frac1t\sum_{i=1}^t W_i\ind\{X_i=r\}$.
For $q\in(0,1]$, define
\[
  \widetilde M_t(q)
  := \sum_{\lambda \in \Lambda_r(q)}
     v_\lambda \widetilde M_t(q,\lambda),
  \qquad
  \widetilde M_t(q,\lambda)
  := \prod_{i=1}^{t}
     \bigl(1 + \lambda(W_i\ind\{X_i = r\} - q)\bigr),
\]
and set
$
  \widetilde L_t(r)
  := \sup(\{0\} \cup
     \{q \in (0, \widehat\mu_t(r)] :
       \widetilde M_t(q) \geq 3/\varepsilon\}).
$

\paragraph{Component (iii): Unseen Upper Bound.}
The unseen bound $U_t$ is unchanged since $\mu_a \leq p_a$.

\paragraph{Stopping Time.}
W-CITE stops when the weighted pairwise check passes for every observed
competitor and the weighted LCB exceeds the unseen bound,
$\widetilde L_t(r)>U_t$; if no non-target competitor has been observed,
the pairwise part is treated as vacuous.  This all-competitors check
raises the per-round cost from $O(1)$ to $O(|\cA_t|)$, since runner-up
monotonicity (Lemma~\ref{lem:monotone}) fails under heterogeneous
weights.  We write $\tau^{(w)}$ for the resulting stopping time, and
give the validity analysis in Appendix~\ref{app:weighted-power-details}.

\subsection{Theoretical Guarantee for Weighted CITE}
\label{sec:weighted-main}

We develop theoretical analysis for W-CITE, including the Type-I error and the stopping time.

\begin{theorem}[Type-I error, weighted case]
\label{thm:weighted-main}
Under Assumption~\ref{asn:weighted-dgp},
$\sup_{P \in H_{0,r}^{(w)}} \PP_P(\tau^{(w)} < \infty) \leq \varepsilon$.
\end{theorem}

\begin{condition}[Weighted grid coverage]
\label{cond:weighted-grid}
$\Lambda_{\mathrm{pw}}$ contains some
$\lambda_{\mathrm{pw}} \in [\Delta_w/8, \Delta_w/4]$ and
$\Lambda_r$ contains some $\lambda_r \in [1/32, 1/16]$.
\end{condition}

\begin{theorem}[Weighted stopping-time bound]
\label{thm:weighted-clean-rate}
Suppose Assumption~\ref{asn:weighted-dgp} holds with $\Delta_w > 0$
and Condition~\ref{cond:weighted-grid}, and the grid weights
$w_{\lambda_{\mathrm{pw}}}, v_{\lambda_r}$ are bounded below by a
constant independent of $\Delta_w, \mu_r, \varepsilon$. Then we have
$\PP(\tau^{(w)} < \infty) = 1$ and
\begin{equation}\label{eq:weighted-headline-rate}
  \EE[\tau^{(w)}]
   = 
  O \left(
    \frac{\log(1/\varepsilon) + \log(1/\Delta_w)}{\Delta_w^2}
     + 
    \frac{\log(1/\varepsilon) + \log(1/\mu_r)}{\mu_r}
  \right).
\end{equation}
\end{theorem}

The weighted guarantee mirrors the unweighted theory \eqref{eq:headline-rate} with
$(p_r,\delta)$ replaced by $(\mu_r,\Delta_w)$. The bound suggests that weighting
can improve stopping time when it increases the effective gap
$\Delta_w$ without reducing the target weighted mass $\mu_r$ too much.

\section{Empirical Validation}
\label{sec:experiments}

We evaluate CITE on simulated multinomial distributions and on
LLM self-consistency benchmarks. The experiments are designed to evaluate the following
quantities: empirical false certification (Type-I error) under null
targets, empirical power measured by certification probability under
alternatives, and stopping time or sample efficiency, both on synthetic
multinomial models and on LLM answer distributions. Baselines are \textbf{Bonferroni}
(fixed-sample union bound), \textbf{KR}~\citep{kim2025locally}
(fixed-target sample-split test), and
\textbf{MMC}~\citep{cordero2025certified} (sequential
leader-tracking), all at level $\varepsilon = 0.05$. 
We do not include adaptive stopping methods without theoretical guarantees of Type-I error control.
Detailed protocols, per-setting tables, and ablations on component scaling
and the confidence-weighted W-CITE are deferred to
Appendices~\ref{app:sim} and~\ref{app:llm}.

\paragraph{Simulations.}
We evaluate CITE on five multinomial settings
(Appendix~\ref{app:sim-settings}).  First, we check false
certification under the fixed-target null.  In Case~B, across all five
settings and all sample sizes, CITE has empirical
false-certification rate zero, consistent with
Theorem~\ref{thm:main-iut}; the same is true for W-CITE and
fixed-sample Bonferroni.  By contrast, KR shows small nonzero
false-certification rates on Settings~3--4, up to $0.002$, and MMC
shows nonzero false certifications in the same regimes, reaching
$0.006$ on Setting~3 and $0.014$ on Setting~4
(Table~\ref{tab:type1-sim}).

Second, we evaluate certification power and stopping time under the
alternative.  In the easy concentrated Setting~2, all methods saturate
quickly: at $N=64$, CITE certifies with rate $0.890$ and mean stopping
time $\bar\tau\approx36$, while Bonferroni and KR achieve rates
$0.928$ and $0.946$, respectively.  In the diffuse Setting~1, CITE is
more efficient than Bonferroni before saturation, with certification
rates $0.188$ vs.\ $0.100$ at $N=64$ and $0.892$ vs.\ $0.804$ at
$N=128$, reaching full certification by $N=256$ with
$\bar\tau\approx90$.  In the moderate LLM-like Setting~5, CITE reaches
rate $0.544$ at $N=256$ and $0.932$ at $N=512$
with $\bar\tau\approx234$, whereas MMC remains at rate at most
$0.002$ through $N=2048$.  KR often has higher finite-sample power, but
it is a fixed-sample baseline rather than an anytime-valid stopping
rule.  In the near-tie Setting~3 and very diffuse Setting~4, all
methods remain low-power within the evaluated budget range.

Finally, the component-level simulations are consistent with the two
terms in Theorem~\ref{thm:power-iut-clean}.  When $p_r$ is varied at
fixed gap, the LCB--unseen component follows the predicted
$p_r^{-1}$ dependence; in the reported sweep,
$\bar\tau_{\mathrm{lu}}p_r$ stays between roughly $14$ and $20$.  When
$\delta$ is varied at fixed $p_r=0.24$, the LCB--unseen component stays
approximately constant, around $\bar\tau_{\mathrm{lu}}\approx70$, while
the pairwise component is the part most affected by the modal gap
(Appendix~\ref{app:sim-bottleneck}).

\paragraph{LLM self-consistency.}
We evaluate the performance of Qwen3-30B
(\texttt{Qwen3-\allowbreak 30B-\allowbreak A3B-\allowbreak Instruct-\allowbreak2507})~\citep{qwen2025qwen3} and
gpt-oss-20b
(\texttt{openai/gpt-oss-20b})~\citep{openai2025gptoss} on
AIME~2026~\citep{dekoninck2026matharena},
FrontierScience-Olympiad~\citep{openai2025frontierscience}, and
the clinical benchmark
ER-REASON~\citep{mehandru2025erreason}.
Figure~\ref{fig:iut-impact} plots mean accuracy against the
sample budget; the horizontal axis is the mean stopping time for the
sequential methods (CITE, W-CITE, MMC) and the prefix sample
size~$N$ for the fixed-sample baselines (Bonferroni, KR).

\begin{figure}[t]
\centering
\includegraphics[width=\textwidth]{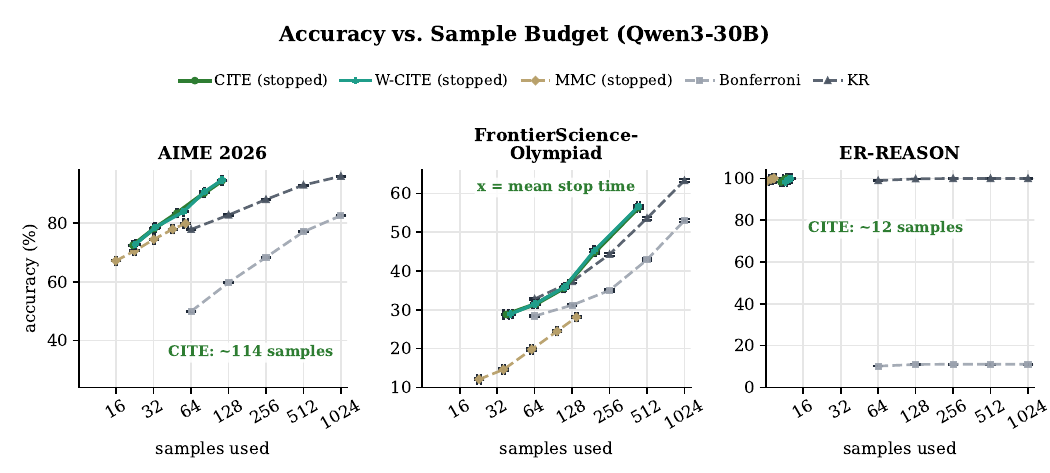}\\
\includegraphics[width=\textwidth]{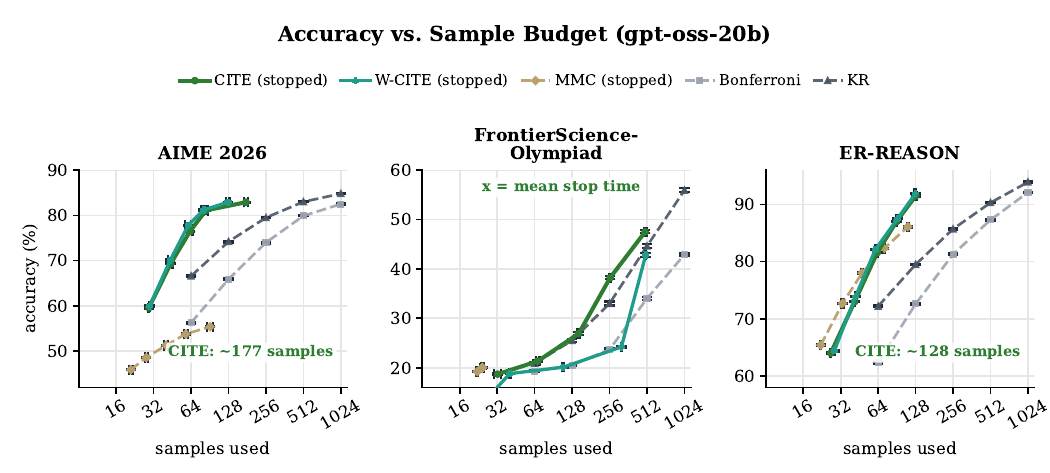}
\caption{Mean accuracy vs.\ sample budget on Qwen3-30B (top) and
gpt-oss-20b (bottom), across AIME~2026,
FrontierScience-Olympiad, and ER-REASON. For the sequential
methods (CITE, W-CITE, MMC) the horizontal axis is the mean
stopping time; for the fixed-sample baselines (Bonferroni, KR)
it is the prefix sample size~$N$.}
\label{fig:iut-impact}
\end{figure}

When the target is set to the empirical runner-up rather than the mode,
the maximum false-certification rate over all evaluated methods, problems,
and budgets is $0.052$ on gpt-oss-20b/ER-REASON and at most
$0.012$ on every other tested (model, dataset) pair; both lie
within Monte Carlo error of the nominal level $\varepsilon = 0.05$
given the multiplicity of comparisons. On Qwen3-30B/ER-REASON,
fixed-sample Bonferroni saturates at certification rate $0.111$
for $N \ge 256$, whereas CITE attains certification rate
$\ge0.98$ at mean stopping time $\bar\tau\approx12$.  This gives a
direct stopping-time diagnostic: in this concentrated-answer regime,
CITE certifies after roughly a dozen samples on average, while the
fixed-sample baseline remains weak even at much larger budgets.  At the largest evaluated budget $N = 1024$, the CITE
certification rate exceeds that of MMC by a factor of
approximately $2$ on FrontierScience-Olympiad (where the
empirical answer set contains about $250$ distinct values per
problem on gpt-oss-20b), while on ER-REASON the two methods
differ by at most $0.054$ across all budgets, attained on
gpt-oss-20b. Relative to the KR baseline, the CITE certification
rate is lower by at most $0.087$, attained on
Qwen3-30B/FrontierScience at $N \in \{256, 512\}$, while
retaining anytime validity, whereas KR, by construction, is not valid under adaptive stopping.
Per-problem and per-budget tables and the corresponding results
for the W-CITE are deferred to
Appendix~\ref{app:llm}.

\section{Conclusion}

We introduced CITE for sequentially certifying the
unique mode of a discrete distribution on a countable, possibly
infinite category set.  %Decomposing the unique-mode null into pairwise,target-LCB, and unseen-mas s component nulls budgeted 
Our theory guarantees that anytime-valid Type-I error control,  %(Theorem~\ref{thm:main-iut}) and an
the stopping time has
%$\log(1/\varepsilon)/\delta^{2}$
%to leading order (Theorem~\ref{thm:power-iut-clean}), 
the category-set-size free order, and the order matches  %multiplicity penalty.  
the information-theoretic lower bound %$\Omega(\log(1/\varepsilon) p_r/\delta^{2})$
%(Theorem~\ref{thm:minimax-combined-lower-bound}) coincides 
up to logarithmic factors. 
%with the upper bound up to universal constants when $p_r$ is bounded below.  
%Appendix~\ref{app:sharp-constants} further shows that an oraclepairwise benchmark achieves the sharp leading constant$4p_r/\delta^{2}$ in the small-gap regime.  
%The decomposition 
We also extend CITE to confidence-weighted voting and develop its weighted version W-CITE. % and recovers the unweighted CITE when $W \equiv 1$ (Proposition~\ref{rop:recovery}).
%The separation of the pairwise and unseen-mass component nulls is essential in the diffuse-tail regime ($p_{(1)} > p_{(2)}$ but $p_{(1)} \leq \sum_{j \geq 2} p_{(j)}$), where leader-tracking certificates such as MMC certify with probability at most the Type-I level $\varepsilon$ on a provable least-favourable family (Proposition~\ref{prop:strict}; Appendix~\ref{sec:mmc}).
In our experiments, both numerical simulations and real LLM environments demonstrated that CITE achieves high accuracy while controlling Type-I errors.

\paragraph{Limitations and future work.}
CITE provides statistical guarantee that a prespecified label is the unique mode of the model's response distribution. This is logically distinct from, and complementary to, \emph{answer correctness}: the two compose, and neither subsumes the other, so a downstream system can combine our anytime-valid mode certificate with any separate correctness signal. Within this scope, a natural extension is to certify the empirically observed top category $\widehat r_t \in \arg\max_a N_t(a)$, i.e., a data-adaptive target; this requires a Bonferroni-type correction over the random sequence of empirical leaders, paralleling the construction of~\citet{cordero2025certified}, and whether the category-size-free rate of Theorem~\ref{thm:power-iut-clean} survives this correction is left open.

% Bibliography inlined for arXiv; no BibTeX run required.

\newpage

\appendix

\section{Related Works}
\label{app:related}

\paragraph{Self-consistency and test-time compute for LLM reasoning.}
Self-consistency and related inference-time scaling methods improve
LLM reasoning by sampling multiple reasoning trajectories and
aggregating their terminal answers, typically by majority voting
\citep{wei2022chain,zhou2023least,wang2023selfconsistency,yao2023tree,
brown2025monkeys,snell2025scaling}.  These methods motivate our
statistical formulation: after normalization, repeated LLM outputs are
viewed as i.i.d. draws from an unknown categorical distribution over
terminal response labels.  Our work is not a new prompting, search, or
reranking method.  Instead, it studies when one may stop sampling and
certify, with a prescribed error level, that a prespecified response
label is the unique mode of this distribution.  This is a guarantee
about the model-induced response distribution, and is distinct from
semantic correctness of the answer.

\paragraph{Adaptive self-consistency and early stopping.}
A growing line of work reduces the cost of self-consistency by adapting
the number of sampled trajectories to the apparent difficulty of each
query.  Adaptive-Consistency uses posterior agreement among sampled
answers as a lightweight stopping signal \citep{aggarwal2023let};
Early-Stopping Self-Consistency stops when a low-entropy window
indicates sufficient agreement \citep{li2024escape}; difficulty-adaptive
and dynamic variants allocate inference budgets using difficulty
estimates, answer-distribution uncertainty, or reasoning-path quality
\citep{wang2025penny,wang2025dynscaling}; and confidence-aware methods
use response-level confidence or verifier-like scores to improve
aggregation and stopping
\citep{taubenfeld2025confidence,aghazadeh2025cges,fu2025deep}.  These
methods are primarily designed to improve the accuracy--cost tradeoff.
By contrast, CITE is designed for statistical certification: it controls
the probability of falsely certifying the target as the unique mode
under arbitrary data-driven stopping, at a user-specified level
$\varepsilon$.

\paragraph{Certified self-consistency.}
The closest LLM-specific statistical work is the Martingale Majority
Certificate (MMC) of \citet{cordero2025certified}, which develops
finite-sample and anytime-valid certificates for self-consistency and
uses martingale arguments to determine when enough samples have been
drawn.  
MMC provides an important first step toward certified self-consistency, but it certifies a stronger majority/residual-tail condition than the fixed-target unique-mode property considered here.
Our target differs in two respects.  First, CITE tests a
fixed-target unique-mode alternative,
$p_r > \sup_{a\ne r} p_a$, whereas MMC-style certificates are naturally
tied to a stronger majority or residual-tail domination condition.
These two notions coincide only in non-diffuse regimes.  In LLM
self-consistency, however, the modal answer may have substantially
larger mass than every individual competitor while still being smaller
than the aggregate mass of the long tail. 
Proposition~\ref{prop:strict} shows that on a provable least-favourable family in this regime, the MMC certifies with probability at most the nominal level $\varepsilon$.
Second, CITE explicitly
handles the possibility that a high-mass competitor has not yet been
observed.  This unseen-category component is essential when the set of
normalized free-form answers is unknown and grows during sampling.

\paragraph{Anytime-valid inference, e-values, and confidence sequences.}
Our construction builds on the classical martingale view of sequential
testing \citep{ville1939} and the modern theory of e-values,
e-processes, testing by betting, and confidence sequences
\citep{howard2020chernoff,howard2021confidence,shafer2021betting,
vovkWang2021evalues,lindon2022anytime,grunwald2024safe,ramdas2025hypothesis}.  These
tools provide tests and confidence statements that remain valid under
continuous monitoring and data-dependent stopping.  CITE instantiates
this machinery for a composite categorical unique-mode problem: it
combines pairwise e-processes against observed competitors with a
time-uniform lower confidence bound for the target mass and an upper
bound for unseen competitors.  The resulting intersection-union test
yields a certification rule whose Type-I error is valid uniformly over
all stopping times.

\paragraph{Best-arm identification and discrete argmin inference.}
The fixed-target unique-mode problem is related to fixed-confidence
best-arm identification and pure-exploration bandits
\citep{garivier2016optimal,kaufmann2016complexity,kaufmann2021mixture}.
Those works characterize the sample complexity of identifying the best
arm under sequential sampling and often rely on a known finite set of
arms.  Our setting is different: each LLM query returns a single
categorical draw from the entire response distribution, the category set
is not known in advance, and a previously unseen response can become a
relevant competitor.  The closest fixed-target finite-dimensional
inference problem is the dimension-agnostic discrete argmin test of
\citet{kim2025locally}, which we include as a baseline.  CITE differs by
providing an anytime-valid stopping rule for a countable and unknown
category set, rather than a fixed-sample or sample-split test over a
given finite vector.

\paragraph{Unseen species, missing mass, and support-agnostic inference.}
The need to reason about unobserved response labels connects our work to
the classical unseen-species and missing-mass literature, including
Good--Turing and Good--Toulmin estimation
\citep{good1953population,goodtoulmin1956new}, concentration bounds for
the missing mass
\citep{mcallesterortiz2003missingmass,berendkontorovich2013missingmass},
and minimax prediction of unseen species
\citep{orlitsky2016unseen}.  These works estimate or bound quantities
such as the aggregate probability of unseen outcomes or the number of
new species.  CITE uses a different object: for certification of a
unique mode, it is enough to rule out the existence of a single unseen
competitor whose mass could exceed the target.  The unseen component of
CITE is therefore tailored to sequential modal certification rather than
to estimating the entire missing mass or support size.

\paragraph{Confidence-weighted aggregation.}
Recent work shows that response-level confidence can improve
self-consistency by weighting or prioritizing high-confidence reasoning
paths \citep{taubenfeld2025confidence,aghazadeh2025cges,fu2025deep}.
Our weighted extension formalizes a complementary question: when the
observations are pairs $(X_i,W_i)$ with bounded confidence weights, can
one certify that a target label is the unique weighted mode?  W-CITE
extends the same anytime-valid and category-agnostic guarantee to this
weighted setting, with the weighted gap replacing the unweighted modal
gap in the stopping-time analysis.

\section{Proofs for Section~\ref{sec:iut}}
\label{app:proofs-sec:iut}
\label{app:ville}

Let $\alpha_{\mathrm{pw}}, \alpha_r, \alpha_u > 0$ be error levels of each component satisfying
$\alpha_{\mathrm{pw}}+\alpha_r+\alpha_u\le\varepsilon$.
The main text uses the equal split
\[
  \alpha_{\mathrm{pw}}
  =
  \alpha_r
  =
  \alpha_u
  =
  \varepsilon/3 .
\]
For readability, the Appendix proofs keep the symbols
$\alpha_{\mathrm{pw}},\alpha_r,\alpha_u$ and substitute the equal split
only at the end.

We prepare some supportive results as follow.

For an $(\cF_t)$-adapted nonnegative process $(M_t)_{t \geq 0}$
with $\EE_P[M_t \mid \cF_{t-1}] \leq M_{t-1}$ and $M_0 \leq 1$
(an NSM), Ville's inequality \citep{ville1939} gives the
time-uniform tail bound
\begin{equation}\label{eq:ville}
  P\Bigl(\sup_{t \geq 0} M_t \geq 1/\alpha\Bigr) \leq \alpha,
  \qquad \alpha \in (0,1],
\end{equation}
which we invoke repeatedly in what follows.

\begin{lemma}[Attainment of a positive supremum]
\label{lem:attained-supremum}
Let $(x_a)_{a\in\cB}$ be a nonnegative summable sequence indexed by a
countable set $\cB$.  If $\sup_{a\in\cB}x_a>0$, then the supremum is
attained.
\end{lemma}

\begin{proof}
Let $s:=\sup_{a\in\cB}x_a>0$.  If no index attained $s$, then for each
$n$ one could choose a distinct $a_n\in\cB$ with $x_{a_n}>s/2$.  This
would contradict $\sum_{a\in\cB}x_a<\infty$.
\end{proof}

\begin{lemma}[Pairwise NSM]
\label{lem:pairwise-nsm}
Fix $a\neq r$ and $\lambda\in(0,1)$.  Under any $P$ with
$p_r\le p_a$, the process
$(E_t^{(r,a)}(\lambda))_{t\ge0}$ is an NSM with
$E_0^{(r,a)}(\lambda)=1$.
\end{lemma}

\begin{lemma}[Runner-up monotonicity]
\label{lem:monotone}
For any $a,b\in\cA_t\setminus\{r\}$ with
$N_t(a)\ge N_t(b)$ and any $\lambda\in(0,1)$, it holds that
\[
  E_t^{(r,a)}(\lambda)
  \le
  E_t^{(r,b)}(\lambda).
\]
Consequently, we have
\[
E_t^{(r,a)}
  \le
E_t^{(r,b)}.
\]
\end{lemma}

\subsection{Proof of Lemma~\ref{lem:pairwise-nsm}}
\label{app:proof-lem:pairwise-nsm}
 
\begin{proof}
We have $E_0 = 1$ (empty product) and $E_t \geq 0$ since each factor
$1 + \lambda Z_i \geq 1 - \lambda > 0$.  For the supermartingale
property:
\begin{align*}
  \EE[E_t^{(r,a)}(\lambda) \mid \cF_{t-1}]
  &= E_{t-1}^{(r,a)}(\lambda)
     \cdot (1 + \lambda \EE[Z_t^{(r,a)}]) \\
  &= E_{t-1}^{(r,a)}(\lambda)
     \cdot (1 + \lambda(p_r - p_a)),
\end{align*}
using the i.i.d.\ property.  Under $p_r \leq p_a$, the factor
$1 + \lambda(p_r - p_a) \leq 1$.
\end{proof}

\subsection{Proof of Corollary~\ref{cor:mixture-nsm}}
\label{app:proof-cor:mixture-nsm}
 
\begin{proof}
For each $\lambda \in \Lambda_{\mathrm{pw}}$,
$(E_t^{(r,a)}(\lambda))_{t \geq 0}$ is an NSM under $p_r \leq p_a$
by Lemma~\ref{lem:pairwise-nsm}.  Since
\[
  {E}_t^{(r,a)}
  =
  \sum_{\lambda\in\Lambda_{\mathrm{pw}}}
  w_\lambda E_t^{(r,a)}(\lambda)
\]
is a nonnegative weighted sum with total weight at most one, it is also
an NSM and satisfies $ E_0^{(r,a)}\le1$.  Ville's inequality
\eqref{eq:ville} then gives
\[
  \PP_P\left(
    \sup_{t\ge0}
     E_t^{(r,a)}
    \ge
    \alpha^{-1}
  \right)
  \le
  \alpha .
\]
\end{proof}

\subsection{Proof of Lemma~\ref{lem:monotone}}
\label{app:proof-lem:monotone}

\begin{proof}
By~\eqref{def:pairwise-e-process},
$E_t^{(r,a)}(\lambda) = (1+\lambda)^{N_t(r)}(1-\lambda)^{N_t(a)}$.
The factor $(1+\lambda)^{N_t(r)}$ depends only on $r$ and is common to both
$E_t^{(r,a)}(\lambda)$ and $E_t^{(r,b)}(\lambda)$.
Since $0 < 1 - \lambda < 1$ and $N_t(a) \geq N_t(b)$,
\[
  (1-\lambda)^{N_t(a)} \leq (1-\lambda)^{N_t(b)},
\]
whence $E_t^{(r,a)}(\lambda) \leq E_t^{(r,b)}(\lambda)$.  Multiplying
by the nonnegative weights $w_\lambda$ and summing over
$\lambda\in\Lambda_{\mathrm{pw}}$ gives
\[
   E_t^{(r,a)}
  \le
   E_t^{(r,b)}.
\]
\end{proof}

\subsection{Proof of Proposition~\ref{prop:lcb-valid}}
\label{app:proof-prop:lcb-valid}

\begin{proof}
If $p_r = 0$, then $N_t(r) = 0$ a.s.\ for all $t$, so $\widehat p_t(r) = 0$
and $L_t(r) = 0 = p_r$ a.s.;
the event $\{\exists t : L_t(r) > p_r\}$ has probability $0$.
Assume $p_r > 0$ in the remainder.

\smallskip
\noindent\textbf{Step 1: Monotonicity in $q$.}
For any $0 < q_1 < q_2 \leq 1$, $\Lambda_r(q_2) \subseteq \Lambda_r(q_1)$
and each factor
$1 + \lambda(\ind\{X_i=r\} - q)$ is nonincreasing in $q$; hence
$q \mapsto  M_t(q)$ is nonincreasing.

\smallskip
\noindent\textbf{Step 2: NSM at $q = p_r$.}
Set $M_t(q,\lambda) := \prod_{i=1}^t (1 + \lambda(\ind\{X_i=r\} - q))$.
For $\lambda \in \Lambda_r(p_r)$, every factor
$1 + \lambda(\ind\{X_i=r\} - p_r) \geq 1 - \lambda p_r > 0$,
and $\EE[1 + \lambda(\ind\{X_i=r\} - p_r)] = 1$, so
$(M_t(p_r, \lambda))_{t \geq 0}$ is a nonnegative martingale with
initial value $1$.  The mixture
\[
   M_t(p_r)
  := \sum_{\lambda \in \Lambda_r(p_r)} v_\lambda M_t(p_r, \lambda)
\]
is therefore an NSM with
$ M_0(p_r) = \sum_{\lambda \in \Lambda_r(p_r)} v_\lambda \leq 1$
(Corollary-style argument analogous to Corollary~\ref{cor:mixture-nsm}).
Ville's inequality~\eqref{eq:ville} gives
\begin{equation}\label{eq:ville-lcb-pr}
  \PP\Bigl(\sup_{t \geq 0}  M_t(p_r) \geq 1/\alpha_r\Bigr)
  \leq \alpha_r.
\end{equation}

\smallskip
\noindent\textbf{Step 3: Event inclusion.}
Suppose $L_t(r) > p_r$ for some $t \geq 1$.  By the definition of
$L_t(r)$, there exists $q \in (p_r, \widehat p_t(r)]$ with
$ M_t(q) \geq 1/\alpha_r$.  Step~1 then implies
$ M_t(p_r) \geq  M_t(q) \geq 1/\alpha_r$.  Hence
\[
  \{\exists t \geq 1 : L_t(r) > p_r\}
  \subseteq
  \{\sup_{t \geq 0}  M_t(p_r) \geq 1/\alpha_r\},
\]
and the desired bound follows from~\eqref{eq:ville-lcb-pr}.
\end{proof}

\subsection{Proof of Proposition~\ref{prop:unseen-valid}}
\label{app:proof-prop:unseen-valid}
 
\begin{proof}
For a fixed integer $t \geq 1$, define
$f_t(u) := u^{-1}(1-u)^t$ for $u \in (0,1]$.
Since $(\log f_t)'(u) = -1/u - t/(1-u) < 0$,
$f_t$ is strictly decreasing on $(0,1]$.
 
Now fix $a \in \cA$ with $p_a > 0$.
Since $U_t$ is nonincreasing in $t$, define
$t_*(a) := \inf\{t \geq 1 : p_a > U_t\}$ (possibly $+\infty$).
If $t_*(a)=\infty$, then $p_a \leq U_t$ for all $t$, so the event
$\{\exists t : N_t(a)=0, p_a > U_t\}$ is empty.
Otherwise, $t_*(a) < \infty$ and $p_a > U_{t_*(a)}$.
For any $t \geq t_*(a)$ with $p_a > U_t$, we have $N_t(a)=0$ only if
$N_{t_*(a)}(a)=0$ (since counts are nondecreasing), so the event
$\{\exists t : N_t(a)=0, p_a > U_t\}$ is contained in $\{N_{t_*(a)}(a)=0\}$.
By definition of $U_{t_*(a)}$, we have
$f_{t_*(a)}(U_{t_*(a)}) \leq \alpha_u$.
Since $p_a > U_{t_*(a)}$ and $f_{t_*(a)}$ is strictly decreasing,
$f_{t_*(a)}(p_a) = p_a^{-1}(1-p_a)^{t_*(a)} < f_{t_*(a)}(U_{t_*(a)}) \leq \alpha_u$, hence
\[
  \PP\bigl(N_{t_*(a)}(a)=0\bigr)
  = (1-p_a)^{t_*(a)}
  < \alpha_u p_a.
\]
A probability-weighted union bound over the countable category set gives
\[
  \PP\Bigl(
    \exists a \in \cA, \exists t \geq 1 :
    N_t(a)=0 \text{ and } p_a > U_t
  \Bigr)
  \leq
  \sum_{a \in \cA} \alpha_u p_a
  =
  \alpha_u.
  \qedhere
\]
\end{proof}

\subsection{Proof of Theorem~\ref{thm:main-iut}}
\label{app:proof-thm:main-iut}

\begin{proof}
Fix any $P\in H_{0,r}$. %By the witness argument in
%Appendix~\ref{app:conventions}, 
We choose a deterministic
$a^\star\neq r$ such that
\[
  p_{a^\star}\ge p_r .
\]
We decompose the Type-I event according to whether $a^\star$ has been
observed by time $\tau$:
\begin{equation}\label{eq:decomp}
  \{\tau < \infty\}
   \subseteq 
  \underbrace{\{\tau < \infty,\ a^\star \in \cA_\tau\}}_{E_1}
   \cup 
  \underbrace{\{\tau < \infty,\ a^\star \notin \cA_\tau\}}_{E_2}.
\end{equation}

\medskip
\noindent\textbf{Step 1 (bounding $\PP(E_1)$).}
On $E_1$, the witness $a^\star$ is observed at time $\tau$, so
$\cA_\tau\setminus\{r\}\neq\emptyset$.  Since CITE stops at $\tau$, the
pairwise part of the stopping rule gives
\[
   E_\tau^{(r,\widehat a_\tau)}
  =
  E_\tau
  \ge
  \alpha_{\mathrm{pw}}^{-1}.
\]
By the definition of the empirical runner-up,
$N_\tau(\widehat a_\tau)\ge N_\tau(a^\star)$.  Lemma~\ref{lem:monotone}
then gives
\[
   E_\tau^{(r,\widehat a_\tau)}
  \le
   E_\tau^{(r,a^\star)}.
\]
Thus
\[
   E_\tau^{(r,a^\star)}
  \ge
  \alpha_{\mathrm{pw}}^{-1}.
\]
Since $a^\star$ is deterministic under $P$ and
$p_r\le p_{a^\star}$, Corollary~\ref{cor:mixture-nsm} and Ville's
inequality yield
\begin{equation}\label{eq:e1-bound}
  \PP_P(E_1)
  \le
  \PP_P\left(
    \sup_{t\ge0}
     E_t^{(r,a^\star)}
    \ge
    \alpha_{\mathrm{pw}}^{-1}
  \right)
  \le
  \alpha_{\mathrm{pw}}.
\end{equation}
\medskip
\noindent\textbf{Step 2 (bounding $\PP(E_2)$).}
On $E_2$, CITE stops at $\tau$, so $L_\tau(r)>U_\tau$, and
$a^\star$ is unseen at $\tau$.  Define the two failure events
\begin{align*}
  F_r &:= \{\exists t \geq 1 : L_t(r) > p_r\}, \\
  F_u &:= \{\exists a \in \cA,\ \exists t \geq 1 :
            N_t(a) = 0 \text{ and } p_a > U_t\}.
\end{align*}
Propositions~\ref{prop:lcb-valid} and~\ref{prop:unseen-valid} give
$\PP_P(F_r) \leq \alpha_r$ and $\PP_P(F_u) \leq \alpha_u$.
On $E_2 \setminus (F_r \cup F_u)$, we have
$L_\tau(r)\le p_r$ and, because $a^\star$ is unseen at $\tau$,
$p_{a^\star}\le U_\tau$.  Combining these inequalities with
$L_\tau(r)>U_\tau$ yields
\[
  p_{a^\star}
  \le
  U_\tau
  <
  L_\tau(r)
  \le
  p_r,
\]
which contradicts $p_{a^\star}\ge p_r$.  Hence,
$E_2 \subseteq F_r \cup F_u$, and
\begin{equation}\label{eq:e2-bound}
  \PP_P(E_2)  \leq  \PP_P(F_r) + \PP_P(F_u)  \leq  \alpha_r + \alpha_u.
\end{equation}

\medskip
\noindent\textbf{Step 3 (union bound).}
Combining \eqref{eq:e1-bound}--\eqref{eq:e2-bound},
\[
  \PP_P(\tau < \infty)
   \leq  \PP_P(E_1) + \PP_P(E_2)
   \leq  \alpha_{\mathrm{pw}} + \alpha_r + \alpha_u
   \leq  \varepsilon.
  \qedhere
\]
\end{proof}

\section{Power Analysis for CITE}
\label{app:power-iut}

Throughout this section, $X_1,X_2,\ldots$ are i.i.d.\ on the countable
category set $\cA$, and the fixed target $r$ is the unique mode.  Write
\[
  p_{(2)}
  :=
  \sup_{a\neq r}p_a,
  \qquad
  \delta
  :=
  p_r-p_{(2)}
  >
  0.
\]
For the analysis, define the conservative competitor count
\[
  R_t
  :=
  \max_{a\neq r}N_t(a),
\]
where zero-count competitors are included in the maximum, and define
\[
  E_t^\circ
  :=
  \sum_{\lambda\in\Lambda_{\mathrm{pw}}}
  w_\lambda
  (1+\lambda)^{N_t(r)}
  (1-\lambda)^{R_t}.
\]
The corresponding conservative stopping time is
\[
  \tau^\circ
  :=
  \inf\left\{
    t\ge1:
    E_t^\circ\ge\alpha_{\mathrm{pw}}^{-1}
    \text{ and }
    L_t(r)>U_t
  \right\}.
\]
This stopping time is no earlier than the implemented CITE stopping
time: replacing the observed runner-up count by
$R_t=\max_{a\neq r}N_t(a)$ can only make the pairwise condition harder,
while the LCB--unseen condition is unchanged.  Hence
\[
  \tau\le\tau^\circ
  \quad\text{pathwise},
\]
and it suffices to upper-bound $\EE[\tau^\circ]$.

\subsection{Auxiliary lemmas for the clean stopping-time bound}

\begin{lemma}[A deterministic logarithmic inequality]
\label{lem:det-log}
Let $A>0$, $C>0$, and $t>0$.  If
\[
  t
  \ge
  2A\log\bigl((AC)\vee e\bigr),
\]
then
\[
  t
  \ge
  A\log(Ct).
\]
\end{lemma}

\begin{proof}
Set $u=t/A$.  The assumption gives
$u\ge2\log((AC)\vee e)\ge2$, so
$\log(CA)\le u/2$ and $\log u\le u/2$.  Therefore
\[
  \log(Ct)
  =
  \log(CA)+\log u
  \le
  u
  =
  t/A .
\]
\end{proof}

\begin{lemma}[Expected maximum competitor count]
\label{lem:max-count}
For every integer $t\ge1$,
\[
  \EE[R_t]
  \le
  t p_{(2)}
  +
  2\sqrt{2t\log(t+1)}.
\]
\end{lemma}

\begin{proof}
For $a\neq r$, set $f_a(x):=\ind\{x=a\}$.  Since
$p_a\le p_{(2)}$,
\[
  \EE[R_t]-tp_{(2)}
  \le
  \EE\sup_{a\neq r}
  \sum_{i=1}^t
  \{f_a(X_i)-\EE f_a(X_i)\}.
\]
By symmetrization, the right-hand side is at most
\[
  2\EE\EE_\sigma
  \sup_{a\neq r}
  \sum_{i=1}^t\sigma_i f_a(X_i),
\]
where $\sigma_i$ are i.i.d.\ Rademacher variables independent of the
sample.  Conditional on the sample, the vectors
\[
  (f_a(X_1),\ldots,f_a(X_t)),
  \qquad a\neq r,
\]
have disjoint supports and hence at most $t+1$ distinct values, each
with Euclidean norm at most $\sqrt t$.  The standard exponential
maximal inequality gives
\[
  \EE_\sigma
  \sup_{a\neq r}
  \sum_{i=1}^t\sigma_i f_a(X_i)
  \le
  \sqrt{2t\log(t+1)}.
\]
Combining the displays proves the claim.
\end{proof}

\begin{lemma}[LCB monotonicity and inversion]
\label{lem:lcb-monotonicity}
Fix $t\ge1$.  For
\[
  M_t(q,\lambda)
  :=
  \prod_{i=1}^t
  \bigl\{1+\lambda(\ind\{X_i=r\}-q)\bigr\},
\]
the following hold.
\begin{enumerate}[label=\textup{(\alph*)},leftmargin=2.2em]
\item The map $q\mapsto M_t(q)$ is nonincreasing on $(0,1]$.
\item If $ M_t(q)>1$, then $\widehat p_t(r)>q$.
\item If $ M_t(q)\ge\alpha_r^{-1}$, then $L_t(r)\ge q$.
\end{enumerate}
\end{lemma}

\begin{proof}
For (a), if $q_1<q_2$, then
$\Lambda_r(q_2)\subseteq\Lambda_r(q_1)$, and each remaining factor
$1+\lambda(\ind\{X_i=r\}-q)$ is nonincreasing in $q$.  Hence
$ M_t(q_1)\ge M_t(q_2)$.

For (b), suppose $\widehat p_t(r)\le q$.  For fixed
$\lambda\in\Lambda_r(q)$, write
\[
  \phi(\lambda)
  :=
  \widehat p_t(r)\log(1+\lambda(1-q))
  +
  (1-\widehat p_t(r))\log(1-\lambda q).
\]
Then $\log M_t(q,\lambda)=t\phi(\lambda)$,
$\phi(0)=0$, $\phi'(0)=\widehat p_t(r)-q\le0$, and $\phi$ is concave on
$[0,1/q)$.  Therefore $M_t(q,\lambda)\le1$ for every admissible
$\lambda$, so
\[
   M_t(q)
  \le
  \sum_{\lambda\in\Lambda_r(q)}v_\lambda
  \le
  1.
\]
This proves the contrapositive.

For (c), if $ M_t(q)\ge\alpha_r^{-1}>1$, then (b) gives
$\widehat p_t(r)>q$, so $q$ belongs to the set defining $L_t(r)$.
Hence $L_t(r)\ge q$.
\end{proof}

\begin{lemma}[Quadratic lower bound for $\log(1+x)$]
\label{lem:power-log-lower}
For every $x\in[-1/2,1/2]$,
\[
  \log(1+x)\ge x-x^2 .
\]
\end{lemma}

\begin{proof}
Let $h(x):=\log(1+x)-x+x^2$.  Then $h(0)=0$ and
$h'(x)=x(1+2x)/(1+x)$, which is nonpositive on $[-1/2,0]$ and
nonnegative on $[0,1/2]$.  Thus $h(x)\ge0$ on $[-1/2,1/2]$.
\end{proof}

\begin{proposition}[Small-bet pairwise drift]
\label{prop:power-smallbet-pair}
Assume $\delta>0$ and
\[
  0<\lambda\le
  \frac{\delta}{2(p_r+p_{(2)})}.
\]
Then, for every $a\neq r$ with $p_a\le p_{(2)}$,
\[
  \EE\left[
    \log\{1+\lambda(\ind\{X=r\}-\ind\{X=a\})\}
  \right]
  \ge
  \frac{\lambda\delta}{2}.
\]
\end{proposition}

\begin{proof}
Since $\delta\le p_r+p_{(2)}$, the assumed bound gives
$\lambda\le1/2$.  Let
\[
  \xi
  :=
  \lambda(\ind\{X=r\}-\ind\{X=a\})
  \in[-1/2,1/2].
\]
By Lemma~\ref{lem:power-log-lower},
\[
  \EE[\log(1+\xi)]
  \ge
  \lambda(p_r-p_a)-\lambda^2(p_r+p_a).
\]
Using $p_a\le p_{(2)}$ gives
\[
  \EE[\log(1+\xi)]
  \ge
  \lambda\delta-\lambda^2(p_r+p_{(2)})
  \ge
  \lambda\delta/2 .
\]
\end{proof}

\subsection{Proof of Theorem~\ref{thm:power-iut-clean}}
\label{app:proof-thm:power-iut-clean}

\begin{proof}
We prove the bound for $\tau^\circ$; since $\tau\le\tau^\circ$
pathwise, the same expectation bound holds for the implemented CITE
stopping time.

\smallskip
\noindent\textit{Step 1: pairwise condition.}
Let $\lambda:=\lambda_{\mathrm{pw}}\in[\delta/8,\delta/4]$ be the grid
point from Condition~\ref{cond:grid}, with weight $w_\lambda$, and set
\[
  B_{\mathrm{pw}}
  :=
  \log\{1/(\alpha_{\mathrm{pw}}w_\lambda)\}.
\]
For this fixed $\lambda$, define
\[
  S_t
  :=
  N_t(r)\log(1+\lambda)
  +
  R_t\log(1-\lambda).
\]
Since $E_t^\circ\ge w_\lambda\exp(S_t)$, the pairwise condition
$E_t^\circ\ge\alpha_{\mathrm{pw}}^{-1}$ holds whenever
$S_t\ge B_{\mathrm{pw}}$.

Using Lemma~\ref{lem:max-count} and $\log(1-\lambda)<0$,
\[
  \EE[S_t]
  \ge
  t\{p_r\log(1+\lambda)+p_{(2)}\log(1-\lambda)\}
  -
  C\lambda\sqrt{t\log(t+1)} .
\]
Because $p_r+p_{(2)}\le1$, $\lambda\le\delta/4$, and
$\log(1+x)\ge x-x^2$ on $[-1/2,1/2]$,
\[
  p_r\log(1+\lambda)+p_{(2)}\log(1-\lambda)
  \ge
  \lambda\delta-\lambda^2(p_r+p_{(2)})
  \ge
  c\delta^2
\]
for a universal constant $c>0$.  Hence
\[
  \EE[S_t]
  \ge
  c_0\delta^2 t
  -
  C_0\delta\sqrt{t\log(t+1)} .
\]
Changing one observation changes both $N_t(r)$ and $R_t$ by at most one,
so $S_t$ is $C_1\delta$-Lipschitz in each sample.  McDiarmid's
inequality gives
\[
  \PP\{S_t<\EE[S_t]-x\}
  \le
  \exp\left\{
    -\frac{c_1x^2}{t\delta^2}
  \right\}.
\]
Taking $x=C_2\delta\sqrt{t\log(et)}$ and using
Lemma~\ref{lem:det-log}, we obtain, for a universal constant $C$,
\begin{equation}\label{eq:clean-pw-tail}
  \PP\left(E_t^\circ<\alpha_{\mathrm{pw}}^{-1}\right)
  \le
  (et)^{-2},
  \qquad
  t\ge
  T_{\mathrm{pw}}
  :=
  C\frac{B_{\mathrm{pw}}+\log(1/\delta)}{\delta^2}.
\end{equation}

\smallskip
\noindent\textit{Step 2: LCB--unseen condition.}
Set $q_0:=p_r/2$.  Condition~\ref{cond:grid} gives
$\lambda_r\in[1/32,1/16]\subset\Lambda_r(q_0)$ with weight
$v_{\lambda_r}$.  The unseen bound satisfies $U_t\le p_r/4$ once
\[
  t
  \ge
  C p_r^{-1}\log\{1/(\alpha_u p_r)\}.
\]
Let
\[
  A_t
  :=
  \left\{
    \frac34 p_rt
    \le
    N_t(r)
    \le
    \frac54 p_rt
  \right\}.
\]
Bernstein's inequality gives
\[
  \PP(A_t^c)\le 2\exp(-c_2p_rt).
\]
On $A_t$, the single LCB factor
\[
  M_t(q_0,\lambda_r)
  =
  \prod_{i=1}^t
  \{1+\lambda_r(\ind\{X_i=r\}-q_0)\}
\]
satisfies
\[
  \log M_t(q_0,\lambda_r)
  \ge
  c_3p_rt
\]
for a universal constant $c_3>0$.  Therefore
\[
  \widetilde M_t(q_0)
  \ge
  v_{\lambda_r}M_t(q_0,\lambda_r)
  \ge
  \alpha_r^{-1}
\]
whenever
\[
  t
  \ge
  C p_r^{-1}\log\{1/(\alpha_r v_{\lambda_r})\}.
\]
By Lemma~\ref{lem:lcb-monotonicity}, this implies
$L_t(r)\ge q_0>U_t$.  Combining this with the Bernstein tail and
Lemma~\ref{lem:det-log}, we obtain
\begin{equation}\label{eq:clean-rare-tail}
  \PP\{L_t(r)\le U_t\}
  \le
  (et)^{-2}
\end{equation}
for all
\[
  t\ge
  T_{\mathrm{rare}}
  :=
  C
  \frac{
    \log(1/\alpha_r)
    +
    \log(1/\alpha_u)
    +
    \log(1/v_{\lambda_r})
    +
    \log(1/p_r)
  }{p_r}.
\]

\smallskip
\noindent\textit{Step 3: combine.}
Let
\[
  T_{\mathrm{cl}}
  :=
  \max\{T_{\mathrm{pw}},T_{\mathrm{rare}}\}.
\]
For all $t\ge T_{\mathrm{cl}}$,
\[
  \PP(\tau^\circ>t)
  \le
  2(et)^{-2}.
\]
The tail-sum identity gives
\[
  \EE[\tau^\circ]
  \le
  T_{\mathrm{cl}}+1.
\]
Finally, substituting
$\alpha_{\mathrm{pw}}=\alpha_r=\alpha_u=\varepsilon/3$ and using the
assumption that $w_\lambda$ and $v_{\lambda_r}$ are bounded below by
constants,
\[
  \EE[\tau]
  \le
  \EE[\tau^\circ]
  =
  O\left(
    \frac{\log(1/\varepsilon)+\log(1/\delta)}{\delta^2}
    +
    \frac{\log(1/\varepsilon)+\log(1/p_r)}{p_r}
  \right),
\]
with constants independent of $\delta,p_r$, and $|\cA|$.
\end{proof}

\subsection{Oracle pairwise growth rate}

\begin{proposition}[Closed form for the oracle pairwise growth rate]
\label{prop:pairwise-oracle}
For $\lambda\in(0,1)$, define
\[
  \rho(\lambda)
  :=
  p_r\log(1+\lambda)+p_{(2)}\log(1-\lambda).
\]
Then $\rho$ is strictly concave on $(0,1)$.  If $p_{(2)}>0$, its unique
maximizer is
\[
  \lambda^\star
  =
  \frac{p_r-p_{(2)}}{p_r+p_{(2)}}
  =
  \frac{\delta}{p_r+p_{(2)}},
\]
and the maximal value is
\[
  \rho^\star
  :=
  \rho(\lambda^\star)
  =
  p_r\log\frac{2p_r}{p_r+p_{(2)}}
  +
  p_{(2)}\log\frac{2p_{(2)}}{p_r+p_{(2)}}.
\]
If $p_{(2)}=0$, then $\rho(\lambda)=p_r\log(1+\lambda)$ is increasing
on $(0,1)$ and
\[
  \sup_{0<\lambda<1}\rho(\lambda)=p_r\log2.
\]
\end{proposition}

\begin{proof}
We have
\[
  \rho'(\lambda)
  =
  \frac{p_r}{1+\lambda}
  -
  \frac{p_{(2)}}{1-\lambda},
  \qquad
  \rho''(\lambda)
  =
  -\frac{p_r}{(1+\lambda)^2}
  -
  \frac{p_{(2)}}{(1-\lambda)^2}
  <
  0.
\]
Thus $\rho$ is strictly concave.  When $p_{(2)}>0$, solving
$\rho'(\lambda)=0$ gives
$\lambda^\star=(p_r-p_{(2)})/(p_r+p_{(2)})$, and substituting this value
gives the stated formula for $\rho^\star$.  The case $p_{(2)}=0$ is
immediate.
\end{proof}

\begin{lemma}[Stopped KL identity for the swapped alternative]
\label{lem:stopped-kl-swap}
Assume $p_{(2)}>0$.  Fix
\[
  a^\star \in \arg\max_{a\neq r} p_a,
\]
and let $Q$ be obtained from $P$ by swapping the masses of $r$ and
$a^\star$:
\[
  q_r = p_{(2)},
  \qquad
  q_{a^\star}=p_r,
  \qquad
  q_a = p_a
  \quad\text{for } a \notin \{r,a^\star\}.
\]
Let $P^\tau$ and $Q^\tau$ denote the laws of the stopped experiment
$(\tau,X_1,\dots,X_\tau)$ under $P$ and $Q$.  If
$\EE_P[\tau] < \infty$, then
\[
  D(P^\tau\|Q^\tau)
  =
  \EE_P[\tau]
  \delta \log\frac{p_r}{p_{(2)}}.
\]
\end{lemma}
 
\begin{proof}
Let $\ell(x) := \log(P(x)/Q(x))$.  Since $\ell$ is bounded
(both $P$ and $Q$ have positive masses on the same support), the
stopping factor cancels in the likelihood ratio
$dP^\tau/dQ^\tau = \prod_{i \leq \tau} (P(X_i)/Q(X_i))$, and Wald's
identity for i.i.d.\ bounded summands yields
\[
  D(P^\tau \| Q^\tau)
   =  \EE_P \left[\sum_{i=1}^\tau \ell(X_i)\right]
   =  \EE_P[\tau]  \EE_P[\ell(X_1)].
\]
A direct calculation gives
$\EE_P[\ell(X_1)] = p_r \log(p_r/p_{(2)}) + p_{(2)} \log(p_{(2)}/p_r)
= \delta \log(p_r/p_{(2)})$, completing the proof.
\end{proof}
 
\begin{corollary}[Information-theoretic lower bound]
\label{cor:power-lower-bound}
Assume $p_{(2)}>0$.  Let $\varepsilon$ be the Type-I error level of a sequential test for
$H_{0,r}$, and assume that under the alternative $P$ one has
\[
  \PP_P(\tau<\infty)=1
  \qquad\text{and}\qquad
  \EE_P[\tau]<\infty.
\]
Then
\[
  \EE_P[\tau]
  \geq
  \frac{\log(1/\varepsilon)}
       {\delta\log(p_r/p_{(2)})}.
\]
In particular, if $\delta/p_r \to 0$, then
\[
  \EE_P[\tau]
  \geq
  \frac{p_r\log(1/\varepsilon)}{\delta^2}
  \bigl(1+o(1)\bigr).
\]
\end{corollary}
 
\begin{proof}
Let $Q$ be the swap from Lemma~\ref{lem:stopped-kl-swap}, so
$Q \in H_{0,r}$ and Type-I validity gives
$Q(\tau < \infty) \leq \varepsilon$.  Set $A := \{\tau < \infty\}$, so
$P(A) = 1$.  By data processing followed by
Lemma~\ref{lem:stopped-kl-swap},
\[
  \log\frac{1}{\varepsilon}
   \leq  \mathrm{kl}\bigl(P(A), Q(A)\bigr)
   \leq  D(P^\tau \| Q^\tau)
   =  \EE_P[\tau]  \delta \log(p_r/p_{(2)}),
\]
which is the first claim.  For the small-gap expansion,
$\delta \log(p_r/p_{(2)}) = \delta \log(1 + \delta/p_{(2)})
= \delta^2/p_r + o(\delta^2/p_r)$ as $\delta/p_r \to 0$.
\end{proof}

\begin{lemma}[Forward I-projection of $P$ onto $H_{0,r}$]
\label{lem:iproj}
Assume $p_{(2)} > 0$ and pick any $a^\star \in \arg\max_{a \neq r} p_a$
(this set may have ties).  Every forward I-projection
\[
  P^\star
   \in 
  \arg\min_{Q \in H_{0,r}} D(P\|Q)
\]
arises by midpointing $r$ with one of the strongest competitors; for
the choice $a^\star$ above, the corresponding projection is the
\emph{midpoint distribution}
\[
  p_r^\star  =  p_{a^\star}^\star  =  \frac{p_r + p_{(2)}}{2}, \qquad
  p_a^\star  =  p_a \quad\text{for } a \notin \{r, a^\star\},
\]
and the per-sample KL divergence is
\[
  D(P \| P^\star)
   = 
  p_r \log \frac{2 p_r}{p_r + p_{(2)}}
   + 
  p_{(2)} \log \frac{2 p_{(2)}}{p_r + p_{(2)}}.
\]
In particular, as $\delta/p_r \to 0$,
\[
  D(P \| P^\star)  =  \frac{\delta^2}{4 p_r}\bigl(1 + o(1)\bigr).
\]
\end{lemma}

\begin{proof}
$H_{0,r} = \bigcup_{a\neq r}\{q_r \leq q_a\}$ is a (possibly countable)
union of half-spaces, hence not itself convex.  We minimise
$D(P\|\cdot)$ over each face $\{q_r \leq q_a\}$ separately (on each
face $D(P\|\cdot)$ is strictly convex and strictly positive since
$P \notin H_{0,r}$), and then take the smallest of the per-face
minima.  The optimum on each face lies on its boundary
$\{q_r = q_a\}$.  We organise the proof in three steps.

\smallskip
\noindent\emph{Step 1 (per-pair midpoint projection).}
For each $a \neq r$ with $p_a < p_r$, the I-projection of $P$ onto
the affine subspace $\{q_r = q_a\}$ is given in closed form by
\begin{equation}\label{eq:per-pair-proj}
  q_r = q_a = \frac{p_r + p_a}{2},
  \qquad
  q_b = p_b \quad \text{for } b \notin \{r, a\}.
\end{equation}
The Lagrangian on the simplex with multiplier $\mu$ has
first-order conditions $q_b = p_b/\mu$ and $q_r = q_a = (p_r+p_a)/(2\mu)$;
the normalisation $\sum_b q_b = 1/\mu = 1$ forces $\mu = 1$,
yielding~\eqref{eq:per-pair-proj}, and the positive-definite Hessian
$\mathrm{diag}(p/q^2)$ ensures uniqueness.  The KL value of this
projection is
\begin{equation}\label{eq:Phi-pa}
  \Phi(p_a)
   :=  p_r \log \frac{2 p_r}{p_r+p_a}
        + p_a \log \frac{2 p_a}{p_r+p_a}.
\end{equation}

\smallskip
\noindent\emph{Step 2 (the optimal competitor is $a^\star$).}
Differentiating~\eqref{eq:Phi-pa} gives
$\Phi'(p_a) = \log \frac{2 p_a}{p_r+p_a} \leq 0$ on $(0, p_r)$, so
$\Phi$ is non-increasing in $p_a$ and is minimised over $\{p_a : a \neq r\}$
at $p_a = p_{(2)} = p_{a^\star}$.  After midpointing on $a^\star$, the
remaining masses satisfy
$q_b = p_b \leq p_{(2)} < (p_r+p_{(2)})/2 = q_r = q_{a^\star}$ for
$b \notin \{r, a^\star\}$, confirming that the candidate lies in
$H_{0,r}$.  Hence midpointing $r$ with the chosen $a^\star$ yields a
global forward I-projection $P^\star$ with
$D(P\|P^\star) = \Phi(p_{(2)})$ (in case of ties at $p_{(2)}$ each
choice of strongest competitor gives a distinct minimiser with the
same KL value), proving the first two displays.

\smallskip
\noindent\emph{Step 3 (small-gap expansion).}
Write $q^\star := (p_r+p_{(2)})/2 = p_r - \delta/2$ and
$x := \delta/(2 q^\star) \in (0, 1)$, so $p_r/q^\star = 1+x$ and
$p_{(2)}/q^\star = 1-x$.  The closed form factorises as
$\Phi(p_{(2)}) = q^\star f(x)$, where
\[
  f(x)  :=  (1+x) \log(1+x) + (1-x) \log(1-x).
\]
$f$ is even with $f''(x) = 2/(1-x^2) = 2 + 2 x^2 + O(x^4)$, hence
$f(x) = x^2 + x^4/6 + O(x^6)$ about $0$.  Substituting,
\[
  D(P \| P^\star)
   =  q^\star x^2 + O(q^\star x^4)
   =  \frac{\delta^2}{4 q^\star} + O \left(\frac{\delta^4}{q^{\star 3}}\right),
\]
and $q^\star = p_r(1 - \delta/(2 p_r)) = p_r(1+o(1))$ as $\delta/p_r \to 0$
yields the third display.
\end{proof}

\begin{lemma}[Stopped KL identity for the I-projection]
\label{lem:stopped-kl-iproj}
Under the assumptions of Lemma~\ref{lem:iproj}, if $\EE_P[\tau] < \infty$,
then
\[
  D(P^\tau\|(P^\star)^\tau)
   = 
  \EE_P[\tau]  D(P\|P^\star).
\]
\end{lemma}

\begin{proof}
$P$ and $P^\star$ are mutually absolutely continuous on $\cA$
(since $q^\star \in (0,1)$ when $p_r \in (0,1)$, and other masses
agree), and the log-likelihood
$\ell(x) := \log(P(x)/P^\star(x))$ is bounded.  The stopping-event
factor cancels in the likelihood ratio, so Fubini and independence
give $D(P^\tau\|(P^\star)^\tau) = \EE_P[\tau]  \EE_P[\ell(X_1)]
= \EE_P[\tau]  D(P\|P^\star)$,
exactly as in Lemma~\ref{lem:stopped-kl-swap}.
\end{proof}

\begin{corollary}[Sharp information-theoretic lower bound]
\label{cor:power-lower-bound-tight}
Under the hypotheses of Corollary~\ref{cor:power-lower-bound},
\[
  \EE_P[\tau]
   \geq 
  \frac{\log(1/\varepsilon)}{D(P\|P^\star)}
   = 
  \frac{4 p_r \log(1/\varepsilon)}{\delta^2}\bigl(1 + o(1)\bigr)
  \qquad (\delta/p_r \to 0).
\]
\end{corollary}

\begin{proof}
Apply the same data-processing argument as in
Corollary~\ref{cor:power-lower-bound} with $Q = P^\star$ in place of
the swap distribution.  Since $P^\star \in H_{0,r}$
(Lemma~\ref{lem:iproj}), Type-I validity gives
$P^\star(\tau<\infty) \leq \varepsilon$, and the data-processing
inequality combined with Lemma~\ref{lem:stopped-kl-iproj} yields
$\log(1/\varepsilon) \leq \EE_P[\tau]  D(P\|P^\star)$.
The asymptotic constant follows from Lemma~\ref{lem:iproj}.
\end{proof}

\subsection{Information-Theoretic Lower Bound for the Unseen-Category Component}
\label{app:unseen-lower-bound}
 
\begin{lemma}[Stopped KL identity for a hidden-competitor alternative]
\label{lem:stopped-kl-hidden}
Fix a distribution $P$ on a countable category set $\cA$ with $p_r \in (0,1)$,
and let $h \in \cA$ satisfy $h \neq r$ and $p_h = 0$.
Define a second distribution $Q$ by
\[
  q_r = \frac{p_r}{2},
  \qquad
  q_h = \frac{p_r}{2},
  \qquad
  q_a = p_a
  \quad\text{for } a \notin \{r,h\}.
\]
Let $P^\tau$ and $Q^\tau$ denote the laws of the stopped experiment
$(\tau,X_1,\dots,X_\tau)$ under $P$ and $Q$.
If $\EE_P[\tau] < \infty$, then
\[
  D(P^\tau\|Q^\tau)
  =
  \EE_P[\tau]p_r\log 2.
\]
\end{lemma}
 
\begin{proof}
By the same Wald argument as Lemma~\ref{lem:stopped-kl-swap},
$D(P^\tau \| Q^\tau) = \EE_P[\tau] \cdot \EE_P[\ell(X_1)]$ with
$\ell(x) := \log(P(x)/Q(x))$, where we set $\ell(h) := 0$ (an arbitrary
finite value on the $P$-null event $\{X = h\}$ that does not affect any
$P$-expectation).  By construction $\ell(r) = \log 2$ and
$\ell(x) = 0$ for $x \neq r$, so
$\EE_P[\ell(X_1)] = p_r \log 2$.
\end{proof}
 
\begin{corollary}[Information-theoretic lower bound for the unseen-category component]
\label{cor:unseen-lower-bound}
Fix $p \in (0,1)$ and $\delta \in (0,p)$, and assume the category set
$\cA$ contains at least $\lceil (1-p)/(p-\delta) \rceil + 2$ distinct
non-target categories (so that an unused ``hidden'' label is
available; this holds automatically on countably infinite category sets).
Define
\[
  \cP(p,\delta)
  :=
  \left\{
    P :
    p_r = p,
    p_{(2)} \le p-\delta
  \right\}.
\]
Let $\tau$ be any stopping time corresponding to a sequential test such that
\[
  \sup_{Q \in H_{0,r}} \PP_Q(\tau<\infty) \le \varepsilon,
\]
and assume that, for every $P \in \cP(p,\delta)$,
\[
  \PP_P(\tau<\infty)=1
  \qquad\text{and}\qquad
  \EE_P[\tau]<\infty.
\]
Then
\[
  \sup_{P \in \cP(p,\delta)} \EE_P[\tau]
  \ge
  \frac{\log(1/\varepsilon)}{p\log 2}.
\]
\end{corollary}
 
\begin{proof}
Pick $m \geq (1-p)/(p-\delta)$ distinct categories $b_1, \dots, b_m, h
\in \cA \setminus \{r\}$ and set $P(r) = p$, $P(h) = 0$,
$P(b_j) = (1-p)/m$, so $P \in \cP(p, \delta)$.
The hidden-competitor distribution
$Q(r) = Q(h) = p/2$, $Q(b_j) = (1-p)/m$ lies in $H_{0,r}$, so
$Q(\tau < \infty) \leq \varepsilon$.  Data processing combined with
Lemma~\ref{lem:stopped-kl-hidden} then gives
$\log(1/\varepsilon) \leq D(P^\tau \| Q^\tau) = \EE_P[\tau]  p \log 2$.
\end{proof}

\begin{proof}[Proof of Theorem~\ref{thm:minimax-combined-lower-bound}]
Fix $p\in(0,1/2]$, $\delta\in(0,p)$, and a level-$\varepsilon$
sequential certifier $\tau'$ satisfying the assumptions of the
corollary.  For notational simplicity, write $\tau:=\tau'$ within this
proof, and let
\[
  \cP(p,\delta)
  :=
  \left\{
    P:
    p_r=p,\quad p_{(2)}\le p-\delta
  \right\}.
\]
We combine two distinct least favorable subfamilies of $\cP(p,\delta)$.
 
\medskip
\noindent\textbf{Step 1: Discrimination-hard instance.}
Because $p \le 1/2$, the quantity $1-2p+\delta$ is nonnegative.
Choose an integer
\[
  m \ge \frac{1-2p+\delta}{p-\delta},
\]
and distinct categories
\[
  a^\star,b_1,\dots,b_m \in \cA\setminus\{r\}.
\]
Define $P_{\mathrm{pw}}$ by
\[
  P_{\mathrm{pw}}(r)=p,
  \qquad
  P_{\mathrm{pw}}(a^\star)=p-\delta,
  \qquad
  P_{\mathrm{pw}}(b_j)=\frac{1-2p+\delta}{m}
  \quad (j=1,\dots,m),
\]
and $P_{\mathrm{pw}}(a)=0$ otherwise.
Then $P_{\mathrm{pw}} \in \cP(p,\delta)$ and its runner-up satisfies
$p_{(2)}=p-\delta>0$.
Corollary~\ref{cor:power-lower-bound} therefore gives
\[
  \EE_{P_{\mathrm{pw}}}[\tau]
  \ge
  \frac{\log(1/\varepsilon)}
       {\delta \log\bigl(p/(p-\delta)\bigr)}.
\]
If $\delta \le p/2$, then
\[
  \log\frac{p}{p-\delta}
  =
  -\log\left(1-\frac{\delta}{p}\right)
  \le
  \frac{\delta/p}{1-\delta/p}
  \le
  \frac{2\delta}{p},
\]
whence
\[
  \EE_{P_{\mathrm{pw}}}[\tau]
  \ge
  \frac{p\log(1/\varepsilon)}{2\delta^2}.
\]
 
\medskip
\noindent\textbf{Step 2: Exploration-hard instance.}
Corollary~\ref{cor:unseen-lower-bound} provides another distribution
$P_{\mathrm{un}} \in \cP(p,\delta)$ such that
\[
  \EE_{P_{\mathrm{un}}}[\tau]
  \ge
  \frac{\log(1/\varepsilon)}{p\log 2}.
\]
 
\medskip
\noindent\textbf{Step 3: Combine the two bounds.}
If $\delta \le p/2$, then $\max\{x,y\} \ge (x+y)/2$ applied to the
two lower bounds gives the result directly.
If $\delta > p/2$, then $p/\delta^2 \le 4/p$, so
$p/\delta^2 + 1/p \le 5/p$, and
Corollary~\ref{cor:unseen-lower-bound} alone gives
$\sup_P \EE_P[\tau] \ge \log(1/\varepsilon)/(p\log 2)
\ge (5\log 2)^{-1}\log(1/\varepsilon)(p/\delta^2+1/p)$.
Since $(5\log 2)^{-1} > 1/4$, the same constant works in both cases.
\end{proof}

\section{Detailed Power Analysis for the W-CITE}
\label{app:weighted-power-details}
 
This appendix contains the proofs for the confidence-weighted extension
in Section~\ref{sec:weighted-main}.  We keep the notation parallel to
the unweighted analysis, replacing $(p_r,p_{(2)},\delta)$ by
$(\mu_r,\mu_\star,\Delta_w)$ where appropriate.

\begin{lemma}[Expected maximum weighted competitor penalty]
\label{lem:weighted-max-rival}
Fix $\lambda \in (0,1)$ and define
$g_\lambda(w) := -\log(1-\lambda w)$ for $w \in [0,1]$.
For each competitor $a \neq r$, let
$H_t^{(\lambda)}(a) := \sum_{i=1}^t g_\lambda(W_i)\ind\{X_i=a\}$,
$v_\lambda(a) := \EE[g_\lambda(W)\ind\{X=a\}]$, and set
$R_t^{(\lambda)} := \max_{a \neq r} H_t^{(\lambda)}(a)$,
$v_\lambda^\star := \sup_{a \neq r} v_\lambda(a)$.
Then, for every integer $t \geq 1$,
\[
  \EE[R_t^{(\lambda)}]
  \leq
  t v_\lambda^\star
  +
  2\sqrt{2}
  \bigl(-\log(1-\lambda)\bigr)
  \sqrt{t\log(t+1)}.
\]
\end{lemma}

\begin{lemma}[Weighted LCB inversion at a fixed time]
\label{lem:weighted-lcb-inversion}
Fix an integer $t \geq 1$ and a number $q \in (0,1]$.
\textup{(a)}~If $\widetilde M_t(q) > 1$, then $\widehat \mu_t(r) > q$.
\textup{(b)}~If $\widetilde M_t(q) \geq 1/\alpha_r$, then $\widetilde L_t(r) \geq q$.
\end{lemma}

\subsection{Weighted pairwise NSM property (Section~\ref{sec:weighted}, item~(i))}
\label{app:proof-prop:weighted-pairwise-nsm}

\begin{proof}
Fix $a\neq r$ and $\lambda\in(0,1)$.  Since $W_i\in[0,1]$, each factor
\[
  1+\lambda W_i\bigl(\ind\{X_i=r\}-\ind\{X_i=a\}\bigr)
\]
is nonnegative.  Under $\mu_r\le\mu_a$,
\[
  \EE\left[
    1+\lambda W_i\bigl(\ind\{X_i=r\}-\ind\{X_i=a\}\bigr)
    \mid\cF_{i-1}
  \right]
  =
  1+\lambda(\mu_r-\mu_a)
  \le
  1,
\]
using independence of $(X_i,W_i)$ from $\cF_{i-1}$.  Therefore
$(\widetilde E_t^{(r,a)}(\lambda))_{t\ge0}$ is an NSM.
\end{proof}

\subsection{Weighted mixture NSM (Section~\ref{sec:weighted}, item~(i))}
\label{app:proof-cor:weighted-mixture-nsm}
 
\begin{proof}
For each fixed $a\neq r$ and each
$\lambda\in\Lambda_{\mathrm{pw}}$,
$(\widetilde E_t^{(r,a)}(\lambda))_{t\ge0}$ is an NSM under
$\mu_r\le\mu_a$ by the preceding proof.  Since the weights
$(w_\lambda)$ are nonnegative and sum to at most one,
\[
  \widetilde E_t^{(r,a)}
  =
  \sum_{\lambda\in\Lambda_{\mathrm{pw}}}
  w_\lambda \widetilde E_t^{(r,a)}(\lambda)
\]
is also an NSM with initial value at most one.  Ville's inequality gives
\[
  \PP_P\left(
    \sup_{t\ge0}
    \widetilde E_t^{(r,a)}
    \ge
    \alpha^{-1}
  \right)
  \le
  \alpha .
\]
\end{proof}
 
\subsection{Weighted LCB monotonicity (Section~\ref{sec:weighted}, item~(ii))}
\label{app:proof-lem:weighted-lcb-monotone}
 
\begin{proof}
The proof is identical to the monotonicity part of
Lemma~\ref{lem:lcb-monotonicity}:
$\Lambda_r(q_2) \subseteq \Lambda_r(q_1)$ and each factor
$1+\lambda(W_i\ind\{X_i=r\}-q)$ is pointwise larger when $q$ is
smaller, so $\widetilde M_t(q_1) \geq \widetilde M_t(q_2)$.
\end{proof}

\subsection{Weighted LCB validity (Section~\ref{sec:weighted}, item~(ii))}
\label{app:proof-prop:weighted-lcb-valid}
 
\begin{proof}
If $\mu_r=0$, then $W_i\ind\{X_i=r\}=0$ almost surely for all $i$, so
$\widehat\mu_t(r)=0$ and $\widetilde L_t(r)=0$ almost surely.  Hence
$\{\exists t:\widetilde L_t(r)>\mu_r\}$ is empty.  Assume
$\mu_r>0$ in the remainder.

The proof parallels that of Proposition~\ref{prop:lcb-valid}
(Appendix~\ref{app:proof-prop:lcb-valid}), with the following
substitutions: $\ind\{X_i=r\}$ becomes $W_i\ind\{X_i=r\}$,
$p_r$ becomes $\mu_r$, $\widehat p_t(r)$ becomes $\widehat\mu_t(r)$,
and the mixture weights $(w_\lambda)$ become $(v_\lambda)$.
At the true value $q=\mu_r$, one has
$\EE[W_i\ind\{X_i=r\}-\mu_r]=0$, so each
$\widetilde M_t(\mu_r,\lambda)$ is a nonnegative martingale with unit initial
value, and Ville's inequality gives
$\PP(\sup_t \widetilde M_t(\mu_r) \geq 1/\alpha_r) \leq \alpha_r$.
The monotonicity from
the weighted LCB monotonicity property then implies
$\{\exists t:\widetilde L_t(r)>\mu_r\} \subseteq
\{\sup_t \widetilde M_t(\mu_r) \geq 1/\alpha_r\}$.
\end{proof}

\subsection{Weighted unseen bound validity (Section~\ref{sec:weighted}, item~(iii))}
\label{app:proof-prop:weighted-unseen-valid}
 
\begin{proof}
Since $\mu_a \leq p_a$ for every $a \in \cA$,
\[
  \Bigl\{
    \exists a,\exists t :
    N_t(a)=0
    \text{ and }
    \mu_a > U_t
  \Bigr\}
  \subseteq
  \Bigl\{
    \exists a,\exists t :
    N_t(a)=0
    \text{ and }
    p_a > U_t
  \Bigr\}.
\]
The probability of the latter event is at most $\alpha_u$ by
Proposition~\ref{prop:unseen-valid}.
\end{proof}

\subsection{Proof of Theorem~\ref{thm:weighted-main}}
\label{app:proof-thm:weighted-main}

\begin{proof}
The proof parallels that of Theorem~\ref{thm:main-iut}
(Appendix~\ref{app:proof-thm:main-iut}), with $p_a$ replaced by $\mu_a$
throughout and Lemma~\ref{lem:monotone} replaced by the
all-competitors check (runner-up monotonicity fails for weighted
observations).

Fix $P\in H_{0,r}^{(w)}$.  If $\cA\setminus\{r\}=\emptyset$, then the
weighted null is vacuous.  Otherwise, since
\[
  \sum_{a\in\cA}\mu_a=\EE[W]\le1,
\]
Lemma~\ref{lem:attained-supremum} implies that any positive value of
$\mu_\star=\sup_{a\neq r}\mu_a$ is attained.  If $\mu_\star>0$, choose
$a^\star\neq r$ with $\mu_{a^\star}=\mu_\star\ge\mu_r$.  If
$\mu_\star=0$, then $P\in H_{0,r}^{(w)}$ implies $\mu_r=0$, and any
$a^\star\neq r$ is a valid witness with
$\mu_{a^\star}\ge\mu_r$.
Decompose
\[
  \{\tau^{(w)} < \infty\}
  \subseteq
  \underbrace{\{\tau^{(w)} < \infty,\ a^\star \in \cA_{\tau^{(w)}}\}}_{E_1}
  \cup
  \underbrace{\{\tau^{(w)} < \infty,\ a^\star \notin \cA_{\tau^{(w)}}\}}_{E_2}.
\]

\medskip
\noindent\textbf{Step~1 (bounding $\PP(E_1)$).}
On $E_1$, the weighted pairwise stopping condition at $\tau^{(w)}$
applied to \emph{every} observed competitor gives
$\widetilde E_{\tau^{(w)}}^{(r,a^\star)} \geq 1/\alpha_{\mathrm{pw}}$.
The process $(\widetilde E_t^{(r,a^\star)})_{t \geq 0}$ is an NSM under
$\mu_r \leq \mu_{a^\star}$ by the weighted mixture NSM property
(Section~\ref{sec:weighted}, item~(i)), with initial value $\leq 1$.
Ville's inequality~\eqref{eq:ville} yields $\PP(E_1) \leq \alpha_{\mathrm{pw}}$.

\medskip
\noindent\textbf{Step~2 (bounding $\PP(E_2)$).}
Define
\[
  F_r := \{\exists t \geq 1 : \widetilde L_t(r) > \mu_r\},
  \quad
  F_u := \{\exists a \in \cA,\ \exists t \geq 1 :
           N_t(a) = 0,\ \mu_a > U_t\}.
\]
Weighted LCB validity gives $\PP(F_r) \leq \alpha_r$, and weighted
unseen-bound validity gives $\PP(F_u) \leq \alpha_u$.
On $E_2 \setminus (F_r \cup F_u)$, the LCB--unseen part of the weighted
stopping rule gives
$\widetilde L_{\tau^{(w)}}(r)>U_{\tau^{(w)}}$; combined with
$F_r^c$ and $F_u^c$ (recall $a^\star$ is unseen at $\tau^{(w)}$),
\[
  \mu_{a^\star} \leq U_{\tau^{(w)}}
  < \widetilde L_{\tau^{(w)}}(r) \leq \mu_r,
\]
contradicting $\mu_{a^\star} \geq \mu_r$.  Hence
$E_2 \subseteq F_r \cup F_u$ and $\PP(E_2) \leq \alpha_r + \alpha_u$.

\medskip
\noindent\textbf{Step~3 (union bound).}
$\PP(\tau^{(w)} < \infty) \leq \PP(E_1) + \PP(E_2)
 \leq \alpha_{\mathrm{pw}} + \alpha_r + \alpha_u \leq \varepsilon$.
\end{proof}

\subsection{Proof of Theorem~\ref{thm:weighted-clean-rate}: almost-sure stopping part}
\label{app:proof-thm:weighted-consistency}
 
\begin{proof}
Lemma~\ref{prop:weighted-smallbet-pair} (a) gives
$\gamma_{\mathrm{pw}}
:= \inf_{a\neq r}\EE[\log(1+\lambda_{\mathrm{pw}}W(
\ind\{X=r\}-\ind\{X=a\}))]
\geq \lambda_{\mathrm{pw}}\Delta_w/2 > 0$,
and Lemma~\ref{prop:weighted-smallbet-pair} (c) with $q_0=\mu_r/2$
gives
$\nu_r := \EE[\log(1+\lambda_r(W\ind\{X=r\}-q_0))]
\geq \lambda_r\mu_r/4 > 0$.
We now show that these positive drifts force $\PP(\tau^{(w)}<\infty)=1$.
 
Fix $\lambda := \lambda_{\mathrm{pw}}$, and let
$w_{\mathrm{pw}}$ denote its mixture weight in
Section~\ref{sec:weighted}, item~(i).
 
\medskip

\noindent\textbf{Step 1: Pairwise condition.}
For each $i$, set
$A_i := \log(1+\lambda W_i)\ind\{X_i=r\}$ and
$B_i(a) := -\log(1-\lambda W_i)\ind\{X_i=a\}$, so that
\begin{equation}
\label{eq:weighted-log-pairwise}
  \log \widetilde E_t^{(r,a)}(\lambda)
  =
  \textstyle\sum_{i=1}^t A_i - \sum_{i=1}^t B_i(a).
\end{equation}
Write $u := \EE[A_1]$, $v_a := \EE[B_1(a)]$, and
$v^\star := \sup_{a\neq r}v_a$.
The definition of $\gamma_{\mathrm{pw}}$ gives
$u - v^\star \geq \gamma_{\mathrm{pw}} > 0$.
 
We control $\max_{a\neq r}t^{-1}\sum_i B_i(a)$ on the countable
category set by a tail-truncation argument.
Since $\sum_{a\neq r}v_a \leq -\log(1-\lambda)<\infty$,
for each $m\geq 1$ there is a finite $F_m\subset\cA\setminus\{r\}$
with $\sum_{a\notin F_m}v_a < 1/m$.
Define the tail process
$R_t^{(m)} := \sum_i -\log(1-\lambda W_i)\ind\{X_i\notin F_m\cup\{r\}\}$.
By the Strong Law of Large Numbers, $t^{-1}\sum_i A_i \to u$,
$t^{-1}\sum_i B_i(a) \to v_a$ for each fixed~$a$,
and $t^{-1}R_t^{(m)} \to \sum_{a\notin F_m}v_a < 1/m$,
all almost surely.
On the probability-one intersection of these countably many events,
$\limsup_{t}\max_{a\neq r}t^{-1}\sum_i B_i(a)
\leq \max\{v^\star,1/m\}$ for every $m$;
letting $m\to\infty$ gives the limsup $\leq v^\star$.
Combining with~\eqref{eq:weighted-log-pairwise},
\[
\liminf_t \min_{a\neq r}t^{-1}\log \widetilde E_t^{(r,a)}(\lambda)
\geq u - v^\star \geq \gamma_{\mathrm{pw}} > 0.
\]
Since $\cA_t\setminus\{r\}\subseteq\cA\setminus\{r\}$ and
$\widetilde E_t^{(r,a)}\geq w_{\mathrm{pw}}\widetilde E_t^{(r,a)}(\lambda)$,
the positive liminf of the normalized log e-process implies
there exists a.s.\ a finite $T_{\mathrm{pw}}'$ after which
$\min_{a\in\cA_t\setminus\{r\}}\widetilde E_t^{(r,a)}
\geq 1/\alpha_{\mathrm{pw}}$.
 
\medskip
\noindent\textbf{Step 2: Weighted LCB condition.}
Fix $\lambda := \lambda_r$ with mixture weight $v_{\lambda_r} > 0$ and
set $C_i := \log(1 + \lambda(W_i \ind\{X_i = r\} - q_0))$;
each $C_i$ is bounded since $\lambda < 1/q_0$ (Lemma~\ref{prop:weighted-smallbet-pair}(c))
and $W_i \ind\{X_i = r\} \in [0,1]$.
Part~(c) of Lemma~\ref{prop:weighted-smallbet-pair} gives
$\EE[C_i] = \nu_r \geq \lambda \mu_r / 4 > 0$, so by the SLLN
$t^{-1} \log \widetilde M_t(q_0, \lambda)
  = t^{-1} \sum_{i=1}^t C_i
  \to \nu_r$
a.s., hence $\widetilde M_t(q_0, \lambda) \to \infty$ and
$\widetilde M_t(q_0) \geq v_{\lambda_r} \widetilde M_t(q_0, \lambda) \geq 1/\alpha_r$
eventually a.s.  An independent SLLN on $W_i \ind\{X_i = r\}$ gives
$\widehat\mu_t(r) \to \mu_r > q_0$, so $\widehat\mu_t(r) \geq q_0$
eventually, and the weighted LCB construction then yields
$\widetilde L_t(r) \geq q_0$ for all large~$t$ almost surely.  Let $T'_r$ denote
the resulting (random, a.s.\ finite) eventual time after which both
$\widetilde M_t(q_0) \geq 1/\alpha_r$ and $\widetilde L_t(r) \geq q_0$ hold.
 
\medskip

\noindent\textbf{Step 3: Unseen upper bound.}
The sequence $(U_t)_{t \geq 1}$ is deterministic and
$U_t \downarrow 0$ as $t \to \infty$.  Since $q_0 > 0$, there exists a
deterministic index $T_u$ such that
\[
  U_t < q_0
  \qquad \forall t \geq T_u.
\]
 
\medskip
 
\noindent\textbf{Step 4: Combine.}
Choose $t$ large enough that
$t \geq T_{\mathrm{pw}}' \vee T'_r \vee T_u$.
Then the pairwise threshold is met,
$\widetilde L_t(r) \geq q_0 > U_t$,
and both stopping conditions hold, giving
$\PP(\tau^{(w)} < \infty) = 1$.
\end{proof}
 
\begin{lemma}[Weighted small-bet drift bounds]
\label{prop:weighted-smallbet-pair}%
\label{prop:weighted-smallbet-lcb}%
\label{cor:weighted-universal-lcb}
Assume $\Delta_w > 0$.
\begin{enumerate}[label=\textup{(\alph*)},leftmargin=2.2em,itemsep=2pt,topsep=2pt]
\item \textbf{Pairwise:}
  if $\lambda \in \Lambda_{\mathrm{pw}}$ satisfies
  $\lambda \leq \Delta_w / \bigl(2(\mu_r + \mu_\star)\bigr)$, then
  \[
    \inf_{a \neq r}
    \EE\bigl[\log\bigl(1 + \lambda W(\ind\{X=r\} - \ind\{X=a\})\bigr)\bigr]
     \geq  \lambda \Delta_w / 2  >  0.
  \]
\item \textbf{LCB:}
  for any $q_0 \in (0, \mu_r)$ and any $\lambda \in \Lambda_r$ with
  $\lambda \leq (\mu_r - q_0)/(2\mu_r)$, one has $\lambda < 1/q_0$ and
  \[
    \EE\bigl[\log\bigl(1 + \lambda(W\ind\{X=r\} - q_0)\bigr)\bigr]
     \geq  \lambda(\mu_r - q_0)/2  >  0.
  \]
\item \textbf{Universal LCB grid point:}
  with $q_0 := \mu_r/2$ in part~(b), the condition reduces to
  $\lambda_r \in (0, 1/4]$, and
  \[
    \EE\bigl[\log\bigl(1 + \lambda_r(W\ind\{X=r\} - \mu_r/2)\bigr)\bigr]
     \geq  \lambda_r \mu_r / 4  >  0.
  \]
\end{enumerate}
\end{lemma}

\begin{proof}
All three parts follow from the quadratic lower bound
$\log(1+\xi) \geq \xi - \xi^2$ on $[-1/2, 1/2]$
(Lemma~\ref{lem:power-log-lower}), exactly as in
Proposition~\ref{prop:power-smallbet-pair}, with two additional
ingredients: $W \in [0,1]$ implies $W^2 \leq W$, and the squared
indicators satisfy
$(\ind\{X=r\} - \ind\{X=a\})^2 = \ind\{X=r\} + \ind\{X=a\}$.

\smallskip
\noindent\textbf{(a)}
Set $\xi_a := \lambda W (\ind\{X=r\} - \ind\{X=a\})$, so
$|\xi_a| \leq \lambda \leq 1/2$.  Taking expectations of
$\log(1+\xi_a) \geq \xi_a - \xi_a^2$ and using $W^2 \leq W$,
\[
  \EE[\log(1+\xi_a)]
   \geq  \lambda(\mu_r - \mu_a) - \lambda^2(\mu_r + \mu_a)
   \geq  \lambda \Delta_w - \lambda^2(\mu_r + \mu_\star).
\]
The hypothesis gives
$\lambda^2(\mu_r + \mu_\star) \leq \lambda \Delta_w / 2$, hence
$\EE[\log(1+\xi_a)] \geq \lambda \Delta_w / 2$ uniformly in $a$.

\smallskip
\noindent\textbf{(b)}
Since $\mu_r \leq 1$, $(\mu_r - q_0)/(2\mu_r) < 1/2 < 1/q_0$, so the
hypothesis implies $\lambda < 1/2$ and $\lambda < 1/q_0$.
Setting $Y := W \ind\{X=r\}$ and $\xi := \lambda(Y - q_0)$, the same
quadratic bound and $Y^2 \leq Y$ yield
\[
  \EE[\log(1+\xi)]
   \geq  \lambda(\mu_r - q_0) - \lambda^2 \mu_r
   \geq  \lambda(\mu_r - q_0) / 2.
\]

\smallskip
\noindent\textbf{(c)}
Specializing (b) to $q_0 = \mu_r/2$ gives
$(\mu_r - q_0)/(2\mu_r) = 1/4$, so $\lambda_r \leq 1/4$ verifies the
hypothesis and the conclusion follows.
\end{proof}

\subsection{Proof of Lemma~\ref{lem:weighted-max-rival}}
\label{app:proof-lem:weighted-max-rival}
 
\begin{proof}
The proof follows the same symmetrization argument as
Lemma~\ref{lem:max-count}, applied to the bounded function class
$\{f_a(x,w) := g_\lambda(w)\ind\{x=a\} : a \neq r\}$
in place of $\{\ind\{x=a\}\}$.  The key modifications are:
\begin{itemize}
\item The centering term $tp_{(2)}$ becomes $tv_\lambda^\star$,
      where $v_\lambda^\star := \sup_{a\neq r}\EE[g_\lambda(W)\ind\{X=a\}]$.
\item The per-coordinate bound $\|v_a\|_\infty \leq 1$ becomes
      $\|v_a\|_\infty \leq g_\lambda(1) = -\log(1-\lambda)$,
      so the sub-Gaussian parameter scales by $-\log(1-\lambda)$.
\item The Rademacher complexity bound remains
      $O(\sqrt{t\log(t+1)})$ because there are still at most $t+1$
      distinct vectors in the function class evaluated on $t$ samples.
\end{itemize}
Optimizing the Donsker--Varadhan parameter $s$ as in
Lemma~\ref{lem:max-count} yields the stated bound.
\end{proof}

\subsection{Proof of Lemma~\ref{lem:weighted-lcb-inversion}}
\label{app:proof-lem:weighted-lcb-inversion}
 
\begin{proof}
\textbf{(a)}
Assume $\widehat\mu_t(r)\leq q$.
Set $Y_i := W_i\ind\{X_i=r\}$.
By Jensen's inequality (concavity of $\log$),
$t^{-1}\log \widetilde M_t(q,\lambda)
\leq \log(1+\lambda(\widehat\mu_t(r)-q))
\leq 0$
for every $\lambda\in\Lambda_r(q)$.
Summing gives $\widetilde M_t(q)\leq\sum v_\lambda\leq 1$,
contradicting $\widetilde M_t(q)>1$.
 
\medskip
\noindent\textbf{(b)}
If $\widetilde M_t(q)\geq 1/\alpha_r > 1$, part~(a) gives
$\widehat\mu_t(r)>q$, so $q\in(0,\widehat\mu_t(r)]$ and
$q$ belongs to the set defining $\widetilde L_t(r)$; hence $\widetilde L_t(r)\geq q$.
\end{proof}

\subsection{Proof of Theorem~\ref{thm:weighted-clean-rate}: expected stopping-time bound}
\label{app:proof-thm:weighted-clean-rate}

\begin{proof}
We use the same structure as the proof of
Theorem~\ref{thm:power-iut-clean}, with
$(p_r,p_{(2)},\delta)$ replaced by
$(\mu_r,\mu_\star,\Delta_w)$.

\smallskip
\noindent\textit{Step 1: pairwise condition.}
Let
$\lambda:=\lambda_{\mathrm{pw}}\in[\Delta_w/8,\Delta_w/4]$ and let
$w_\lambda$ be its mixture weight.  Set
\[
  B_{\mathrm{pw}}^{(w)}
  :=
  \log\{1/(\alpha_{\mathrm{pw}}w_\lambda)\}.
\]
For each competitor $a\neq r$, write
\[
  A_i:=\log(1+\lambda W_i)\ind\{X_i=r\},
  \qquad
  B_i(a):=-\log(1-\lambda W_i)\ind\{X_i=a\}.
\]
The relevant fixed-$\lambda$ log pairwise statistic is
\[
  S_t^{(w)}
  :=
  \sum_{i=1}^t A_i
  -
  \max_{a\neq r}\sum_{i=1}^t B_i(a).
\]
If $S_t^{(w)}\ge B_{\mathrm{pw}}^{(w)}$, then the weighted
all-competitors pairwise condition holds at time $t$.

Lemma~\ref{prop:weighted-smallbet-pair} gives the drift lower bound
\[
  \inf_{a\neq r}
  \EE\log\{1+\lambda W(\ind\{X=r\}-\ind\{X=a\})\}
  \ge
  \lambda\Delta_w/2
  \ge
  c\Delta_w^2 .
\]
Combining this with Lemma~\ref{lem:weighted-max-rival} yields
\[
  \EE[S_t^{(w)}]
  \ge
  c_0\Delta_w^2 t
  -
  C_0\Delta_w\sqrt{t\log(t+1)}.
\]
Changing one observation changes $S_t^{(w)}$ by at most
$C_1\Delta_w$, so McDiarmid's inequality and
Lemma~\ref{lem:det-log} imply that, for a universal constant $C$,
\[
  \PP\{\text{weighted pairwise condition fails at time }t\}
  \le
  (et)^{-2}
\]
whenever
\[
  t
  \ge
  T_{\mathrm{pw}}^{(w)}
  :=
  C
  \frac{
    \log\{1/(\alpha_{\mathrm{pw}}w_\lambda)\}
    +
    \log(1/\Delta_w)
  }{\Delta_w^2}.
\]

\smallskip
\noindent\textit{Step 2: LCB--unseen condition.}
Set $q_0:=\mu_r/2$.  The grid condition gives
$\lambda_r\in[1/32,1/16]\subset\Lambda_r(q_0)$ with weight
$v_{\lambda_r}$.  Let
\[
  Y_i:=W_i\ind\{X_i=r\},
  \qquad
  A_t:=
  \left\{
    \frac34\mu_r t
    \le
    \sum_{i=1}^tY_i
    \le
    \frac54\mu_r t
  \right\}.
\]
Since $Y_i\in[0,1]$ and $\EE Y_i=\mu_r$, Bernstein's inequality gives
\[
  \PP(A_t^c)
  \le
  2\exp(-c\mu_r t).
\]
On $A_t$, the single weighted LCB factor satisfies
\[
  \log \widetilde M_t(q_0,\lambda_r)
  =
  \sum_{i=1}^t
  \log\{1+\lambda_r(Y_i-q_0)\}
  \ge
  c'\mu_r t
\]
for a universal constant $c'>0$.  Thus
\[
  \widetilde M_t(q_0)
  \ge
  v_{\lambda_r}\widetilde M_t(q_0,\lambda_r)
  \ge
  \alpha_r^{-1}
\]
once
\[
  t
  \ge
  C\mu_r^{-1}\log\{1/(\alpha_r v_{\lambda_r})\}.
\]
Moreover, the unchanged unseen bound satisfies $U_t\le\mu_r/4<q_0$
once
\[
  t
  \ge
  C\mu_r^{-1}\log\{1/(\alpha_u\mu_r)\}.
\]
Combining these bounds with Lemma~\ref{lem:det-log}, we obtain
\[
  \PP\{\widetilde L_t(r)\le U_t\}
  \le
  (et)^{-2}
\]
for all
\[
  t
  \ge
  T_{\mathrm{rare}}^{(w)}
  :=
  C
  \frac{
    \log(1/\alpha_r)
    +
    \log(1/\alpha_u)
    +
    \log(1/v_{\lambda_r})
    +
    \log(1/\mu_r)
  }{\mu_r}.
\]

\smallskip
\noindent\textit{Step 3: combine.}
Let
\[
  T_{\mathrm{cl}}^{(w)}
  :=
  \max\{T_{\mathrm{pw}}^{(w)},T_{\mathrm{rare}}^{(w)}\}.
\]
For all $t\ge T_{\mathrm{cl}}^{(w)}$,
\[
  \PP(\tau^{(w)}>t)
  \le
  2(et)^{-2}.
\]
The tail-sum identity gives
\[
  \EE[\tau^{(w)}]
  \le
  T_{\mathrm{cl}}^{(w)}+1.
\]
Substituting
$\alpha_{\mathrm{pw}}=\alpha_r=\alpha_u=\varepsilon/3$ and using the
selected-weight lower bounds yields
\[
  \EE[\tau^{(w)}]
  =
  O\left(
    \frac{\log(1/\varepsilon)+\log(1/\Delta_w)}{\Delta_w^2}
    +
    \frac{\log(1/\varepsilon)+\log(1/\mu_r)}{\mu_r}
  \right).
\]
\end{proof}
 
\begin{corollary}[Independent confidence scores]
\label{cor:weighted-power-independent}
Suppose $W$ is independent of $X$ and let $\bar w:=\EE[W]>0$.  If
$r$ is the unweighted unique mode with gap
$\delta:=p_r-p_{(2)}>0$, then
\[
  \mu_r=\bar w p_r,
  \qquad
  \mu_\star=\bar w p_{(2)},
  \qquad
  \Delta_w=\bar w\delta.
\]
Thus Theorem~\ref{thm:weighted-clean-rate} applies whenever the
weighted grid condition holds, equivalently when
\[
  \lambda_{\mathrm{pw}}
  \in
  \left[
    \frac{\bar w\delta}{8},
    \frac{\bar w\delta}{4}
  \right]
  \quad\text{and}\quad
  \lambda_r\in[1/32,1/16]
\]
are available on the grids.
\end{corollary}
 
\subsection{Proof of Corollary~\ref{cor:weighted-power-independent}}
\label{app:proof-cor:weighted-power-independent}
 
\begin{proof}
Independence gives
\[
  \mu_a
  =
  \EE[W\ind\{X=a\}]
  =
  \bar w p_a
  \qquad
  \forall a\in\cA.
\]
Therefore
\[
  \mu_r=\bar w p_r,
  \qquad
  \mu_\star=\bar w p_{(2)},
  \qquad
  \Delta_w=\bar w\delta.
\]
Substituting these identities into the grid condition and the rate in
Theorem~\ref{thm:weighted-clean-rate} proves the claim.
\end{proof}

\begin{proposition}[Recovery of the unweighted CITE]
\label{prop:recovery}
Suppose $W_i \equiv 1$ almost surely. Then
$\mu_a = p_a$ for every $a \in \cA$;
$\widetilde E_t^{(r,a)}(\lambda) = (1+\lambda)^{N_t(r)}(1-\lambda)^{N_t(a)}$;
$\widehat \mu_t(r)=\widehat p_t(r)$;
and the weighted stopping rule agrees with the all-competitors
version of the unweighted CITE, with runner-up monotonicity reducing
the pairwise check to the empirical runner-up.
\end{proposition}
 
\subsection{Proof of Proposition~\ref{prop:recovery}}
\label{app:proof-prop:recovery}
 
\begin{proof}
\textbf{Claim (i): Mass equality.}
If $W_i \equiv 1$ almost surely, then
\[
  \mu_a
  = \EE[W\ind\{X=a\}]
  = \EE[\ind\{X=a\}]
  = p_a
  \quad \forall a \in \cA.
\]

\textbf{Claim (ii): Pairwise e-process recovery.}
Under $W_i \equiv 1$, the weighted e-process simplifies:
\[
  \widetilde Z_i^{(r,a)}
  = W_i(\ind\{X_i=r\}-\ind\{X_i=a\})
  = \ind\{X_i=r\} - \ind\{X_i=a\}.
\]
Thus
\[
  \widetilde E_t^{(r,a)}(\lambda)
  = \prod_{i=1}^t \bigl(1+\lambda \widetilde Z_i^{(r,a)}\bigr)
  = \prod_{i=1}^t \bigl(1+\lambda(\ind\{X_i=r\}-\ind\{X_i=a\})\bigr)
  = (1+\lambda)^{N_t(r)}(1-\lambda)^{N_t(a)}.
\]

\textbf{Claim (iii): Empirical mean equality.}
\[
  \widehat \mu_t(r)
  = \frac{1}{t}\sum_{i=1}^t W_i\ind\{X_i=r\}
  = \frac{1}{t}\sum_{i=1}^t \ind\{X_i=r\}
  = \widehat p_t(r).
\]
Substituting $W_i \equiv 1$ into the weighted LCB e-process yields
\[
  \widetilde M_t(q,\lambda) = \prod_{i=1}^t (1+\lambda(\ind\{X_i=r\}-q)),
\]
which is identical to the unweighted LCB e-process after dropping the tilde.

\textbf{Claim (iv): Stopping rule recovery.}
From (ii), in the unweighted specialization the weighted pairwise condition
\[
\widetilde E_t^{(r,a)} = \sum_\lambda w_\lambda \widetilde E_t^{(r,a)}(\lambda) \geq 1/\alpha_{\mathrm{pw}}
\quad \text{for all } a \in \cA_t \setminus \{r\}
\]
checks the same pairwise e-values as the unweighted all-competitors CITE.
By Lemma~\ref{lem:monotone}, those e-values are monotone in the competitor count, so among observed competitors the minimum is
attained at the empirical runner-up.  Thus the all-competitors pairwise check reduces
to a single empirical runner-up, recovering the original unweighted
fixed-target CITE.
\end{proof}

\section{Comparison with Martingale Majority Certificates}
\label{sec:mmc}
 
\subsection{The MMC Hypothesis}
\label{sec:mmc-hypothesis}
 
The top-$m$ MMC \citep{cordero2025certified} simultaneously tests the leader
against $m-1$ pairwise competitors and the aggregated residual mass.
The alternative certified upon stopping is
\begin{equation}\label{eq:mmc-target}
  \cH_{\mathrm{MMC},m}
  := \{P : p_{(1)} > p_{(i)}\forall i=2,\ldots,m\}
  \cap \{P : p_{(1)} > s_m\},
  \qquad
  s_m := \sum_{j > m} p_{(j)}.
\end{equation}
 
\begin{proposition}[Strict inclusion]
\label{prop:strict}
Suppose the target $r$ coincides with the (unique) modal label
$p_r = p_{(1)}$, equivalently $r$ is the unique maximiser of
$a \mapsto p_a$.  Then for every $m \geq 2$,
$\cH_{\mathrm{MMC},m} \subsetneq H_{1,r}$;
without the unique-mode-target identification,  this strict inclusion
must be read with $H_{1,r}$ replaced by
$\bigcup_{r' \in \cA} H_{1,r'}$.
\end{proposition}
 
\begin{remark}[The diffuse-tail regime]
\label{rem:diffuse-tail}
Distributions in $H_{1,r} \setminus \cH_{\mathrm{MMC},m}$ have a
unique mode that does not dominate the aggregate tail:
$p_{(1)} > p_{(2)}$ but $p_{(1)} \leq s_m$.  Several of our LLM
self-consistency settings (Section~\ref{sec:experiments}) fall into
this regime, with the mode-answer carrying $20$--$30\%$ of the mass
and thousands of alternative answers sharing the remainder.  For each fixed $m$, the MMC controls its certification rate at level $\alpha$ on any such $P$: its residual-mass channel at that $m$
tests the null $\{p_{(1)} \leq s_m\}$, which is true under $P$, so the
corresponding e-process cannot grow past the rejection threshold with
probability $> \alpha$.  An MMC variant that tries multiple $m$ in
parallel pays a Bonferroni-style cost over the active tuples
(Appendix~\ref{sec:repair}).
\end{remark}
 
\subsection{The Switching-Null Challenge}
\label{sec:switching-gap}
 
At each round $n$, the MMC selects the empirical leader $A_{n-1}$
and runner-up $B_{n-1}$ (both $\cF_{n-1}$-measurable), maintaining
e-processes: $e_n^{\mathrm{run}}$ (leader vs.\ runner-up) and $e_n^{\mathrm{oth}}$ (leader vs.\ residual mass).
The per-round nulls are:
\begin{equation}\label{eq:per-round}
  R_n := \{p_{A_{n-1}} \leq p_{B_{n-1}}\},
  \quad
  O_n := \Bigl\{p_{A_{n-1}} \leq
    \sum\nolimits_{j \notin \{A_{n-1},B_{n-1}\}} p_j\Bigr\}.
\end{equation}
 
Theorem~3.1 of \citet{cordero2025certified} establishes that
$(e_n^{\mathrm{run}})$ is an NSM under the all-rounds null
$H_0^{\mathrm{run}} := \bigcap_{n \geq 1} R_n$, and similarly for the residual-mass channel.
Corollary~3.2 proves
\begin{equation}\label{eq:cor32}
  \sup_{P \in H_0} \PP_P(N < \infty) \leq \varepsilon,
  \qquad
  H_0 := H_0^{\mathrm{run}} \cup H_0^{\mathrm{oth}}.
\end{equation}
The intersection-union structure (stop requires \emph{both} channels, null is a \emph{union})
gives $\varepsilon$ without Bonferroni.
 
The challenge lies in bridging from~\eqref{eq:cor32} to the
adaptive misclassification guarantee
$\PP(A_{\tau-1} \neq a^\star) \leq \varepsilon$ for the certified
pre-round leader (cf.\ Proposition~\ref{prop:repair} below).
At any fixed time $n$, misclassification of the pre-round leader
satisfies
\begin{equation}\label{eq:local-inclusion}
  \{A_{n-1} \neq a^\star\} \subseteq R_n \cup O_n.
\end{equation}
However, Corollary~3.2 controls the all-rounds event
\[
  H_0 = \Bigl(\bigcap_{t \geq 1} R_t\Bigr)
  \cup \Bigl(\bigcap_{t \geq 1} O_t\Bigr),
\]
which is logically distinct from the existential per-round inclusion:
\begin{equation}\label{eq:logic-distinction}
  (\forall t R_t) \lor (\forall t O_t)
  \quad\not\Longleftarrow\quad
  \exists n(R_n \lor O_n).
\end{equation}
Bridging this gap requires a \emph{switching-null reduction} showing that
misclassification at~$\tau$ is contained in the all-rounds null.
Since the labels $(A_{n-1}, B_{n-1})$ are data-dependent and may switch while
e-processes accumulate evidence without restart, this reduction is nontrivial but feasible.
It introduces Bonferroni costs absent from CITE.
 
\subsection{Tuple-Indexed Construction}
\label{sec:repair}
 
To handle the switching-null challenge, \emph{index e-processes by tuple}: let
\[
  \cT
  := \bigl\{(m, a, b_1, \ldots, b_{m-1}) :
     m \geq 2,\ a \in \cA,\ b_1, \ldots, b_{m-1} \in \cA \setminus \{a\}
     \text{ pairwise distinct}\bigr\}
\]
denote the (countable) set of admissible tuples and assign each
$\ell = (m, a, b_1, \ldots, b_{m-1}) \in \cT$ its own e-process,
with multiplicative factor $1$ on rounds where the current tuple differs from $\ell$.
 
\begin{proposition}[Tuple-indexed repair]
\label{prop:repair}
Assume the data-generating distribution has a unique mode
$a^\star \in \arg\max_{a \in \cA} p_a$.  Let
$\{\alpha_\ell\}_{\ell \in \cT}$ satisfy
$\sum_\ell \alpha_\ell \leq \varepsilon$.
For each $\ell$, maintain separate tuple-indexed e-processes
$E_t^{\ell,i}$ (pairwise, slot $i$) and
$E_t^{\ell,\mathrm{oth}}$ (residual mass),
active only when $L_{t-1} = \ell$.
Define
\[
  \tau := \inf\Bigl\{
    n : \exists \ell \in \cT \text{ s.t.\ }
    L_{n-1} = \ell,
    E_n^{\ell,i} \geq 1/\alpha_\ell\forall i,
    E_n^{\ell,\mathrm{oth}} \geq 1/\alpha_\ell
  \Bigr\}.
\]
Let $A_{\tau-1}$ denote the leader component of the active tuple
$L_{\tau-1}$ (the pre-round empirical leader at the certifying time).
Then $\PP(\tau < \infty, A_{\tau - 1} \neq a^\star) \leq \varepsilon$.
\end{proposition}
 
\paragraph{The Bonferroni cost.}
Each tuple receives error allocation $\alpha_\ell$, yielding threshold $1/\alpha_\ell$.
With $\pi_\ell := \alpha_\ell / \varepsilon$:
\begin{equation}\label{eq:bonf-cost}
  \log \frac{1}{\alpha_\ell}
  = \underbrace{\log \frac{1}{\varepsilon}}_{\text{base}}
  + \underbrace{\log \frac{1}{\pi_\ell}}_{\text{multiplicity penalty}}.
\end{equation}
Under uniform allocation over $K$ tuples, the multiplicity penalty is $\log K$.
With $L_t$ observed categories: $K_{m,t} = L_t \binom{L_t-1}{m-1} = O(L_t^m)$,
yielding additional sample complexity of order $m \log L_t / I_\ell$,
where $I_\ell$ denotes the per-observation KL information.
 
\paragraph{Contrast with CITE.}
CITE avoids this multiplicity cost: with $r$ fixed a priori, Ville's inequality
applies to a single e-process with threshold $1/\alpha_{\mathrm{pw}}$
and no multiplicity correction.
 
\subsection{Structural Comparison Summary}
 
\begin{table}[H]
\centering\small
\renewcommand{\arraystretch}{1.3}
\begin{tabular}{p{2.8cm}p{5.0cm}p{5.0cm}}
\toprule
& \textbf{Tuple-Indexed MMC} & \textbf{CITE (ours)} \\
\midrule
Hypothesis &
  $p_{(1)} > p_{(i)}$ for all $i \leq m$ and $p_{(1)} > s_m$ &
  $p_{(1)} > p_{(2)}$ (exact unique mode) \\
Target selection &
  Adaptive (tuple may switch) &
  Fixed target $r$ a priori \\
Multiplicity cost &
  $\log K_{m,t}$ threshold inflation &
  None \\
Diffuse-tail power &
  $\leq \alpha$ ($p_{(1)} \leq s_m$ $\Rightarrow$ null true; Ville bound) &
  Full (pairwise margin only) \\
Unseen control &
  Aggregated into residual mass &
  Explicit: $L_t(r) > U_t$ \\
\bottomrule
\end{tabular}
\caption{Structural comparison of the tuple-indexed MMC and CITE.}
\label{tab:comparison}
\end{table}
 
\subsection{Proof of Proposition~\ref{prop:strict}}
\label{app:proof-prop:strict}

\begin{proof}
\emph{Inclusion $\cH_{\mathrm{MMC},m} \subseteq H_{1,r}$.}
Any $P \in \cH_{\mathrm{MMC},m}$ satisfies
$p_{(1)} > p_{(i)}$ for every $i \in \{2, \ldots, m\}$, and
$p_{(1)} > s_m \geq p_{(j)}$ for every $j > m$; combining, $p_r = p_{(1)}$
strictly dominates every other category, i.e.\ $P \in H_{1,r}$.

\emph{Strict inclusion.}  Fix $m \geq 2$ and $K \geq 2m$, and consider
the perturbed-uniform distribution
\[
  p_{(1)} := \frac{1}{K} + \eta,
  \qquad
  p_{(i)} := \frac{1}{K} - \frac{\eta}{K-1}
  \quad (i \geq 2),
  \qquad
  \eta \in (0,\tfrac{1}{K}),
\]
so that $\sum_i p_{(i)} = 1$ and $p_{(1)} > p_{(2)}$.
Then $P \in H_{1,r}$ (since $r$ is strictly the unique mode), while
\[
  s_m
  = \sum_{j > m} p_{(j)}
  = (K-m)\left(\tfrac{1}{K} - \tfrac{\eta}{K-1}\right)
  \geq \tfrac{K-m}{K} - \eta
  \geq \tfrac{1}{2} - \eta
  > \tfrac{1}{K} + \eta = p_{(1)}
\]
for sufficiently small $\eta$ (using $K \geq 2m$ and $K \geq 2$).
Hence $P \notin \cH_{\mathrm{MMC},m}$, proving strict inclusion.
\end{proof}

\subsection{Proof of Proposition~\ref{prop:repair}}
\label{app:proof-prop:repair}

\begin{proof}
Fix a tuple $\ell = (a, b_1, \ldots, b_{m-1}) \in \cT$ with leader $a$ and
slots $b_1, \ldots, b_{m-1}$.
Define its stopping time and misclassification event
\[
  \tau_\ell := \inf\bigl\{
    n \geq 1 : L_{n-1} = \ell,\
    E_n^{\ell,i} \geq 1/\alpha_\ell\ \forall i,\
    E_n^{\ell,\mathrm{oth}} \geq 1/\alpha_\ell
  \bigr\},
  \quad
  G_\ell := \{\tau = \tau_\ell < \infty,\ a \neq a^\star\}.
\]
Suppose $G_\ell$ occurs.  Since $a \neq a^\star$ and $a^\star$ lies in
$\cA$, exactly one of the following holds:

\smallskip
\noindent\textbf{Case (i): $a^\star = b_i$ for some $i \in \{1, \ldots, m-1\}$.}
Then the pairwise null for slot $i$ is
$R^{(\ell,i)} := \{p_a \leq p_{b_i}\}$, which is \emph{true} because
$p_a < p_{a^\star} = p_{b_i}$ (as $a^\star$ is the unique mode).  Under
$R^{(\ell,i)}$ the pairwise e-process $(E_n^{\ell,i})_{n \geq 0}$ is an NSM
with initial value $\leq 1$, so Ville's inequality gives
$\PP(E_{\tau_\ell}^{\ell,i} \geq 1/\alpha_\ell) \leq \alpha_\ell$.
Hence $\PP(G_\ell \cap \text{Case (i)}) \leq \alpha_\ell$.

\smallskip
\noindent\textbf{Case (ii): $a^\star \notin \{a, b_1, \ldots, b_{m-1}\}$.}
Then the residual-mass null
$O^{(\ell)} := \{p_a \leq \sum_{j \notin \{a, b_1, \ldots, b_{m-1}\}} p_j\}$
is true, since the sum on the right side includes $p_{a^\star} > p_a$.
Under $O^{(\ell)}$, $(E_n^{\ell,\mathrm{oth}})_{n \geq 0}$ is an NSM with
initial value $\leq 1$, and Ville gives
$\PP(E_{\tau_\ell}^{\ell,\mathrm{oth}} \geq 1/\alpha_\ell) \leq \alpha_\ell$.
Hence $\PP(G_\ell \cap \text{Case (ii)}) \leq \alpha_\ell$.

\smallskip
Combining both cases, $\PP(G_\ell) \leq \alpha_\ell$.

\smallskip
\noindent\textbf{Mutual exclusion across tuples.}
At the stopping time $\tau$, only the tuple
$L_{\tau-1}$ has $L_{\tau-1} = \ell$, so the events
$\{\tau = \tau_\ell, L_{\tau_\ell - 1} = \ell\}$ are disjoint across $\ell$.
Therefore
\[
  \PP(\tau < \infty,\ A_{\tau - 1} \neq a^\star)
  = \sum_{\ell \in \cT} \PP(G_\ell)
  \leq \sum_{\ell \in \cT} \alpha_\ell
  \leq \varepsilon.
  \qedhere
\]
\end{proof}

%% ============================================================
%% Appendix: Additional Theoretical Results
%% ============================================================
\section{Additional Theoretical Results}
\label{app:theory-extras}

This appendix collects additional results used to interpret and extend
the main theory: a sharp oracle pairwise benchmark
(\ref{app:sharp-constants}), an adaptive geometric grid
(\ref{app:adaptive-grid}), a growth-rate optimality result for the
oracle pairwise bet and its finite mixture
(\ref{app:grow}), and a top-$k$ extension
(\ref{app:top-k}).

\subsection{Sharp oracle pairwise benchmark}
\label{app:sharp-constants}

The main text compares the stopping-time upper bound with the minimax
lower bound at the level of rates.  Here we isolate a sharper
constant-level statement for an oracle pairwise comparison between the
target and its true runner-up.  This result is not a sharp-constant
theorem for the full CITE stopping time; rather, it shows that the
pairwise e-process has the correct leading constant when both the true
runner-up and the oracle betting parameter are available.

\begin{proposition}[Identity and small-gap expansion of $\rho^\star$]
\label{prop:rho-star-expansion}
The oracle pairwise growth rate
$\rho^\star := \max_{\lambda \in (0,1)} \bigl[p_r \log(1+\lambda) + p_{(2)} \log(1-\lambda)\bigr]$
(Proposition~\ref{prop:pairwise-oracle}) coincides exactly with the
forward I-projection KL of $P$ onto $H_{0,r}$:
\begin{equation}\label{eq:rho-iproj}
  \rho^\star  =  D(P\|P^\star)
   =  p_r \log \frac{2 p_r}{p_r+p_{(2)}}
        + p_{(2)} \log \frac{2 p_{(2)}}{p_r+p_{(2)}}.
\end{equation}
In particular, by Lemma~\ref{lem:iproj},
\[
  \rho^\star
   =  \frac{\delta^2}{2(p_r+p_{(2)})} + O \left(\frac{\delta^4}{(p_r+p_{(2)})^3}\right)
   =  \frac{\delta^2}{4 p_r}\bigl(1 + o(1)\bigr).
\]
\end{proposition}

\begin{proof}
Substituting $\lambda^\star = (p_r-p_{(2)})/(p_r+p_{(2)})$ from
Proposition~\ref{prop:pairwise-oracle} into the definition of
$\rho^\star$ yields the closed form of $D(P\|P^\star)$ in
Lemma~\ref{lem:iproj}, since $1\pm\lambda^\star = 2 p_r/(p_r+p_{(2)})$
and $2 p_{(2)}/(p_r+p_{(2)})$ respectively.  Both displays now follow
directly from Lemma~\ref{lem:iproj} after substituting $q^\star = (p_r+p_{(2)})/2$.
\end{proof}

\begin{theorem}[Sharp constant for the oracle pairwise benchmark]
\label{thm:sharp-ratio}
Assume $p_{(2)}>0$, $p_r\in(0,1/2]$, and $\delta/p_r\to0$.  Suppose
\[
  \lambda^\star
  :=
  \frac{\delta}{p_r+p_{(2)}}
  \in\Lambda_{\mathrm{pw}},
  \qquad
  w_{\lambda^\star}>0.
\]
Let $a^\star\in\arg\max_{a\neq r}p_a$, and let
$P^\star\in H_{0,r}$ be the forward I-projection of $P$ from
Lemma~\ref{lem:iproj}.  Define
\[
  \rho^\star
  :=
  \max_{\lambda\in(0,1)}
  \{p_r\log(1+\lambda)+p_{(2)}\log(1-\lambda)\}.
\]
Then
\[
  \rho^\star
  =
  D(P\|P^\star)
  =
  \frac{\delta^2}{4p_r}\{1+o(1)\}.
\]
For the oracle pairwise hitting time
\[
  \sigma_{\mathrm{pw}}^\star
  :=
  \inf\left\{
    t\ge1:
    w_{\lambda^\star}E_t^{(r,a^\star)}(\lambda^\star)
    \ge
    \alpha_{\mathrm{pw}}^{-1}
  \right\},
\]
we have
\[
  \EE_P[\sigma_{\mathrm{pw}}^\star]
  \le
  \frac{
    \log\{1/(\alpha_{\mathrm{pw}}w_{\lambda^\star})\}
  }{\rho^\star}
  \{1+o(1)\}.
\]
Conversely, any level-$\varepsilon$ sequential certifier $\tau'$ for
$H_{0,r}$ with $\PP_P(\tau'<\infty)=1$ and
$\EE_P[\tau']<\infty$ obeys
\[
  \EE_P[\tau']
  \ge
  \frac{\log(1/\varepsilon)}{\rho^\star}.
\]
In particular, if $\alpha_{\mathrm{pw}}=c\varepsilon$ for a constant
$c\in(0,1]$ and
$\log(1/w_{\lambda^\star})=o\{\log(1/\varepsilon)\}$, the oracle
pairwise upper bound and the information lower bound share the leading
constant $4p_r/\delta^2$.
\end{theorem}

\begin{proof}
Proposition~\ref{prop:rho-star-expansion} gives
\[
  \rho^\star
  =
  D(P\|P^\star)
  =
  \frac{\delta^2}{4p_r}\{1+o(1)\}.
\]

\textit{(a) Upper bound.}
Define
\[
  Y_i
  :=
  \log\{1+\lambda^\star Z_i^{(r,a^\star)}\}.
\]
Then the $Y_i$ are i.i.d., bounded, and
$\EE_P[Y_i]=\rho^\star>0$.  Let
\[
  B
  :=
  \log\{1/(\alpha_{\mathrm{pw}}w_{\lambda^\star})\}.
\]
At the hitting time $\sigma_{\mathrm{pw}}^\star$,
\[
  \sum_{i=1}^{\sigma_{\mathrm{pw}}^\star}Y_i
  \ge
  B.
\]
Moreover, the crossing overshoot is at most the largest positive
increment, so
\[
  \sum_{i=1}^{\sigma_{\mathrm{pw}}^\star}Y_i
  \le
  B+\log(1+\lambda^\star).
\]
Since the increments are bounded and the drift is positive, Wald's
identity gives
\[
  \rho^\star\EE_P[\sigma_{\mathrm{pw}}^\star]
  =
  \EE_P\left[
    \sum_{i=1}^{\sigma_{\mathrm{pw}}^\star}Y_i
  \right]
  \le
  B+\log(1+\lambda^\star).
\]
As $\lambda^\star\to0$ in the small-gap regime,
\[
  \EE_P[\sigma_{\mathrm{pw}}^\star]
  \le
  \frac{
    \log\{1/(\alpha_{\mathrm{pw}}w_{\lambda^\star})\}+o(1)
  }{\rho^\star}.
\]
Substituting the expansion of $\rho^\star$ proves the oracle upper
bound.

\textit{(b) Lower bound.}
Apply the sequential change-of-measure argument to the I-projection
$P^\star\in H_{0,r}$.  Type-I validity gives
$P^\star(\tau'<\infty)\le\varepsilon$, while
$\PP_P(\tau'<\infty)=1$.  By data processing and
Lemma~\ref{lem:stopped-kl-iproj},
\[
  \log(1/\varepsilon)
  \le
  \EE_P[\tau']D(P\|P^\star)
  =
  \EE_P[\tau']\rho^\star .
\]
Hence
\[
  \EE_P[\tau']
  \ge
  \frac{\log(1/\varepsilon)}{\rho^\star}.
\]
\textit{(c) Ratio.}
With $\alpha_{\mathrm{pw}} = c\varepsilon$,
\[
  \frac{\log(1/(\alpha_{\mathrm{pw}} w_{\lambda^\star}))}{\log(1/\varepsilon)}
   =  1 + \frac{\log(1/(c w_{\lambda^\star}))}{\log(1/\varepsilon)}
   \to  1
\]
using $\alpha_{\mathrm{pw}}=c\varepsilon$ and
$\log(1/w_{\lambda^\star})=o\{\log(1/\varepsilon)\}$.
\end{proof}

\begin{remark}[Scope of the oracle result]
Theorem~\ref{thm:sharp-ratio} is an oracle benchmark for comparing $r$
with its true runner-up.  It is not a sharp-constant theorem for the
full CITE stopping time, which also uses the empirical runner-up, the
LCB--unseen check, and a finite grid.
\end{remark}

\subsection{Adaptive geometric grid}
\label{app:adaptive-grid}

Condition~\ref{cond:grid} requires a grid point in
$[\delta/8,\delta/4]$, which depends on the unknown gap $\delta$.  A
geometric grid gives a uniform version of the same rate over all
$\delta\ge\delta_0$.

\begin{condition}[Geometric grid]
\label{cond:geom-grid}
For a fixed minimum gap $\delta_0\in(0,1]$, set
\[
  \Lambda_{\mathrm{pw}}^{\mathrm{geo}}
  :=
  \{2^{-k}:1\le k\le K\},
  \qquad
  K
  :=
  \left\lceil\log_2(8/\delta_0)\right\rceil,
\]
with uniform weights $w_\lambda=1/K$.
\end{condition}

\begin{theorem}[Adaptive grid rate]
\label{thm:adaptive-grid}
Suppose Assumption~\ref{asn:dgp} holds and CITE uses
$\Lambda_{\mathrm{pw}}^{\mathrm{geo}}$ in the pairwise component.  Keep
the LCB-grid and selected-weight assumptions of
Theorem~\ref{thm:power-iut-clean}.  Then, uniformly over all
$P$ with modal gap $\delta\ge\delta_0$,
\[
  \EE[\tau]
  =
  O\left(
    \frac{
      \log(1/\varepsilon)
      +
      \log(\log(8/\delta_0)\vee e)
      +
      \log(1/\delta)
    }{\delta^2}
    +
    \frac{
      \log(1/\varepsilon)+\log(1/p_r)
    }{p_r}
  \right),
\]
with constants independent of $\delta,p_r$, and $|\cA|$.
\end{theorem}

\begin{proof}
Fix $P$ with $\delta\ge\delta_0$ and set
\[
  k^\star
  :=
  \left\lfloor\log_2(8/\delta)\right\rfloor,
  \qquad
  \lambda
  :=
  2^{-k^\star}.
\]
Then $\lambda\in[\delta/8,\delta/4]$ and, since
$\delta\ge\delta_0$, we have $k^\star\le K$.  Hence
$\lambda\in\Lambda_{\mathrm{pw}}^{\mathrm{geo}}$.

The proof of Theorem~\ref{thm:power-iut-clean} applies with this grid
point.  The only change is the pairwise mixture weight:
\[
  \log\frac{1}{\alpha_{\mathrm{pw}}w_\lambda}
  =
  \log\frac{1}{\alpha_{\mathrm{pw}}}
  +
  \log K
  =
  \log\frac{1}{\alpha_{\mathrm{pw}}}
  +
  O\{\log(\log(8/\delta_0)\vee e)\}.
\]
Substituting this into the pairwise part of the stopping-time bound,
and keeping the LCB--unseen part unchanged, proves the display.
\end{proof}

\begin{remark}
The geometric grid removes the need to know $\delta$ exactly, but it
still requires a lower design scale $\delta_0$.  Removing this design
parameter would require a separate adaptive-betting analysis, which we
do not pursue here.
\end{remark}

\subsection{Growth-rate optimality of the oracle pairwise bet}
\label{app:grow}

We record a class-restricted growth-rate optimality property for the
oracle pairwise bet, in the GRO/GROW sense of
\citet{grunwald2024safe}.  The comparison is restricted to predictable
linear bets in the pairwise statistic.

\begin{definition}[Class-restricted GRO/GROW]
\label{def:grow}
Fix a pair $(r,a)$ with $r\neq a$, a class $\cQ$ of alternatives, and a
class $\cE$ of e-processes for the pairwise null
$H_0^{(a)}=\{P:p_r\le p_a\}$.  An element $E^\star\in\cE$ is
\emph{GRO/GROW within $\cE$ with respect to $\cQ$} if, for every fixed
horizon $T$,
\[
  E^\star
  \in
  \arg\max_{E\in\cE}
  \inf_{Q\in\cQ}
  \EE_Q[\log E_T].
\]
\end{definition}

\begin{theorem}[GRO bet and mixture near-optimality]
\label{thm:grow}
Let $\cE_Z$ denote the class of processes
$E_t = \prod_{i=1}^t (1 + \lambda_i Z_i^{(r,a)})$ with predictable
bets $\lambda_i \in [0, 1)$ measurable with respect to
$\cF_{i-1} := \sigma(X_1, \ldots, X_{i-1})$.
Every $E \in \cE_Z$ is an e-process for the composite null
$H_0^{(a)} = \{P : p_r \leq p_a\}$ in the sense of
Definition~\ref{def:grow}.

Fix a simple alternative $Q$ with $\PP_Q(X = r) = p$,
$\PP_Q(X = a) = q$, and $p > q > 0$ (the boundary case $q = 0$ is
degenerate: $\lambda^\star = 1$ lies on the boundary of the open
betting interval $[0, 1)$, so the supremum is approached only in the
limit $\lambda \uparrow 1$ and is not attained in $\mathcal{E}_Z$),
and define
\[
  \lambda^\star := \frac{p - q}{p + q},
  \qquad
  \rho^\star := p \log(1 + \lambda^\star) + q \log(1 - \lambda^\star).
\]
\begin{enumerate}[label=\textup{(\alph*)},leftmargin=2.2em,topsep=2pt]
\item \textbf{GRO bet.}
For every $T \geq 1$ and every $E \in \cE_Z$,
$\EE_Q[\log E_T] \leq T  \rho^\star$, with equality if and only if
$\lambda_i = \lambda^\star$ holds $Q$-almost surely for every
$i \leq T$.  In particular, $E_t^{(r,a)}(\lambda^\star)$ is the unique
growth-rate maximizer within $\cE_Z$ under the simple alternative $Q$.

\item \textbf{Mixture near-optimality.}
For any finite grid $\Lambda_{\mathrm{pw}} \subset (0, 1)$ with
$\lambda^\star \in \Lambda_{\mathrm{pw}}$ and positive weights
$(w_\lambda)$ summing to at most $1$, the mixture
$ E_t^{(r,a)} := \sum_{\lambda \in \Lambda_{\mathrm{pw}}}
   w_\lambda E_t^{(r,a)}(\lambda)$
satisfies, for every $T \geq 1$,
\[
  T  \rho^\star - \log\frac{1}{w_{\lambda^\star}}
   \leq 
  \EE_Q[\log  E_T^{(r,a)}]
   \leq 
  T  \rho^\star.
\]
The mixture therefore has the same asymptotic expected log-growth rate
as the oracle fixed-$\lambda^\star$ bet, with the additive log-weight
term $\log(1/w_{\lambda^\star})$.

\end{enumerate}
\end{theorem}

\begin{proof}
\noindent\emph{Preliminary: e-process property.}
For any $P \in H_0^{(a)}$, $X_i$ is independent of $\cF_{i-1}$ with
$\EE_P[Z_i^{(r,a)}] = p_r - p_a \leq 0$, so
\[
  \EE_P[1 + \lambda_i Z_i^{(r,a)} \mid \cF_{i-1}]
   =  1 + \lambda_i (p_r - p_a)
   \leq  1
\]
($\lambda_i \geq 0$).  Thus $(E_t)$ is a nonneg.\ supermartingale with
$E_0 = 1$, and optional stopping gives $\EE_P[E_\tau] \leq 1$ for any
stopping time $\tau$.

\smallskip
\noindent\textit{Part (a): GRO bet.}
Fix $E \in \cE_Z$.  For each $i$, $\lambda_i$ is
$\cF_{i-1}$-measurable and $X_i$ is independent of $\cF_{i-1}$ under
$Q$, taking values $r$, $a$, and ``other'' with probabilities $p$,
$q$, and $1 - p - q$ (in which last case $Z_i = 0$ and the log term
vanishes).  Conditioning on $\cF_{i-1}$,
\[
  \EE_Q \bigl[\log(1 + \lambda_i Z_i^{(r,a)}) \mid \cF_{i-1}\bigr]
   =  p \log(1 + \lambda_i) + q \log(1 - \lambda_i)
   =:  \rho(\lambda_i).
\]
Summing over $i = 1, \ldots, T$ via the tower property gives
\begin{equation}\label{eq:grow-decomp}
  \EE_Q[\log E_T]  =  \sum_{i=1}^T \EE_Q[\rho(\lambda_i)].
\end{equation}
By Proposition~\ref{prop:pairwise-oracle}, $\rho$ is strictly concave
on $(0, 1)$ with unique maximiser $\lambda^\star$, so
$\rho(\lambda_i) \leq \rho^\star$ pointwise.  Substituting
into~\eqref{eq:grow-decomp} yields $\EE_Q[\log E_T] \leq T \rho^\star$,
with equality iff $\lambda_i = \lambda^\star$ $Q$-a.s.\ for each $i$.

\smallskip
\noindent\textit{Lower bound for part (b).}

For every $\lambda \in \Lambda_{\mathrm{pw}}$, $E_t^{(r,a)}(\lambda) > 0$,
hence $ E_T^{(r,a)} \geq w_{\lambda^\star}  E_T^{(r,a)}(\lambda^\star)$
$Q$-a.s.  Taking logs and $Q$-expectations and applying (a) to the
constant bet $\lambda_i \equiv \lambda^\star$:
\[
  \EE_Q[\log  E_T^{(r,a)}]
   \geq  \log w_{\lambda^\star} + T \rho^\star
   =  T \rho^\star - \log(1/w_{\lambda^\star}).
\]

\smallskip
\noindent\textit{Upper bound for part (b), via variational duality.}
Let $P^\star$ be the boundary distribution
\[
  P^\star(X = r) = P^\star(X = a) := \tfrac{p+q}{2},
  \qquad P^\star(X = b) := Q(X = b) \text{ for } b \notin \{r, a\}.
\]
The Lagrangian argument of Lemma~\ref{lem:iproj}, applied to
$H_0^{(a)} = \{P : p_r \leq p_a\}$ in place of $H_{0,r}$, identifies
$P^\star$ as the forward I-projection of $Q$ onto $H_0^{(a)}$, with
the closed form
\begin{equation}\label{eq:grow-id}
  D(Q \| P^\star)
   =  p \log \frac{2 p}{p+q} + q \log \frac{2 q}{p+q}
   =  p \log(1 + \lambda^\star) + q \log(1 - \lambda^\star)
   =  \rho^\star,
\end{equation}
the second equality from
$1 \pm \lambda^\star = 2 p/(p+q)$ and $2 q/(p+q)$.

Writing $W := \sum_{\lambda \in \Lambda_{\mathrm{pw}}} w_\lambda \leq 1$,
$ E_T^{(r,a)}$ is an e-process for $H_0^{(a)}$ (mixture of
e-processes), so $\EE_{P^\star}[ E_T^{(r,a)}] \leq W \leq 1$.
Decomposing the $Q$-expectation via change of measure and applying
Jensen to the concave $\log$,
\begin{align*}
  \EE_Q[\log  E_T^{(r,a)}]
  &= \EE_Q \left[\log \left( E_T^{(r,a)} \cdot
        \frac{d (P^\star)^T}{d Q^T}\right)\right]
    + \EE_Q \left[\log \frac{d Q^T}{d (P^\star)^T}\right] \\
  &\leq \log \EE_{P^\star}[ E_T^{(r,a)}]
        + D(Q^T \| (P^\star)^T)
   \leq  0 + T  D(Q \| P^\star)
   =  T \rho^\star,
\end{align*}
where the last line uses $\EE_{P^\star}[ E_T^{(r,a)}] \leq 1$,
tensorisation of KL, and \eqref{eq:grow-id}.  Combining with the lower
bound finishes the proof.
\end{proof}

\subsection{Top-$k$ mode certification}
\label{app:top-k}

Fix a target set $S \subseteq \cA$ of size $k \geq 1$ with
$\cA \setminus S \neq \emptyset$ (so that the outsider competitor
set is meaningful).
We extend CITE to certify that $S$ contains the top-$k$ categories
in the sense
\[
  H_{1,S}^{(k)}
  := \bigl\{P :
     \min_{s \in S} p_s
      > 
     \sup_{a \notin S} p_a\bigr\},
  \qquad
  H_{0,S}^{(k)} := \neg H_{1,S}^{(k)}.
\]

\paragraph{Top-$k$ CITE.}
Let
\[
  \delta_S
  :=
  \min_{s\in S}p_s-\sup_{a\notin S}p_a,
  \qquad
  p_S
  :=
  \min_{s\in S}p_s .
\]
For each $s\in S$ and $a\notin S$, define
$ E_t^{(s,a)}$ as in Definition~\ref{def:mixture-e}, and let
$L_t(s)$ be the LCB for $p_s$.

At time $t$, if $\cA_t\setminus S\neq\emptyset$, define
\[
  \widehat a_t^{\mathrm{out}}
  \in
  \arg\max_{a\in\cA_t\setminus S}N_t(a)
\]
with deterministic tie-breaking.  Define
\[
  \mathsf{PW}_t^{(k)}
  :=
  \begin{cases}
    \mathrm{true},
    & \cA_t\setminus S=\emptyset,\\[1mm]
    \left\{
      \displaystyle
      \min_{s\in S}
      \widetilde E_t^{(s,\widehat a_t^{\mathrm{out}})}
      \ge
      \alpha_{\mathrm{pw}}^{-1}
    \right\},
    & \cA_t\setminus S\neq\emptyset.
  \end{cases}
\]
The top-$k$ CITE stopping time is
\begin{equation}\label{eq:top-k-tau}
  \tau^{(k)}
  :=
  \inf\left\{
    t\ge1:
    \mathsf{PW}_t^{(k)}
    \text{ and }
    \min_{s\in S}L_t(s)>U_t
  \right\}.
\end{equation}

\begin{theorem}[Top-$k$ Type-I and power]
\label{thm:top-k}
Run the top-$k$ CITE~\eqref{eq:top-k-tau} under the same
component-budget and selected-grid-weight bounded-below assumptions
as Theorem~\ref{thm:power-iut-clean}, with the same per-component
budget as the single-target CITE,
$\alpha_{\mathrm{pw}} + \alpha_r + \alpha_u \leq \varepsilon$.
\begin{enumerate}[label=\textup{(\alph*)},leftmargin=2.2em,topsep=2pt]
\item \textbf{Type-I:} Under Assumption~\ref{asn:dgp},
$\sup_{P \in H_{0,S}^{(k)}} \PP_P(\tau^{(k)} < \infty) \leq \varepsilon$.
\item \textbf{Power:} Suppose $\delta_S > 0$ and the grids satisfy the
analogue of Condition~\ref{cond:grid} with $\delta_S$ in place of
$\delta$.  Define the threshold
\[
  T_{\mathrm{cl}}^{(k)}  :=  C_0 \left(
    \frac{\log(k/\varepsilon)+\log(1/\delta_S)}{\delta_S^{2}}
     +  \frac{\log(k/\varepsilon)+\log(1/p_S)}{p_S}
  \right)
\]
for a universal constant $C_0$.  For \emph{fixed} $k$, as
$\varepsilon \to 0$,
\[
  \EE[\tau^{(k)}]  \leq  T_{\mathrm{cl}}^{(k)}\bigl(1 + o(1)\bigr)
   =  O(T_{\mathrm{cl}}^{(k)}).
\]
The upper bound extends to growing $k$ provided
$k = O(T_{\mathrm{cl}}^{(k)})$, e.g.\ when $\delta_S^{-2} \log(k/\varepsilon)$
dominates $k$.
\end{enumerate}
\end{theorem}

\begin{proof}
\textit{(a) Type-I.}
Fix any $P \in H_{0,S}^{(k)}$.  Choose
$s^\star\in S$ such that
\[
  p_{s^\star}=\min_{s\in S}p_s .
\]
Since $P\in H_{0,S}^{(k)}$,
\[
  p_{s^\star}\le \sup_{a\notin S}p_a .
\]
If this supremum is positive, Lemma~\ref{lem:attained-supremum} gives
$a^\star\notin S$ with $p_{a^\star}\ge p_{s^\star}$.  If $p_{s^\star}=0$, any outsider $a^\star\notin S$ can be used.

Decompose
\[
  \{\tau^{(k)} < \infty\}
   \subseteq 
  \underbrace{\{\tau^{(k)} < \infty,  a^\star \in \cA_{\tau^{(k)}}\}}_{E_1}
   \cup 
  \underbrace{\{\tau^{(k)} < \infty,  a^\star \notin \cA_{\tau^{(k)}}\}}_{E_2}.
\]

On $E_1$, the procedure's pairwise condition holds at $\tau^{(k)}$ for
\emph{every} $s \in S$, including the witness $s^\star$:
$ E_{\tau^{(k)}}^{(s^\star, \widehat a^{\mathrm{out}}_{\tau^{(k)}})} \geq 1/\alpha_{\mathrm{pw}}$.
Runner-up monotonicity (Lemma~\ref{lem:monotone}) restricted to
outsiders (since $a^\star \in \cA_{\tau^{(k)}} \setminus S$ implies
$N_{\tau^{(k)}}(a^\star) \leq N_{\tau^{(k)}}(\widehat a^{\mathrm{out}}_{\tau^{(k)}})$) gives
$ E_{\tau^{(k)}}^{(s^\star, a^\star)} \geq  E_{\tau^{(k)}}^{(s^\star, \widehat a^{\mathrm{out}}_{\tau^{(k)}})} \geq 1/\alpha_{\mathrm{pw}}$.
Since $s^\star, a^\star$ are deterministic and
$( E_t^{(s^\star, a^\star)})_{t \geq 0}$ is an NSM with
initial value $\leq 1$ under $p_{s^\star} \leq p_{a^\star}$
(Corollary~\ref{cor:mixture-nsm}), Ville's inequality~\eqref{eq:ville}
yields $\PP(E_1) \leq \alpha_{\mathrm{pw}}$.

On $E_2$, condition~(b) at $\tau^{(k)}$ holds for the witness
$s^\star$: $L_{\tau^{(k)}}(s^\star) > U_{\tau^{(k)}}$.  Define the two
failure events
\[
  F_r := \{\exists t : L_t(s^\star) > p_{s^\star}\},
  \qquad
  F_u := \{\exists a, t : N_t(a) = 0,\ p_a > U_t\}.
\]
Proposition~\ref{prop:lcb-valid} (applied to the single index
$s^\star$) gives $\PP(F_r) \leq \alpha_r$, and
Proposition~\ref{prop:unseen-valid} gives $\PP(F_u) \leq \alpha_u$.
On $E_2 \setminus (F_r \cup F_u)$,
$L_{\tau^{(k)}}(s^\star) \leq p_{s^\star}$ and
$p_{a^\star} \leq U_{\tau^{(k)}}$ (since $a^\star \notin \cA_{\tau^{(k)}}$),
yielding the contradiction
$p_{a^\star} \leq U_{\tau^{(k)}} < L_{\tau^{(k)}}(s^\star) \leq p_{s^\star}$,
contrary to $p_{a^\star} \geq p_{s^\star}$.  Hence
$\PP(E_2) \leq \alpha_r + \alpha_u$.

A union bound gives
$\PP(\tau^{(k)} < \infty) \leq \alpha_{\mathrm{pw}} + \alpha_r + \alpha_u \leq \varepsilon$.

\textit{(b) Power.}
The single-target CITE of Definition~\ref{def:stopping} applied to a
target $s \in S$ is \emph{not} appropriate here, because its pairwise
check ranges over all of $\cA \setminus \{s\}$ and may include
competitors $s' \in S$ with $p_{s'} \geq p_s$, on which the pairwise
null is true; the resulting stopping time would be $+\infty$ a.s.
Instead, fix $s \in S$ and consider the
\emph{outsider-restricted single-target CITE for target $s$}, defined
exactly as Definition~\ref{def:stopping} except that the empirical
runner-up $\widehat a_t$ is replaced by
$\widehat a_t^{\mathrm{out}} := \arg\max_{a \in \cA_t \setminus S} N_t(a)$
and the pairwise threshold is $k/\alpha_{\mathrm{pw}}$, the LCB
threshold is $k/\alpha_r$, the unseen budget $\alpha_u$ is unchanged.
We adopt the same convention as for the actual top-$k$ procedure: if
$\cA_t \setminus S = \emptyset$, the per-$s$ pairwise condition is
treated as vacuously satisfied at round $t$.
The tightened thresholds $k/\alpha_{\mathrm{pw}}$ and $k/\alpha_r$ are
\emph{strictly stronger sufficient} requirements than the actual top-$k$
stopping criteria (which use the original thresholds $1/\alpha_{\mathrm{pw}}$
and the LCB inequality $L_t(s) > U_t$); they are introduced solely so
that a Bonferroni union bound over $s \in S$ on the per-target failure
events yields a uniform $\varepsilon$-level conclusion.
Write $\tau_s^{\mathrm{out}}$ for its stopping time.

We use a fixed-time tail bound rather than a max-of-first-hitting-times
argument (the latter fails because pairwise e-processes and LCBs are
not monotone in $t$).  For each fixed $t \geq 1$,
\begin{equation}\label{eq:topk-fail-decomp}
  \{\tau^{(k)} > t\}
  \subseteq
  \bigcup_{s\in S}
  \left(
    \left\{
      \cA_t\setminus S\neq\emptyset,\
       E_t^{(s,\widehat a_t^{\mathrm{out}})}
      <
      \alpha_{\mathrm{pw}}^{-1}
    \right\}
    \cup
    \{L_t(s)\le U_t\}
  \right).
\end{equation}
Apply the proof of Theorem~\ref{thm:power-iut-clean} to the
\emph{outsider-restricted single-target CITE for target $s$}, with
$\cA \setminus \{r\}$ replaced by $\cA \setminus S$ in the
pairwise/LCB analysis (the unseen bound $U_t$ is unchanged since it is
already over all of $\cA$): the runner-up monotonicity
(Lemma~\ref{lem:monotone}) holds within any subset, the LCB
construction is local to target $s$, and the effective modal gap
$p_s - \max_{a \notin S} p_a \geq \delta_S$ and target mass $p_s \geq p_S$.
With per-$s$ thresholds tightened to $k/\alpha_{\mathrm{pw}}$ for the
pairwise check and $k/\alpha_r$ for the LCB (a Bonferroni split over
$|S| = k$), the same McDiarmid + Bernstein arguments give, for each
fixed $s$ and each $t \geq T_{\mathrm{cl}}^{(k)}$,
\[
  \PP \bigl( E_t^{(s, \widehat a_t^{\mathrm{out}})} < 1/\alpha_{\mathrm{pw}}
        \text{or}  L_t(s) \leq U_t\bigr)
   \leq  2(et)^{-2},
\]
where
$T_{\mathrm{cl}}^{(k)} = C_0\bigl((\log(k/\varepsilon)+\log(1/\delta_S))/\delta_S^{2}
+ (\log(k/\varepsilon)+\log(1/p_S))/p_S\bigr)$ for a universal constant $C_0$.
Substituting into~\eqref{eq:topk-fail-decomp} and applying a union
bound over $s \in S$,
\[
  \PP(\tau^{(k)} > t)  \leq  2 k (e t)^{-2}
  \qquad \text{for all } t \geq T_{\mathrm{cl}}^{(k)}.
\]
The tail-sum identity ($\sum_{t \geq T} 1/t^2 \leq 2/T$ for $T \geq 2$)
yields
\[
  \EE[\tau^{(k)}]
   \leq  T_{\mathrm{cl}}^{(k)}
   + \sum_{t > T_{\mathrm{cl}}^{(k)}} 2 k (et)^{-2}
   \leq  T_{\mathrm{cl}}^{(k)}\bigl(1 + o(1)\bigr)
  \quad (\varepsilon \to 0,\ k \text{ fixed}),
\]
which gives the claimed bound.
\end{proof}

\begin{remark}
The top-$k$ extension preserves the runner-up reduction within the
``outsider'' set, so the per-round cost is $k$ pairwise checks and
$k$ LCB updates, plus one shared unseen bound.
The case $k = 1$ recovers CITE of
Definition~\ref{def:stopping}, with Type-I and rate guarantees given
by Theorems~\ref{thm:main-iut} and~\ref{thm:power-iut-clean}
respectively.
\end{remark}

%% ============================================================
%% Appendix: Supplementary Simulation Experiments
%% ============================================================

\section{Simulation Details}
\label{app:sim}

\subsection{Distribution Settings}
\label{app:sim-settings}

Table~\ref{tab:settings} summarizes the five distribution settings
used in Section~\ref{sec:experiments}.  In each setting, the tail
probabilities follow a shifted-Zipf distribution capped so that the
stated modal gap $\delta=p_{(1)}-p_{(2)}$ is preserved exactly.

\begin{table}[H]
\caption{Distribution settings for the simulation study.}
\label{tab:settings}
\centering\small
\begin{tabular}{clcccl}
\toprule
Setting & Name & $K$ & $p_r$ & $\delta$ & Tail \\
\midrule
1 & LLM-like Zipf (diffuse) & 5000 & 0.24 & 0.215 & $\text{Zipf}_{1.1}$ on 4987 labels \\
2 & Concentrated (strong mode) & 100 & 0.60 & 0.450 & Uniform on 94 labels \\
3 & Near-tie (small gap) & 500 & 0.12 & 0.010 & $\text{Zipf}_{1.3}$ on 497 labels \\
4 & Very diffuse (huge category set) & 10000 & 0.06 & 0.010 & $\text{Zipf}_{1.0}$ on 9998 labels \\
5 & Moderate LLM-like & 1000 & 0.35 & 0.150 & $\text{Zipf}_{1.2}$ on 997 labels \\
\bottomrule
\end{tabular}
\end{table}

\subsection{Results for All Reported Settings}
\label{app:sim-tables}

\begin{figure}[p]
\centering
\includegraphics[width=\textwidth]{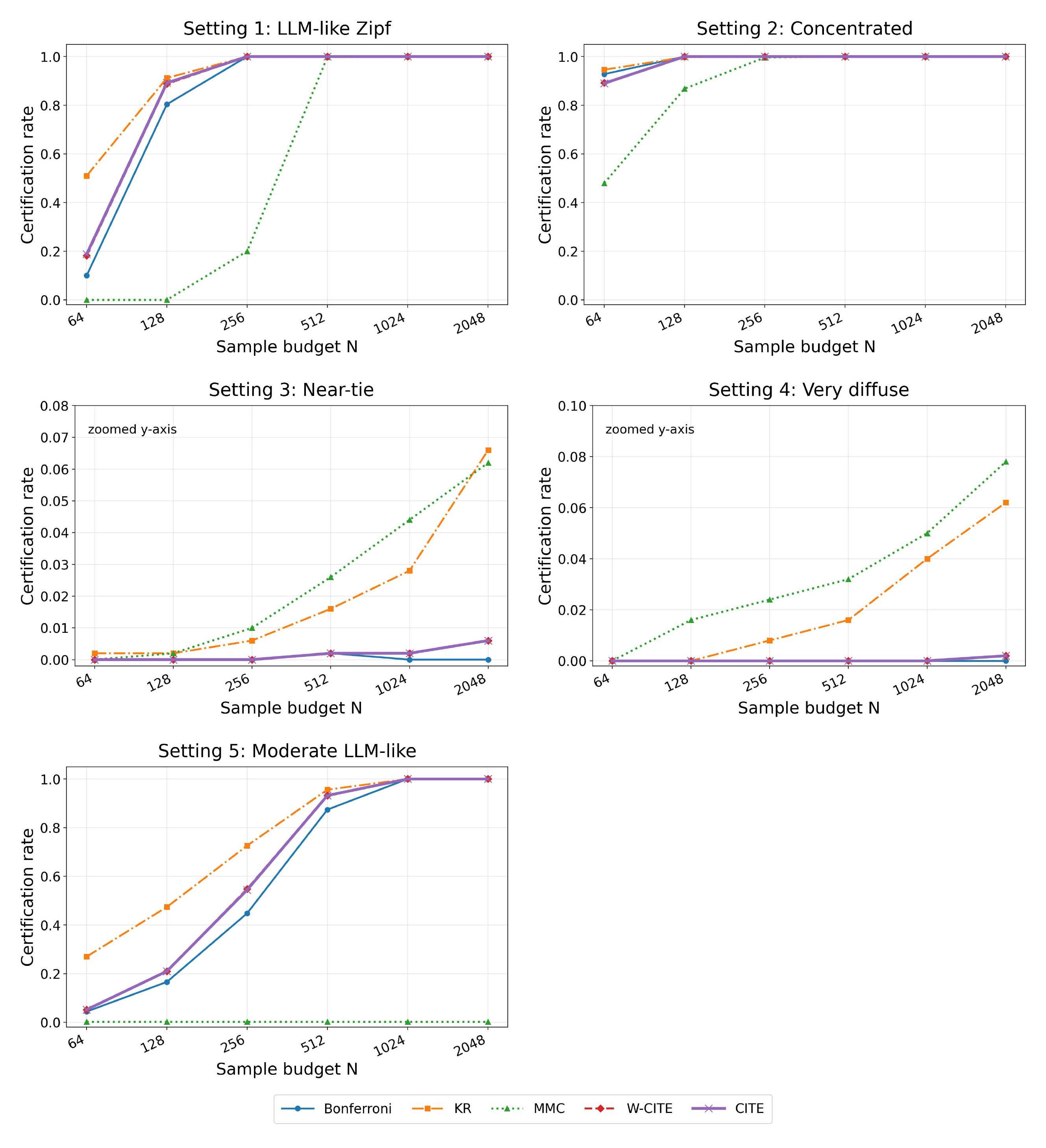}
\caption{Certification rate versus sample budget $N$ for CITE, W-CITE,
fixed-sample Bonferroni, the sample-split fixed-target test KR
\citep{kim2025locally}, and the leader-tracking MMC
\citep{cordero2025certified}, across all five reported simulation
settings.  The panels for Settings~3 and~4 use zoomed vertical axes
because all methods remain low-power over the evaluated budget range.
CITE matches or exceeds Bonferroni in most regimes, with the easy
concentrated Setting~2 saturating quickly for both methods.  Settings
3--4 are also the regimes where KR and MMC show nonzero false
certifications under Case~B (Table~\ref{tab:type1-sim}).}
\label{fig:power}
\end{figure}

\paragraph{Type-I error control.}
Across all five settings and all budgets
$N\in\{64,128,256,512,1024,2048\}$, CITE, W-CITE, and Bonferroni have
empirical false-certification rate exactly zero in every Case~B
experiment.  KR also has zero false certifications in Settings~1, 2,
and~5, with only small nonzero rates on Settings~3--4.  MMC shows
nonzero false certifications in the same two regimes, illustrating that
leader-tracking certificates are not designed for the fixed-target null
considered here.  Table~\ref{tab:type1-sim} reports the full
per-budget Case~B false-certification rates.

\begin{table}[H]
\caption{Case~B false-certification rates by budget.  Entries are
empirical false-certification rates at level $\varepsilon=0.05$ under
the fixed-target null.}
\label{tab:type1-sim}
\centering\scriptsize
\setlength{\tabcolsep}{4pt}
\begin{tabular}{llccccc}
\toprule
Setting & $N$ & Bonferroni & CITE & W-CITE & KR & MMC \\
\midrule
1: LLM-like Zipf & 64   & .000 & .000 & .000 & .000 & .000 \\
                 & 128  & .000 & .000 & .000 & .000 & .000 \\
                 & 256  & .000 & .000 & .000 & .000 & .000 \\
                 & 512  & .000 & .000 & .000 & .000 & .000 \\
                 & 1024 & .000 & .000 & .000 & .000 & .000 \\
                 & 2048 & .000 & .000 & .000 & .000 & .000 \\
\midrule
2: Concentrated  & 64   & .000 & .000 & .000 & .000 & .000 \\
                 & 128  & .000 & .000 & .000 & .000 & .000 \\
                 & 256  & .000 & .000 & .000 & .000 & .000 \\
                 & 512  & .000 & .000 & .000 & .000 & .000 \\
                 & 1024 & .000 & .000 & .000 & .000 & .000 \\
                 & 2048 & .000 & .000 & .000 & .000 & .000 \\
\midrule
3: Near-tie      & 64   & .000 & .000 & .000 & .002 & .000 \\
                 & 128  & .000 & .000 & .000 & .000 & .000 \\
                 & 256  & .000 & .000 & .000 & .000 & .002 \\
                 & 512  & .000 & .000 & .000 & .002 & .004 \\
                 & 1024 & .000 & .000 & .000 & .002 & .006 \\
                 & 2048 & .000 & .000 & .000 & .002 & .006 \\
\midrule
4: Very diffuse  & 64   & .000 & .000 & .000 & .000 & .000 \\
                 & 128  & .000 & .000 & .000 & .000 & .012 \\
                 & 256  & .000 & .000 & .000 & .000 & .012 \\
                 & 512  & .000 & .000 & .000 & .000 & .012 \\
                 & 1024 & .000 & .000 & .000 & .002 & .014 \\
                 & 2048 & .000 & .000 & .000 & .000 & .014 \\
\midrule
5: Moderate LLM-like & 64   & .000 & .000 & .000 & .000 & .000 \\
                     & 128  & .000 & .000 & .000 & .000 & .000 \\
                     & 256  & .000 & .000 & .000 & .000 & .000 \\
                     & 512  & .000 & .000 & .000 & .000 & .000 \\
                     & 1024 & .000 & .000 & .000 & .000 & .000 \\
                     & 2048 & .000 & .000 & .000 & .000 & .000 \\
\bottomrule
\end{tabular}
\end{table}

\paragraph{Power (Case~A).}
Tables~\ref{tab:full-s1}--\ref{tab:full-s5} report certification rates
for all five settings.  Setting~2 saturates quickly for all methods,
whereas Settings~3 and~4 are the most difficult regimes.  Setting~5
highlights a diffuse-tail regime in which CITE substantially outperforms
MMC at moderate budgets.  Bonferroni can be slightly more powerful in
the easiest concentrated regime at the smallest budget, but CITE
matches or improves on it in the remaining reported regimes and
budgets.

\begin{table}[H]
\caption{Setting 1 (LLM-like Zipf): certification rate, Case A.}
\label{tab:full-s1}
\centering\small
\begin{tabular}{lccccc c}
\toprule
$N$ & Bonferroni & CITE & W-CITE & KR & MMC & $\bar{\tau}_{\text{CITE}}$ \\
\midrule
64   & .100 & .188 & .182 & .510 & .000 & 53 \\
128  & .804 & .892 & .888 & .912 & .000 & 84 \\
256  & 1.00 & 1.00 & 1.00 & 1.00 & .200 & 90 \\
512  & 1.00 & 1.00 & 1.00 & 1.00 & .998 & 90 \\
1024 & 1.00 & 1.00 & 1.00 & 1.00 & 1.00 & 90 \\
2048 & 1.00 & 1.00 & 1.00 & 1.00 & 1.00 & 90 \\
\bottomrule
\end{tabular}
\end{table}

\begin{table}[H]
\caption{Setting 2 (Concentrated, strong mode): certification rate, Case A.}
\label{tab:full-s2}
\centering\small
\begin{tabular}{lccccc c}
\toprule
$N$ & Bonferroni & CITE & W-CITE & KR & MMC & $\bar{\tau}_{\text{CITE}}$ \\
\midrule
64   & .928 & .890 & .892 & .946 & .480 & 36 \\
128  & 1.00 & 1.00 & 1.00 & .998 & .868 & 41 \\
256  & 1.00 & 1.00 & 1.00 & 1.00 & .996 & 41 \\
512  & 1.00 & 1.00 & 1.00 & 1.00 & 1.00 & 41 \\
1024 & 1.00 & 1.00 & 1.00 & 1.00 & 1.00 & 41 \\
2048 & 1.00 & 1.00 & 1.00 & 1.00 & 1.00 & 41 \\
\bottomrule
\end{tabular}
\end{table}

\begin{table}[H]
\caption{Setting 3 (Near-tie, $\delta=0.01$): certification rate, Case A.}
\label{tab:full-s3}
\centering\small
\begin{tabular}{lccccc}
\toprule
$N$ & Bonferroni & CITE & W-CITE & KR & MMC \\
\midrule
64   & .000 & .000 & .000 & .002 & .000 \\
128  & .000 & .000 & .000 & .002 & .002 \\
256  & .000 & .000 & .000 & .006 & .010 \\
512  & .002 & .002 & .002 & .016 & .026 \\
1024 & .000 & .002 & .002 & .028 & .044 \\
2048 & .000 & .006 & .006 & .066 & .062 \\
\bottomrule
\end{tabular}
\end{table}

\begin{table}[H]
\caption{Setting 4 (Very diffuse, huge category set): certification rate, Case A.}
\label{tab:full-s4}
\centering\small
\begin{tabular}{lccccc c}
\toprule
$N$ & Bonferroni & CITE & W-CITE & KR & MMC & $\bar{\tau}_{\text{CITE}}$ \\
\midrule
64   & .000 & .000 & .000 & .000 & .000 & --- \\
128  & .000 & .000 & .000 & .000 & .016 & --- \\
256  & .000 & .000 & .000 & .008 & .024 & --- \\
512  & .000 & .000 & .000 & .016 & .032 & --- \\
1024 & .000 & .000 & .000 & .040 & .050 & --- \\
2048 & .000 & .002 & .002 & .062 & .078 & 1149 \\
\bottomrule
\end{tabular}
\end{table}

\begin{table}[H]
\caption{Setting 5 (Moderate LLM-like): certification rate, Case A.}
\label{tab:full-s5}
\centering\small
\begin{tabular}{lccccc c}
\toprule
$N$ & Bonferroni & CITE & W-CITE & KR & MMC & $\bar{\tau}_{\text{CITE}}$ \\
\midrule
64   & .044 & .052 & .052 & .270 & .002 & 49 \\
128  & .166 & .210 & .210 & .474 & .002 & 84 \\
256  & .448 & .544 & .548 & .726 & .002 & 151 \\
512  & .874 & .932 & .932 & .956 & .002 & 234 \\
1024 & 1.00 & 1.00 & 1.00 & 1.00 & .002 & 259 \\
2048 & 1.00 & 1.00 & 1.00 & 1.00 & .002 & 259 \\
\bottomrule
\end{tabular}
\end{table}

\subsection{Bottleneck Analysis}
\label{app:sim-bottleneck}

\begin{figure}[H]
\centering
\includegraphics[width=\textwidth]{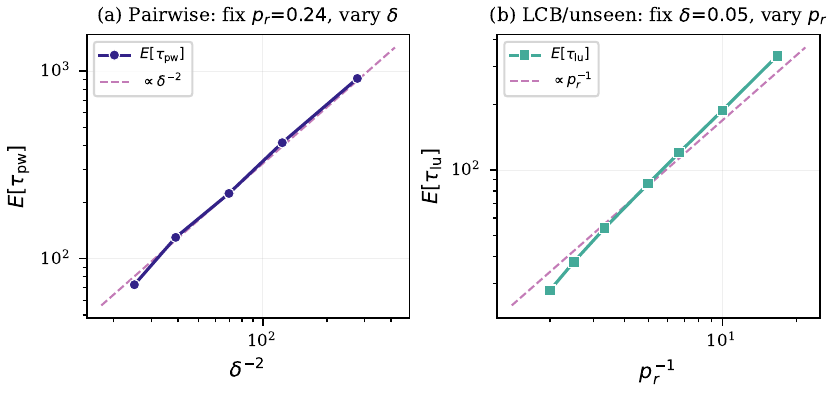}
\caption{Component-level diagnostics for the two terms in the
stopping-time upper bound of Theorem~\ref{thm:power-iut-clean}, on
log-log axes. Dashed reference
lines show the theoretical slopes.  (a)~Fixing $p_r = 0.24$ and
varying $\delta$, the pairwise component
$\EE[\tau_{\mathrm{pw}}]$ (first time the pairwise e-value crosses
$1/\alpha_{\mathrm{pw}}$) is consistent with the $\delta^{-2}$ term.
(b)~Fixing $\delta = 0.05$ and varying $p_r$, the LCB/unseen
component $\EE[\tau_{\mathrm{lu}}]$ (first time
$L_t(r) > U_t$) is consistent with the $p_r^{-1}$ term.}
\label{fig:decomp}
\end{figure}

\subsection{Logistic Weight Model Saturation}
\label{app:sim-weight-logistic}

As discussed in Section~\ref{sec:experiments}, the logistic weight
model can saturate when $K$ is large.  For Settings~1 and~5, sweeping
$\beta\in\{0.5,1,2,4,8,16\}$ produces certification rates and stopping
times that are nearly unchanged relative to unweighted CITE.  The
effective gap ratio $\Delta_w/\mu_r$ is essentially the same as
$\delta/p_r$ for $\beta\ge2$.  This occurs because the top categories
are already assigned weights close to one, so the logistic
transformation does not substantially separate the mode from its
runner-up.

\subsection{Rank-Based Weight Model}
\label{app:sim-weight-rank}

The rank-based model assigns $E[W \mid X = a] = 0.95 \cdot e^{-\gamma \cdot \mathrm{rank}(a)}$
for the top $K_0 = 10$ categories and $w_{\mathrm{low}} = 0.1$ otherwise,
with additive $\mathrm{Uniform}(-0.05, 0.05)$ noise clipped to $[0.01, 1]$.

Table~\ref{tab:weight-rank} reports detailed results for Setting~5 at
$N=512$. We omit a separate figure here to avoid duplicating
Figure~\ref{fig:power}; the table directly shows how the effective
weighted gap and stopping time change with the rank-decay parameter.

\begin{table}[H]
\caption{Setting 5: rank-based W-CITE at $N=512$ in the supplementary
weight sweep.}\label{tab:weight-rank}
\centering\small
\begin{tabular}{cccccc}
\toprule
$\gamma$ & W-CITE rate & $\bar{\tau}_{\text{W-CITE}}$ & $\bar{\tau}_{\text{CITE}}$ & $\Delta_w/\mu_r$ & $\delta/p_r$ \\
\midrule
0.0 & .920 & 237 & 235 & 0.429 & 0.429 \\
0.1 & .990 & 185 & 235 & 0.483 & 0.429 \\
0.2 & 1.00 & 142 & 235 & 0.532 & 0.429 \\
0.5 & 1.00 & 71  & 235 & 0.653 & 0.429 \\
1.0 & 1.00 & 48  & 235 & 0.790 & 0.429 \\
2.0 & 1.00 & 45  & 235 & 0.923 & 0.429 \\
3.0 & 1.00 & 45  & 235 & 0.972 & 0.429 \\
\bottomrule
\end{tabular}
\end{table}

In this synthetic rank-based model, performance improves as the
effective gap ratio $\Delta_w/\mu_r$ increases.  A modest
$\gamma=0.2$ already reaches full certification in this sweep and
reduces $\bar{\tau}$ by about 40\%; at $\gamma=1.0$, the reduction is
about 80\%.  For Setting~1, the rank-based model with
$\gamma\ge0.1$ reaches 99.6\% certification at $N=128$ (versus 88.6\%
unweighted) with $\bar{\tau}\approx72$ (versus 92).

\section{LLM Self-Consistency Details}
\label{app:llm}

Table~\ref{tab:llm-main-case-a} reports the certification rate averaged
across problems for each (model, dataset, budget) under Case A
(target = mode); Table~\ref{tab:llm-main-case-b} reports the
corresponding Case B (target = runner-up) certification rates as a
sanity check on Type-I error.  For each
(model, problem) pair we draw 1{,}000 i.i.d.\ samples, then bootstrap
500 replicates of each budget $N \in \{64, 128, 256, 512, 1024\}$ from
the empirical answer pool to compute certification rates.

% Combined LLM main table (Case A). Auto-generated by build_llm_main.py from tables/
% Source: {model}_{dataset}-main-case-a.tex for 6 conditions
\begin{table}[t]
\caption{Certification rate on LLM self-consistency outputs (Case A, target = mode), averaged across problems. $\bar{K}_N$: Monte Carlo mean number of observed answer categories at budget $N$. $\bar{\tau}$ columns report mean stopping time among certified replicates.}
\label{tab:llm-main-case-a}
\centering\small
\resizebox{\textwidth}{!}{%
\begin{tabular}{ll c ccccc c ccc}
\toprule
Model & Dataset & $N$ & Bonferroni & CITE & W-CITE & KR & MMC & $\bar{K}_N$ & $\bar{\tau}_{\mathrm{CITE}}$ & $\bar{\tau}_{\mathrm{W\mbox{-}CITE}}$ & $\bar{\tau}_{\mathrm{MMC}}$ \\
\midrule
\multirow{15}{*}{Qwen3-30B} & \multirow{5}{*}{AIME 2026} & 64 & 0.499$\pm$0.002 & 0.724$\pm$0.001 & 0.727$\pm$0.001 & 0.778$\pm$0.002 & 0.671$\pm$0.001 & 7.1$\pm$0.0 & 21.8$\pm$0.3 & 22.8$\pm$0.2 & 15.9$\pm$0.3 \\
  & & 128 & 0.597$\pm$0.002 & 0.781$\pm$0.001 & 0.785$\pm$0.001 & 0.827$\pm$0.002 & 0.705$\pm$0.001 & 10.2$\pm$0.0 & 32.4$\pm$0.5 & 33.0$\pm$0.5 & 22.5$\pm$0.5 \\
  & & 256 & 0.683$\pm$0.002 & 0.833$\pm$0.002 & 0.838$\pm$0.002 & 0.880$\pm$0.002 & 0.744$\pm$0.001 & 14.5$\pm$0.0 & 49.1$\pm$0.6 & 55.5$\pm$0.8 & 32.1$\pm$0.6 \\
  & & 512 & 0.770$\pm$0.002 & 0.902$\pm$0.001 & 0.906$\pm$0.001 & 0.929$\pm$0.001 & 0.779$\pm$0.001 & 19.9$\pm$0.0 & 81.8$\pm$1.3 & 82.0$\pm$1.1 & 45.6$\pm$1.1 \\
  & & 1024 & 0.825$\pm$0.001 & 0.944$\pm$0.001 & 0.946$\pm$0.001 & 0.959$\pm$0.001 & 0.798$\pm$0.000 & 25.7$\pm$0.0 & 113.7$\pm$1.8 & 113.2$\pm$1.4 & 57.8$\pm$2.2 \\
\cmidrule{2-12}
  & \multirow{5}{*}{\shortstack{FrontierScience-\\Olympiad}} & 64 & 0.284$\pm$0.002 & 0.288$\pm$0.002 & 0.289$\pm$0.002 & 0.328$\pm$0.003 & 0.121$\pm$0.002 & 28.0$\pm$0.0 & 37.9$\pm$1.9 & 40.9$\pm$1.6 & 22.9$\pm$0.9 \\
  & & 128 & 0.312$\pm$0.002 & 0.315$\pm$0.002 & 0.315$\pm$0.002 & 0.373$\pm$0.003 & 0.146$\pm$0.003 & 46.0$\pm$0.0 & 65.5$\pm$2.5 & 64.3$\pm$2.9 & 35.9$\pm$1.1 \\
  & & 256 & 0.350$\pm$0.003 & 0.356$\pm$0.003 & 0.359$\pm$0.003 & 0.443$\pm$0.004 & 0.198$\pm$0.003 & 73.7$\pm$0.1 & 110.0$\pm$2.7 & 112.0$\pm$2.5 & 60.6$\pm$1.5 \\
  & & 512 & 0.429$\pm$0.003 & 0.449$\pm$0.004 & 0.453$\pm$0.004 & 0.536$\pm$0.004 & 0.245$\pm$0.002 & 112.9$\pm$0.1 & 195.5$\pm$3.3 & 193.5$\pm$3.2 & 97.0$\pm$2.2 \\
  & & 1024 & 0.530$\pm$0.003 & 0.563$\pm$0.004 & 0.567$\pm$0.004 & 0.633$\pm$0.004 & 0.281$\pm$0.002 & 160.0$\pm$0.1 & 436.1$\pm$15.0 & 439.2$\pm$12.3 & 138.5$\pm$2.9 \\
\cmidrule{2-12}
  & \multirow{5}{*}{ER-REASON} & 64 & 0.102$\pm$0.000 & 0.983$\pm$0.000 & 0.983$\pm$0.000 & 0.990$\pm$0.000 & 0.988$\pm$0.000 & 1.1$\pm$0.0 & 10.6$\pm$0.0 & 11.5$\pm$0.0 & 8.2$\pm$0.0 \\
  & & 128 & 0.110$\pm$0.000 & 0.996$\pm$0.000 & 0.996$\pm$0.000 & 0.998$\pm$0.000 & 0.996$\pm$0.000 & 1.1$\pm$0.0 & 11.3$\pm$0.0 & 12.2$\pm$0.0 & 8.8$\pm$0.0 \\
  & & 256 & 0.111$\pm$0.000 & 1.000$\pm$0.000 & 1.000$\pm$0.000 & 1.000$\pm$0.000 & 1.000$\pm$0.000 & 1.1$\pm$0.0 & 11.8$\pm$0.0 & 12.6$\pm$0.0 & 9.3$\pm$0.0 \\
  & & 512 & 0.111$\pm$0.000 & 1.000$\pm$0.000 & 1.000$\pm$0.000 & 1.000$\pm$0.000 & 1.000$\pm$0.000 & 1.1$\pm$0.0 & 11.8$\pm$0.0 & 12.7$\pm$0.0 & 9.3$\pm$0.1 \\
  & & 1024 & 0.111$\pm$0.000 & 1.000$\pm$0.000 & 1.000$\pm$0.000 & 1.000$\pm$0.000 & 1.000$\pm$0.000 & 1.1$\pm$0.0 & 11.8$\pm$0.0 & 12.7$\pm$0.0 & 9.3$\pm$0.1 \\
\midrule
\multirow{15}{*}{gpt-oss-20b} & \multirow{5}{*}{AIME 2026} & 64 & 0.563$\pm$0.002 & 0.596$\pm$0.002 & 0.599$\pm$0.002 & 0.666$\pm$0.002 & 0.459$\pm$0.001 & 12.6$\pm$0.0 & 29.3$\pm$0.2 & 29.6$\pm$0.2 & 21.0$\pm$0.6 \\
  & & 128 & 0.659$\pm$0.002 & 0.693$\pm$0.001 & 0.701$\pm$0.002 & 0.742$\pm$0.002 & 0.486$\pm$0.001 & 18.3$\pm$0.0 & 43.8$\pm$1.6 & 43.5$\pm$1.6 & 27.9$\pm$0.8 \\
  & & 256 & 0.740$\pm$0.002 & 0.765$\pm$0.001 & 0.778$\pm$0.001 & 0.795$\pm$0.001 & 0.513$\pm$0.001 & 25.7$\pm$0.0 & 62.8$\pm$2.9 & 60.7$\pm$3.2 & 40.3$\pm$1.5 \\
  & & 512 & 0.799$\pm$0.001 & 0.811$\pm$0.001 & 0.812$\pm$0.001 & 0.830$\pm$0.001 & 0.538$\pm$0.001 & 34.8$\pm$0.0 & 84.2$\pm$2.9 & 80.7$\pm$3.1 & 58.0$\pm$2.1 \\
  & & 1024 & 0.825$\pm$0.001 & 0.829$\pm$0.001 & 0.829$\pm$0.001 & 0.848$\pm$0.001 & 0.554$\pm$0.001 & 44.0$\pm$0.0 & 176.7$\pm$6.0 & 126.6$\pm$3.7 & 91.0$\pm$4.5 \\
\cmidrule{2-12}
  & \multirow{5}{*}{\shortstack{FrontierScience-\\Olympiad}} & 64 & 0.193$\pm$0.001 & 0.187$\pm$0.002 & 0.147$\pm$0.002 & 0.209$\pm$0.002 & 0.192$\pm$0.001 & 37.8$\pm$0.0 & 32.3$\pm$2.2 & 28.7$\pm$0.5 & 22.1$\pm$0.4 \\
  & & 128 & 0.205$\pm$0.001 & 0.213$\pm$0.002 & 0.188$\pm$0.001 & 0.255$\pm$0.003 & 0.199$\pm$0.000 & 65.1$\pm$0.1 & 67.6$\pm$2.4 & 40.4$\pm$0.7 & 24.0$\pm$0.5 \\
  & & 256 & 0.237$\pm$0.002 & 0.269$\pm$0.003 & 0.201$\pm$0.001 & 0.330$\pm$0.004 & 0.200$\pm$0.000 & 108.5$\pm$0.1 & 144.6$\pm$2.9 & 108.9$\pm$1.8 & 24.3$\pm$0.5 \\
  & & 512 & 0.340$\pm$0.003 & 0.382$\pm$0.003 & 0.242$\pm$0.002 & 0.446$\pm$0.004 & 0.200$\pm$0.000 & 172.0$\pm$0.1 & 258.7$\pm$4.5 & 319.0$\pm$2.3 & 24.3$\pm$0.5 \\
  & & 1024 & 0.429$\pm$0.002 & 0.475$\pm$0.003 & 0.428$\pm$0.004 & 0.559$\pm$0.004 & 0.200$\pm$0.000 & 250.3$\pm$0.1 & 494.7$\pm$13.2 & 499.9$\pm$3.7 & 24.3$\pm$0.5 \\
\cmidrule{2-12}
  & \multirow{5}{*}{ER-REASON} & 64 & 0.623$\pm$0.001 & 0.640$\pm$0.001 & 0.644$\pm$0.001 & 0.722$\pm$0.001 & 0.654$\pm$0.001 & 3.2$\pm$0.0 & 26.6$\pm$0.2 & 28.4$\pm$0.2 & 22.1$\pm$0.2 \\
  & & 128 & 0.726$\pm$0.001 & 0.730$\pm$0.001 & 0.739$\pm$0.001 & 0.795$\pm$0.001 & 0.727$\pm$0.001 & 3.6$\pm$0.0 & 41.4$\pm$0.5 & 42.1$\pm$0.5 & 33.1$\pm$0.6 \\
  & & 256 & 0.813$\pm$0.001 & 0.816$\pm$0.001 & 0.821$\pm$0.001 & 0.857$\pm$0.001 & 0.780$\pm$0.001 & 3.8$\pm$0.0 & 63.0$\pm$0.7 & 60.7$\pm$0.6 & 47.5$\pm$0.7 \\
  & & 512 & 0.873$\pm$0.001 & 0.870$\pm$0.001 & 0.875$\pm$0.001 & 0.903$\pm$0.001 & 0.823$\pm$0.001 & 4.0$\pm$0.0 & 89.7$\pm$1.0 & 91.5$\pm$1.6 & 71.4$\pm$1.4 \\
  & & 1024 & 0.921$\pm$0.001 & 0.915$\pm$0.001 & 0.919$\pm$0.001 & 0.939$\pm$0.001 & 0.861$\pm$0.001 & 4.2$\pm$0.0 & 128.3$\pm$2.3 & 127.5$\pm$2.2 & 110.8$\pm$2.1 \\
\bottomrule
\end{tabular}%
}
\end{table}

% Combined LLM main table (Case B). Auto-generated by build_llm_main.py from tables/
% Source: {model}_{dataset}-main-case-b.tex for 6 conditions
\begin{table}[t]
\caption{Certification rate on LLM self-consistency outputs (Case B, target = runner-up), averaged across problems. $\bar{K}_N$: Monte Carlo mean number of observed answer categories at budget $N$. $\bar{\tau}$ columns report mean stopping time among certified replicates.}
\label{tab:llm-main-case-b}
\centering\small
\resizebox{\textwidth}{!}{%
\begin{tabular}{ll c ccccc c ccc}
\toprule
Model & Dataset & $N$ & Bonferroni & CITE & W-CITE & KR & MMC & $\bar{K}_N$ & $\bar{\tau}_{\mathrm{CITE}}$ & $\bar{\tau}_{\mathrm{W\mbox{-}CITE}}$ & $\bar{\tau}_{\mathrm{MMC}}$ \\
\midrule
\multirow{15}{*}{Qwen3-30B} & \multirow{5}{*}{AIME 2026} & 64 & 0.000$\pm$0.000 & 0.000$\pm$0.000 & 0.000$\pm$0.000 & 0.001$\pm$0.000 & 0.000$\pm$0.000 & 7.1$\pm$0.0 & --- & --- & 17.2$\pm$0.7 \\
  & & 128 & 0.000$\pm$0.000 & 0.000$\pm$0.000 & 0.000$\pm$0.000 & 0.000$\pm$0.000 & 0.000$\pm$0.000 & 10.2$\pm$0.0 & --- & --- & 17.2$\pm$0.7 \\
  & & 256 & 0.000$\pm$0.000 & 0.000$\pm$0.000 & 0.000$\pm$0.000 & 0.000$\pm$0.000 & 0.000$\pm$0.000 & 14.5$\pm$0.0 & --- & --- & 17.2$\pm$0.7 \\
  & & 512 & 0.000$\pm$0.000 & 0.000$\pm$0.000 & 0.000$\pm$0.000 & 0.000$\pm$0.000 & 0.000$\pm$0.000 & 19.9$\pm$0.0 & --- & --- & 17.2$\pm$0.7 \\
  & & 1024 & 0.000$\pm$0.000 & 0.000$\pm$0.000 & 0.000$\pm$0.000 & 0.000$\pm$0.000 & 0.000$\pm$0.000 & 25.7$\pm$0.0 & --- & --- & 17.2$\pm$0.7 \\
\cmidrule{2-12}
  & \multirow{5}{*}{\shortstack{FrontierScience-\\Olympiad}} & 64 & 0.000$\pm$0.000 & 0.000$\pm$0.000 & 0.000$\pm$0.000 & 0.001$\pm$0.000 & 0.000$\pm$0.000 & 28.0$\pm$0.0 & --- & --- & 26.0$\pm$0.0 \\
  & & 128 & 0.000$\pm$0.000 & 0.000$\pm$0.000 & 0.000$\pm$0.000 & 0.000$\pm$0.000 & 0.000$\pm$0.000 & 46.0$\pm$0.0 & --- & --- & 26.0$\pm$0.0 \\
  & & 256 & 0.000$\pm$0.000 & 0.000$\pm$0.000 & 0.000$\pm$0.000 & 0.000$\pm$0.000 & 0.000$\pm$0.000 & 73.7$\pm$0.1 & --- & --- & 26.0$\pm$0.0 \\
  & & 512 & 0.000$\pm$0.000 & 0.000$\pm$0.000 & 0.000$\pm$0.000 & 0.000$\pm$0.000 & 0.000$\pm$0.000 & 112.9$\pm$0.1 & --- & --- & 26.0$\pm$0.0 \\
  & & 1024 & 0.000$\pm$0.000 & 0.000$\pm$0.000 & 0.000$\pm$0.000 & 0.001$\pm$0.000 & 0.000$\pm$0.000 & 160.0$\pm$0.1 & --- & --- & 26.0$\pm$0.0 \\
\cmidrule{2-12}
  & \multirow{5}{*}{ER-REASON} & 64 & 0.000$\pm$0.000 & 0.000$\pm$0.000 & 0.000$\pm$0.000 & 0.000$\pm$0.000 & 0.000$\pm$0.000 & 1.1$\pm$0.0 & --- & --- & 7.0$\pm$0.0 \\
  & & 128 & 0.000$\pm$0.000 & 0.000$\pm$0.000 & 0.000$\pm$0.000 & 0.000$\pm$0.000 & 0.000$\pm$0.000 & 1.1$\pm$0.0 & --- & --- & 7.0$\pm$0.0 \\
  & & 256 & 0.000$\pm$0.000 & 0.000$\pm$0.000 & 0.000$\pm$0.000 & 0.000$\pm$0.000 & 0.000$\pm$0.000 & 1.1$\pm$0.0 & --- & --- & 7.0$\pm$0.0 \\
  & & 512 & 0.000$\pm$0.000 & 0.000$\pm$0.000 & 0.000$\pm$0.000 & 0.000$\pm$0.000 & 0.000$\pm$0.000 & 1.1$\pm$0.0 & --- & --- & 7.0$\pm$0.0 \\
  & & 1024 & 0.000$\pm$0.000 & 0.000$\pm$0.000 & 0.000$\pm$0.000 & 0.000$\pm$0.000 & 0.000$\pm$0.000 & 1.1$\pm$0.0 & --- & --- & 7.0$\pm$0.0 \\
\midrule
\multirow{15}{*}{gpt-oss-20b} & \multirow{5}{*}{AIME 2026} & 64 & 0.000$\pm$0.000 & 0.000$\pm$0.000 & 0.000$\pm$0.000 & 0.002$\pm$0.000 & 0.000$\pm$0.000 & 12.6$\pm$0.0 & 9.0$\pm$0.0 & 9.0$\pm$0.0 & 26.0$\pm$7.8 \\
  & & 128 & 0.000$\pm$0.000 & 0.000$\pm$0.000 & 0.000$\pm$0.000 & 0.002$\pm$0.000 & 0.000$\pm$0.000 & 18.3$\pm$0.0 & 40.0$\pm$31.0 & 40.0$\pm$31.0 & 26.0$\pm$7.8 \\
  & & 256 & 0.000$\pm$0.000 & 0.000$\pm$0.000 & 0.000$\pm$0.000 & 0.002$\pm$0.000 & 0.000$\pm$0.000 & 25.7$\pm$0.0 & 40.0$\pm$31.0 & 40.0$\pm$31.0 & 26.0$\pm$7.8 \\
  & & 512 & 0.000$\pm$0.000 & 0.000$\pm$0.000 & 0.000$\pm$0.000 & 0.001$\pm$0.000 & 0.000$\pm$0.000 & 34.8$\pm$0.0 & 40.0$\pm$31.0 & 40.0$\pm$31.0 & 26.0$\pm$7.8 \\
  & & 1024 & 0.000$\pm$0.000 & 0.000$\pm$0.000 & 0.000$\pm$0.000 & 0.001$\pm$0.000 & 0.000$\pm$0.000 & 44.0$\pm$0.0 & 354.5$\pm$15.5 & 485.5$\pm$13.4 & 26.0$\pm$7.8 \\
\cmidrule{2-12}
  & \multirow{5}{*}{\shortstack{FrontierScience-\\Olympiad}} & 64 & 0.000$\pm$0.000 & 0.000$\pm$0.000 & 0.000$\pm$0.000 & 0.000$\pm$0.000 & 0.000$\pm$0.000 & 37.8$\pm$0.0 & --- & --- & --- \\
  & & 128 & 0.000$\pm$0.000 & 0.000$\pm$0.000 & 0.000$\pm$0.000 & 0.001$\pm$0.000 & 0.000$\pm$0.000 & 65.1$\pm$0.1 & --- & --- & --- \\
  & & 256 & 0.000$\pm$0.000 & 0.000$\pm$0.000 & 0.000$\pm$0.000 & 0.000$\pm$0.000 & 0.000$\pm$0.000 & 108.5$\pm$0.1 & --- & --- & --- \\
  & & 512 & 0.000$\pm$0.000 & 0.000$\pm$0.000 & 0.000$\pm$0.000 & 0.001$\pm$0.000 & 0.000$\pm$0.000 & 172.0$\pm$0.1 & --- & --- & --- \\
  & & 1024 & 0.000$\pm$0.000 & 0.000$\pm$0.000 & 0.000$\pm$0.000 & 0.000$\pm$0.000 & 0.000$\pm$0.000 & 250.3$\pm$0.1 & --- & --- & --- \\
\cmidrule{2-12}
  & \multirow{5}{*}{ER-REASON} & 64 & 0.000$\pm$0.000 & 0.000$\pm$0.000 & 0.000$\pm$0.000 & 0.002$\pm$0.000 & 0.001$\pm$0.000 & 3.2$\pm$0.0 & 24.6$\pm$5.3 & 30.0$\pm$9.3 & 16.0$\pm$1.6 \\
  & & 128 & 0.000$\pm$0.000 & 0.000$\pm$0.000 & 0.000$\pm$0.000 & 0.002$\pm$0.000 & 0.001$\pm$0.000 & 3.6$\pm$0.0 & 28.1$\pm$6.3 & 36.8$\pm$11.5 & 16.0$\pm$1.6 \\
  & & 256 & 0.000$\pm$0.000 & 0.000$\pm$0.000 & 0.000$\pm$0.000 & 0.001$\pm$0.000 & 0.001$\pm$0.000 & 3.8$\pm$0.0 & 28.1$\pm$6.3 & 36.8$\pm$11.5 & 21.1$\pm$3.2 \\
  & & 512 & 0.000$\pm$0.000 & 0.000$\pm$0.000 & 0.000$\pm$0.000 & 0.001$\pm$0.000 & 0.001$\pm$0.000 & 4.0$\pm$0.0 & 28.1$\pm$6.3 & 36.8$\pm$11.5 & 21.1$\pm$3.2 \\
  & & 1024 & 0.000$\pm$0.000 & 0.000$\pm$0.000 & 0.000$\pm$0.000 & 0.001$\pm$0.000 & 0.001$\pm$0.000 & 4.2$\pm$0.0 & 123.4$\pm$44.4 & 36.8$\pm$11.5 & 27.1$\pm$6.6 \\
\bottomrule
\end{tabular}%
}
\end{table}

\subsection{Model settings}
\label{app:llm-model-settings}

\paragraph{Sampling parameters.}
Qwen3-30B-A3B-Instruct-2507: \texttt{temperature=0.7}, \texttt{top\_p=0.8}, \texttt{top\_k=20}. gpt-oss-20b: \texttt{temperature=0.4}, \texttt{top\_p=0.9}, \texttt{top\_k=20}, \texttt{reasoning\_effort=low}.

\paragraph{Response probability weight for W-CITE.}
For each sample $i$, let $(y_1, \ldots, y_{T_i})$ denote the tokens of the extracted answer span (the final \verb|\boxed{...}| contents for AIME and FrontierScience, the \texttt{Triage:~$\langle$acuity$\rangle \mid$ Disposition:~$\langle$disposition$\rangle$} string for ER-REASON). The W-CITE weight is the response probability used by the weighted variant of self-consistency \citep{wang2023selfconsistency},
\[
  w_i  =  \prod_{t=1}^{T_i} p_\theta \left(y_t \mid y_{<t},  \text{prompt}\right),
\]
i.e., the model's own likelihood of the extracted answer under token-by-token factorization. Weights are not renormalized across samples.

\subsection{Dataset protocols}
\label{app:llm-datasets}

\paragraph{AIME 2026.}
All 30 problems of AIME 2026 \citep{dekoninck2026matharena} (AIME~I and AIME~II combined), each with an integer gold answer in $[0, 999]$. Prompt is the problem statement with a per-model suffix: ``Solve in under 100 tokens. Put your final answer within \verb|\boxed{}|.'' (Qwen3-30B-A3B-Instruct-2507) or ``Do not show your reasoning. Put only the final answer within \verb|\boxed{}|.'' (gpt-oss-20b). Answers are extracted from the last \verb|\boxed{...}| and matched by exact integer equality.

\paragraph{FrontierScience-Olympiad.}
FrontierScience-Olympiad \citep{openai2025frontierscience} contains 100 olympiad-level physics, chemistry, and biology problems with numeric or symbolic answers. We use a 10-problem subset obtained by taking every 10th problem (problems 10, 20, \ldots, 100). Prompt suffix is the same as for AIME 2026. Each extracted answer is compared against the gold by an LLM self-judge (the generating model decides whether its own output is equivalent to the gold).

\paragraph{ER-REASON.}
ER-REASON \citep{mehandru2025erreason} contains 3{,}984 emergency-room encounters. We use the 72-encounter subset with expert-authored clinical rationales, which the paper designates as the primary evaluation set. Both models use a two-turn chat prompt in place of the \verb|\boxed{}| suffix. Turn 1 asks the model to assign an ESI acuity level (1--5) from the chief complaint, demographics, and triage vital signs. Turn 2, which echoes the Turn 1 acuity, asks for the disposition (one of 12 labels) given the full ED provider note and all historical notes. The final answer is parsed as the fixed string ``Triage:~$\langle$acuity$\rangle \mid$ Disposition:~$\langle$disposition$\rangle$''.

\subsection{Per-problem certification rates: Qwen3-30B-A3B-Instruct-2507}
\label{app:llm-qwen3}

\paragraph{AIME 2026.}
% Auto-generated by run.py
% Model: Qwen3-30B-A3B-Instruct-2507 / Dataset: AIME 2026 (30 problems)  (30 problems, Case A, eps=0.05)
% Requires \usepackage{longtable, booktabs}
{\tiny
\fontsize{5.2pt}{6.0pt}\selectfont
\setlength{\tabcolsep}{1.2pt}
\renewcommand{\arraystretch}{0.92}
\begin{longtable}{llccccc c ccc}
\caption{Per-problem certification rate on AIME 2026 (30 problems) (Case A, target = mode). \textbf{Bold}: CITE $>$ Bonferroni. $\bar{K}_N$ is the Monte Carlo mean number of observed answer categories at budget $N$.}
\label{tab:llm-perprob-qwen3-30b-a3b-instruct-2507-aime-2026--30-problems} \\
\toprule
Problem & $N$ & Bonf. & CITE & W-CITE & KR & MMC & $\bar{K}_N$ & $\bar{\tau}_{\mathrm{CITE}}$ & $\bar{\tau}_{\mathrm{W\mbox{-}CITE}}$ & $\bar{\tau}_{\mathrm{MMC}}$ \\
\midrule
\endfirsthead
\multicolumn{11}{c}{\tablename\ \thetable\ -- continued}\\
\toprule
Problem & $N$ & Bonf. & CITE & W-CITE & KR & MMC & $\bar{K}_N$ & $\bar{\tau}_{\mathrm{CITE}}$ & $\bar{\tau}_{\mathrm{W\mbox{-}CITE}}$ & $\bar{\tau}_{\mathrm{MMC}}$ \\
\midrule
\endhead
\bottomrule
\endlastfoot
  P6 ($\delta$=1.000) & 64 & 0.000$\pm$0.000 & \textbf{1.000$\pm$0.000} & 1.000$\pm$0.000 & 1.000$\pm$0.000 & 1.000$\pm$0.000 & 1.0$\pm$0.0 & 9.0$\pm$0.0 & 10.0$\pm$0.0 & 7.0$\pm$0.0 \\
  gold=441, mode=441~\cmark & 128 & 0.000$\pm$0.000 & \textbf{1.000$\pm$0.000} & 1.000$\pm$0.000 & 1.000$\pm$0.000 & 1.000$\pm$0.000 & 1.0$\pm$0.0 & 9.0$\pm$0.0 & 10.0$\pm$0.0 & 7.0$\pm$0.0 \\
    & 256 & 0.000$\pm$0.000 & \textbf{1.000$\pm$0.000} & 1.000$\pm$0.000 & 1.000$\pm$0.000 & 1.000$\pm$0.000 & 1.0$\pm$0.0 & 9.0$\pm$0.0 & 10.0$\pm$0.0 & 7.0$\pm$0.0 \\
    & 512 & 0.000$\pm$0.000 & \textbf{1.000$\pm$0.000} & 1.000$\pm$0.000 & 1.000$\pm$0.000 & 1.000$\pm$0.000 & 1.0$\pm$0.0 & 9.0$\pm$0.0 & 10.0$\pm$0.0 & 7.0$\pm$0.0 \\
    & 1024 & 0.000$\pm$0.000 & \textbf{1.000$\pm$0.000} & 1.000$\pm$0.000 & 1.000$\pm$0.000 & 1.000$\pm$0.000 & 1.0$\pm$0.0 & 9.0$\pm$0.0 & 10.0$\pm$0.0 & 7.0$\pm$0.0 \\
  \addlinespace
  P20 ($\delta$=1.000) & 64 & 0.000$\pm$0.000 & \textbf{1.000$\pm$0.000} & 1.000$\pm$0.000 & 1.000$\pm$0.000 & 1.000$\pm$0.000 & 1.0$\pm$0.0 & 9.0$\pm$0.0 & 10.0$\pm$0.0 & 7.0$\pm$0.0 \\
  gold=190, mode=190~\cmark & 128 & 0.000$\pm$0.000 & \textbf{1.000$\pm$0.000} & 1.000$\pm$0.000 & 1.000$\pm$0.000 & 1.000$\pm$0.000 & 1.0$\pm$0.0 & 9.0$\pm$0.0 & 10.0$\pm$0.0 & 7.0$\pm$0.0 \\
    & 256 & 0.000$\pm$0.000 & \textbf{1.000$\pm$0.000} & 1.000$\pm$0.000 & 1.000$\pm$0.000 & 1.000$\pm$0.000 & 1.0$\pm$0.0 & 9.0$\pm$0.0 & 10.0$\pm$0.0 & 7.0$\pm$0.0 \\
\end{longtable}
}

\paragraph{FrontierScience-Olympiad.}
% Auto-generated by run.py
% Model: Qwen3-30B-A3B-Instruct-2507 / Dataset: FrontierScience-Olympiad (10 problems, every-10 subset)  (10 problems, Case A, eps=0.05)
% Requires \usepackage{longtable, booktabs}
{\tiny
\fontsize{5.2pt}{6.0pt}\selectfont
\setlength{\tabcolsep}{1.2pt}
\renewcommand{\arraystretch}{0.92}
\begin{longtable}{llccccc c ccc}
\caption{Per-problem certification rate on FrontierScience-Olympiad (10 problems, every-10 subset) (Case A, target = mode). \textbf{Bold}: CITE $>$ Bonferroni. $\bar{K}_N$ is the Monte Carlo mean number of observed answer categories at budget $N$.}
\label{tab:llm-perprob-qwen3-30b-a3b-instruct-2507-frontierscience-olympiad--10-problems--every-10-subset} \\
\toprule
Problem & $N$ & Bonf. & CITE & W-CITE & KR & MMC & $\bar{K}_N$ & $\bar{\tau}_{\mathrm{CITE}}$ & $\bar{\tau}_{\mathrm{W\mbox{-}CITE}}$ & $\bar{\tau}_{\mathrm{MMC}}$ \\
\midrule
\endfirsthead
\multicolumn{11}{c}{\tablename\ \thetable\ -- continued}\\
\toprule
Problem & $N$ & Bonf. & CITE & W-CITE & KR & MMC & $\bar{K}_N$ & $\bar{\tau}_{\mathrm{CITE}}$ & $\bar{\tau}_{\mathrm{W\mbox{-}CITE}}$ & $\bar{\tau}_{\mathrm{MMC}}$ \\
\midrule
\endhead
\bottomrule
\endlastfoot
  P60 ($\delta$=0.630) & 64 & 1.000$\pm$0.000 & 0.996$\pm$0.003 & 0.998$\pm$0.002 & 0.994$\pm$0.003 & 1.000$\pm$0.000 & 2.7$\pm$0.0 & 23.5$\pm$0.5 & 23.4$\pm$0.5 & 18.2$\pm$0.5 \\
  gold=6.98, mode=6.49~\xmark & 128 & 1.000$\pm$0.000 & 1.000$\pm$0.000 & 1.000$\pm$0.000 & 1.000$\pm$0.000 & 1.000$\pm$0.000 & 3.1$\pm$0.0 & 23.7$\pm$0.5 & 23.5$\pm$0.5 & 18.2$\pm$0.5 \\
    & 256 & 1.000$\pm$0.000 & 1.000$\pm$0.000 & 1.000$\pm$0.000 & 1.000$\pm$0.000 & 1.000$\pm$0.000 & 3.4$\pm$0.0 & 23.7$\pm$0.5 & 23.5$\pm$0.5 & 18.2$\pm$0.5 \\
    & 512 & 1.000$\pm$0.000 & 1.000$\pm$0.000 & 1.000$\pm$0.000 & 1.000$\pm$0.000 & 1.000$\pm$0.000 & 3.8$\pm$0.0 & 23.7$\pm$0.5 & 23.5$\pm$0.5 & 18.2$\pm$0.5 \\
    & 1024 & 1.000$\pm$0.000 & 1.000$\pm$0.000 & 1.000$\pm$0.000 & 1.000$\pm$0.000 & 1.000$\pm$0.000 & 4.3$\pm$0.0 & 23.7$\pm$0.5 & 23.5$\pm$0.5 & 18.2$\pm$0.5 \\
  \addlinespace
  P50 ($\delta$=0.524) & 64 & 0.994$\pm$0.003 & 0.994$\pm$0.003 & 0.994$\pm$0.003 & 0.992$\pm$0.004 & 0.156$\pm$0.016 & 14.1$\pm$0.1 & 29.4$\pm$0.4 & 30.1$\pm$0.4 & 34.7$\pm$2.1 \\
  gold=0.2, mode=0.2~\cmark & 128 & 1.000$\pm$0.000 & 1.000$\pm$0.000 & 1.000$\pm$0.000 & 1.000$\pm$0.000 & 0.354$\pm$0.021 & 19.2$\pm$0.1 & 29.7$\pm$0.4 & 30.4$\pm$0.4 & 70.3$\pm$2.7 \\
    & 256 & 1.000$\pm$0.000 & 1.000$\pm$0.000 & 1.000$\pm$0.000 & 1.000$\pm$0.000 & 0.756$\pm$0.019 & 24.6$\pm$0.1 & 29.7$\pm$0.4 & 30.4$\pm$0.4 & 133.9$\pm$3.6 \\
\end{longtable}
}

\paragraph{ER-REASON.}
% Auto-generated by run.py
% Model: Qwen3-30B-A3B-Instruct-2507 / Dataset: ER-REASON  (72 problems, Case A, eps=0.05)
% Requires \usepackage{longtable, booktabs}
{\tiny
\fontsize{5.2pt}{6.0pt}\selectfont
\setlength{\tabcolsep}{1.2pt}
\renewcommand{\arraystretch}{0.92}
\begin{longtable}{llccccc c ccc}
\caption{Per-problem certification rate on ER-REASON (Case A, target = mode). \textbf{Bold}: CITE $>$ Bonferroni. $\bar{K}_N$ is the Monte Carlo mean number of observed answer categories at budget $N$.}
\label{tab:llm-perprob-qwen3-30b-a3b-instruct-2507-er-reason} \\
\toprule
Problem & $N$ & Bonf. & CITE & W-CITE & KR & MMC & $\bar{K}_N$ & $\bar{\tau}_{\mathrm{CITE}}$ & $\bar{\tau}_{\mathrm{W\mbox{-}CITE}}$ & $\bar{\tau}_{\mathrm{MMC}}$ \\
\midrule
\endfirsthead
\multicolumn{11}{c}{\tablename\ \thetable\ -- continued}\\
\toprule
Problem & $N$ & Bonf. & CITE & W-CITE & KR & MMC & $\bar{K}_N$ & $\bar{\tau}_{\mathrm{CITE}}$ & $\bar{\tau}_{\mathrm{W\mbox{-}CITE}}$ & $\bar{\tau}_{\mathrm{MMC}}$ \\
\midrule
\endhead
\bottomrule
\endlastfoot
  P1 ($\delta$=1.000) & 64 & 0.000$\pm$0.000 & \textbf{1.000$\pm$0.000} & 1.000$\pm$0.000 & 1.000$\pm$0.000 & 1.000$\pm$0.000 & 1.0$\pm$0.0 & 9.0$\pm$0.0 & 10.0$\pm$0.0 & 7.0$\pm$0.0 \\
  \cmark & 128 & 0.000$\pm$0.000 & \textbf{1.000$\pm$0.000} & 1.000$\pm$0.000 & 1.000$\pm$0.000 & 1.000$\pm$0.000 & 1.0$\pm$0.0 & 9.0$\pm$0.0 & 10.0$\pm$0.0 & 7.0$\pm$0.0 \\
    & 256 & 0.000$\pm$0.000 & \textbf{1.000$\pm$0.000} & 1.000$\pm$0.000 & 1.000$\pm$0.000 & 1.000$\pm$0.000 & 1.0$\pm$0.0 & 9.0$\pm$0.0 & 10.0$\pm$0.0 & 7.0$\pm$0.0 \\
    & 512 & 0.000$\pm$0.000 & \textbf{1.000$\pm$0.000} & 1.000$\pm$0.000 & 1.000$\pm$0.000 & 1.000$\pm$0.000 & 1.0$\pm$0.0 & 9.0$\pm$0.0 & 10.0$\pm$0.0 & 7.0$\pm$0.0 \\
    & 1024 & 0.000$\pm$0.000 & \textbf{1.000$\pm$0.000} & 1.000$\pm$0.000 & 1.000$\pm$0.000 & 1.000$\pm$0.000 & 1.0$\pm$0.0 & 9.0$\pm$0.0 & 10.0$\pm$0.0 & 7.0$\pm$0.0 \\
  \addlinespace
  P2 ($\delta$=1.000) & 64 & 0.000$\pm$0.000 & \textbf{1.000$\pm$0.000} & 1.000$\pm$0.000 & 1.000$\pm$0.000 & 1.000$\pm$0.000 & 1.0$\pm$0.0 & 9.0$\pm$0.0 & 10.0$\pm$0.0 & 7.0$\pm$0.0 \\
  \xmark & 128 & 0.000$\pm$0.000 & \textbf{1.000$\pm$0.000} & 1.000$\pm$0.000 & 1.000$\pm$0.000 & 1.000$\pm$0.000 & 1.0$\pm$0.0 & 9.0$\pm$0.0 & 10.0$\pm$0.0 & 7.0$\pm$0.0 \\
    & 256 & 0.000$\pm$0.000 & \textbf{1.000$\pm$0.000} & 1.000$\pm$0.000 & 1.000$\pm$0.000 & 1.000$\pm$0.000 & 1.0$\pm$0.0 & 9.0$\pm$0.0 & 10.0$\pm$0.0 & 7.0$\pm$0.0 \\
\end{longtable}
}

\subsection{Per-problem certification rates: gpt-oss-20b}
\label{app:llm-gptoss}

\paragraph{AIME 2026.}
% Auto-generated by run.py
% Model: gpt-oss-20b / Dataset: AIME 2026 (30 problems)  (30 problems, Case A, eps=0.05)
% Requires \usepackage{longtable, booktabs}
{\tiny
\fontsize{5.2pt}{6.0pt}\selectfont
\setlength{\tabcolsep}{1.2pt}
\renewcommand{\arraystretch}{0.92}
\begin{longtable}{llccccc c ccc}
\caption{Per-problem certification rate on AIME 2026 (30 problems) (Case A, target = mode). \textbf{Bold}: CITE $>$ Bonferroni. $\bar{K}_N$ is the Monte Carlo mean number of observed answer categories at budget $N$.}
\label{tab:llm-perprob-gpt-oss-20b-aime-2026--30-problems} \\
\toprule
Problem & $N$ & Bonf. & CITE & W-CITE & KR & MMC & $\bar{K}_N$ & $\bar{\tau}_{\mathrm{CITE}}$ & $\bar{\tau}_{\mathrm{W\mbox{-}CITE}}$ & $\bar{\tau}_{\mathrm{MMC}}$ \\
\midrule
\endfirsthead
\multicolumn{11}{c}{\tablename\ \thetable\ -- continued}\\
\toprule
Problem & $N$ & Bonf. & CITE & W-CITE & KR & MMC & $\bar{K}_N$ & $\bar{\tau}_{\mathrm{CITE}}$ & $\bar{\tau}_{\mathrm{W\mbox{-}CITE}}$ & $\bar{\tau}_{\mathrm{MMC}}$ \\
\midrule
\endhead
\bottomrule
\endlastfoot
  P6 ($\delta$=0.995) & 64 & 0.222$\pm$0.019 & \textbf{1.000$\pm$0.000} & 1.000$\pm$0.000 & 1.000$\pm$0.000 & 1.000$\pm$0.000 & 1.2$\pm$0.0 & 9.2$\pm$0.0 & 9.2$\pm$0.0 & 7.2$\pm$0.0 \\
  gold=441, mode=441~\cmark & 128 & 0.348$\pm$0.021 & \textbf{1.000$\pm$0.000} & 1.000$\pm$0.000 & 1.000$\pm$0.000 & 1.000$\pm$0.000 & 1.4$\pm$0.0 & 9.2$\pm$0.0 & 9.2$\pm$0.0 & 7.2$\pm$0.0 \\
    & 256 & 0.548$\pm$0.022 & \textbf{1.000$\pm$0.000} & 1.000$\pm$0.000 & 1.000$\pm$0.000 & 1.000$\pm$0.000 & 1.6$\pm$0.0 & 9.2$\pm$0.0 & 9.2$\pm$0.0 & 7.2$\pm$0.0 \\
    & 512 & 0.792$\pm$0.018 & \textbf{1.000$\pm$0.000} & 1.000$\pm$0.000 & 1.000$\pm$0.000 & 1.000$\pm$0.000 & 2.1$\pm$0.0 & 9.2$\pm$0.0 & 9.2$\pm$0.0 & 7.2$\pm$0.0 \\
    & 1024 & 0.968$\pm$0.008 & \textbf{1.000$\pm$0.000} & 1.000$\pm$0.000 & 1.000$\pm$0.000 & 1.000$\pm$0.000 & 2.6$\pm$0.0 & 9.2$\pm$0.0 & 9.2$\pm$0.0 & 7.2$\pm$0.0 \\
  \addlinespace
  P1 ($\delta$=0.970) & 64 & 0.636$\pm$0.022 & \textbf{1.000$\pm$0.000} & 1.000$\pm$0.000 & 1.000$\pm$0.000 & 1.000$\pm$0.000 & 1.7$\pm$0.0 & 9.6$\pm$0.1 & 9.6$\pm$0.1 & 7.5$\pm$0.1 \\
  gold=277, mode=277~\cmark & 128 & 0.852$\pm$0.016 & \textbf{1.000$\pm$0.000} & 1.000$\pm$0.000 & 1.000$\pm$0.000 & 1.000$\pm$0.000 & 2.1$\pm$0.0 & 9.6$\pm$0.1 & 9.6$\pm$0.1 & 7.5$\pm$0.1 \\
    & 256 & 0.976$\pm$0.007 & \textbf{1.000$\pm$0.000} & 1.000$\pm$0.000 & 1.000$\pm$0.000 & 1.000$\pm$0.000 & 2.4$\pm$0.0 & 9.6$\pm$0.1 & 9.6$\pm$0.1 & 7.5$\pm$0.1 \\
\end{longtable}
}

\paragraph{FrontierScience-Olympiad.}
% Auto-generated by run.py
% Model: gpt-oss-20b / Dataset: FrontierScience-Olympiad (10 problems)  (10 problems, Case A, eps=0.05)
% Requires \usepackage{longtable, booktabs}
{\tiny
\fontsize{5.2pt}{6.0pt}\selectfont
\setlength{\tabcolsep}{1.2pt}
\renewcommand{\arraystretch}{0.92}
\begin{longtable}{llccccc c ccc}
\caption{Per-problem certification rate on FrontierScience-Olympiad (10 problems) (Case A, target = mode). \textbf{Bold}: CITE $>$ Bonferroni. $\bar{K}_N$ is the Monte Carlo mean number of observed answer categories at budget $N$.}
\label{tab:llm-perprob-gpt-oss-20b-frontierscience-olympiad--10-problems} \\
\toprule
Problem & $N$ & Bonf. & CITE & W-CITE & KR & MMC & $\bar{K}_N$ & $\bar{\tau}_{\mathrm{CITE}}$ & $\bar{\tau}_{\mathrm{W\mbox{-}CITE}}$ & $\bar{\tau}_{\mathrm{MMC}}$ \\
\midrule
\endfirsthead
\multicolumn{11}{c}{\tablename\ \thetable\ -- continued}\\
\toprule
Problem & $N$ & Bonf. & CITE & W-CITE & KR & MMC & $\bar{K}_N$ & $\bar{\tau}_{\mathrm{CITE}}$ & $\bar{\tau}_{\mathrm{W\mbox{-}CITE}}$ & $\bar{\tau}_{\mathrm{MMC}}$ \\
\midrule
\endhead
\bottomrule
\endlastfoot
  P60 ($\delta$=0.763) & 64 & 1.000$\pm$0.000 & 1.000$\pm$0.000 & 1.000$\pm$0.000 & 1.000$\pm$0.000 & 1.000$\pm$0.000 & 5.2$\pm$0.1 & 16.1$\pm$0.3 & 17.0$\pm$0.3 & 14.3$\pm$0.3 \\
  gold=6.98, mode=6.98~\cmark & 128 & 1.000$\pm$0.000 & 1.000$\pm$0.000 & 1.000$\pm$0.000 & 1.000$\pm$0.000 & 1.000$\pm$0.000 & 7.3$\pm$0.1 & 16.1$\pm$0.3 & 17.0$\pm$0.3 & 14.3$\pm$0.3 \\
    & 256 & 1.000$\pm$0.000 & 1.000$\pm$0.000 & 1.000$\pm$0.000 & 1.000$\pm$0.000 & 1.000$\pm$0.000 & 9.8$\pm$0.1 & 16.1$\pm$0.3 & 17.0$\pm$0.3 & 14.3$\pm$0.3 \\
    & 512 & 1.000$\pm$0.000 & 1.000$\pm$0.000 & 1.000$\pm$0.000 & 1.000$\pm$0.000 & 1.000$\pm$0.000 & 12.3$\pm$0.1 & 16.1$\pm$0.3 & 17.0$\pm$0.3 & 14.3$\pm$0.3 \\
    & 1024 & 1.000$\pm$0.000 & 1.000$\pm$0.000 & 1.000$\pm$0.000 & 1.000$\pm$0.000 & 1.000$\pm$0.000 & 14.4$\pm$0.0 & 16.1$\pm$0.3 & 17.0$\pm$0.3 & 14.3$\pm$0.3 \\
  \addlinespace
  P50 ($\delta$=0.474) & 64 & 0.926$\pm$0.012 & 0.862$\pm$0.015 & 0.474$\pm$0.022 & 0.934$\pm$0.011 & 0.916$\pm$0.012 & 6.8$\pm$0.1 & 33.4$\pm$0.7 & 40.3$\pm$0.9 & 30.0$\pm$0.7 \\
  gold=0.2, mode=0.2~\cmark & 128 & 0.998$\pm$0.002 & 0.994$\pm$0.003 & 0.882$\pm$0.014 & 0.996$\pm$0.003 & 0.994$\pm$0.003 & 9.9$\pm$0.1 & 39.6$\pm$0.9 & 63.7$\pm$1.4 & 33.7$\pm$0.9 \\
    & 256 & 1.000$\pm$0.000 & 1.000$\pm$0.000 & 0.998$\pm$0.002 & 1.000$\pm$0.000 & 1.000$\pm$0.000 & 14.1$\pm$0.1 & 40.2$\pm$1.0 & 74.9$\pm$1.9 & 34.3$\pm$0.9 \\
\end{longtable}
}

\paragraph{ER-REASON.}
% Auto-generated by run.py
% Model: gpt-oss-20b / Dataset: ER-REASON (72 problems)  (72 problems, Case A, eps=0.05)
% Requires \usepackage{longtable, booktabs}
{\tiny
\fontsize{5.2pt}{6.0pt}\selectfont
\setlength{\tabcolsep}{1.2pt}
\renewcommand{\arraystretch}{0.92}
\begin{longtable}{llccccc c ccc}
\caption{Per-problem certification rate on ER-REASON (72 problems) (Case A, target = mode). \textbf{Bold}: CITE $>$ Bonferroni. $\bar{K}_N$ is the Monte Carlo mean number of observed answer categories at budget $N$.}
\label{tab:llm-perprob-gpt-oss-20b-er-reason--72-problems} \\
\toprule
Problem & $N$ & Bonf. & CITE & W-CITE & KR & MMC & $\bar{K}_N$ & $\bar{\tau}_{\mathrm{CITE}}$ & $\bar{\tau}_{\mathrm{W\mbox{-}CITE}}$ & $\bar{\tau}_{\mathrm{MMC}}$ \\
\midrule
\endfirsthead
\multicolumn{11}{c}{\tablename\ \thetable\ -- continued}\\
\toprule
Problem & $N$ & Bonf. & CITE & W-CITE & KR & MMC & $\bar{K}_N$ & $\bar{\tau}_{\mathrm{CITE}}$ & $\bar{\tau}_{\mathrm{W\mbox{-}CITE}}$ & $\bar{\tau}_{\mathrm{MMC}}$ \\
\midrule
\endhead
\bottomrule
\endlastfoot
  P8 ($\delta$=1.000) & 64 & 0.000$\pm$0.000 & \textbf{1.000$\pm$0.000} & 1.000$\pm$0.000 & 1.000$\pm$0.000 & 1.000$\pm$0.000 & 1.0$\pm$0.0 & 9.0$\pm$0.0 & 11.1$\pm$0.0 & 7.0$\pm$0.0 \\
  \xmark & 128 & 0.000$\pm$0.000 & \textbf{1.000$\pm$0.000} & 1.000$\pm$0.000 & 1.000$\pm$0.000 & 1.000$\pm$0.000 & 1.0$\pm$0.0 & 9.0$\pm$0.0 & 11.1$\pm$0.0 & 7.0$\pm$0.0 \\
    & 256 & 0.000$\pm$0.000 & \textbf{1.000$\pm$0.000} & 1.000$\pm$0.000 & 1.000$\pm$0.000 & 1.000$\pm$0.000 & 1.0$\pm$0.0 & 9.0$\pm$0.0 & 11.1$\pm$0.0 & 7.0$\pm$0.0 \\
    & 512 & 0.000$\pm$0.000 & \textbf{1.000$\pm$0.000} & 1.000$\pm$0.000 & 1.000$\pm$0.000 & 1.000$\pm$0.000 & 1.0$\pm$0.0 & 9.0$\pm$0.0 & 11.1$\pm$0.0 & 7.0$\pm$0.0 \\
    & 1024 & 0.000$\pm$0.000 & \textbf{1.000$\pm$0.000} & 1.000$\pm$0.000 & 1.000$\pm$0.000 & 1.000$\pm$0.000 & 1.0$\pm$0.0 & 9.0$\pm$0.0 & 11.1$\pm$0.0 & 7.0$\pm$0.0 \\
  \addlinespace
  P47 ($\delta$=0.996) & 64 & 0.104$\pm$0.014 & \textbf{1.000$\pm$0.000} & 1.000$\pm$0.000 & 1.000$\pm$0.000 & 1.000$\pm$0.000 & 1.1$\pm$0.0 & 9.1$\pm$0.0 & 11.1$\pm$0.0 & 7.1$\pm$0.0 \\
  \xmark & 128 & 0.196$\pm$0.018 & \textbf{1.000$\pm$0.000} & 1.000$\pm$0.000 & 1.000$\pm$0.000 & 1.000$\pm$0.000 & 1.2$\pm$0.0 & 9.1$\pm$0.0 & 11.1$\pm$0.0 & 7.1$\pm$0.0 \\
    & 256 & 0.372$\pm$0.022 & \textbf{1.000$\pm$0.000} & 1.000$\pm$0.000 & 1.000$\pm$0.000 & 1.000$\pm$0.000 & 1.4$\pm$0.0 & 9.1$\pm$0.0 & 11.1$\pm$0.0 & 7.1$\pm$0.0 \\
\end{longtable}
}

\section{Reproducibility Statement}
\label{appendix:reprod}

\paragraph{Hardware and inference server.}
Real-LLM experiments are conducted using NVIDIA A100 80~GB GPUs with vLLM as the inference server.

\paragraph{Data and code access.}
All code will be open-sourced upon publication. Datasets and models are publicly available (Table~\ref{tab:links}). ER-REASON requires PhysioNet credentialed access.

\begin{table}[H]
    \centering
    \small
    \begin{tabularx}{\textwidth}{>{\raggedright\arraybackslash}p{0.28\textwidth}X}
    \toprule
        Name & Reference \\
    \midrule
        \multicolumn{2}{l}{\textit{Models}} \\ \addlinespace
        Qwen3-30B-A3B-Instruct-2507 \citep{qwen2025qwen3} & \url{https://huggingface.co/Qwen/Qwen3-30B-A3B-Instruct-2507} \\ \addlinespace
        gpt-oss-20b \citep{openai2025gptoss} & \url{https://huggingface.co/openai/gpt-oss-20b} \\
    \midrule
        \multicolumn{2}{l}{\textit{Datasets}} \\ \addlinespace
        AIME 2026 \citep{dekoninck2026matharena} & \url{https://huggingface.co/datasets/MathArena/aime_2026} \\ \addlinespace
        FrontierScience-Olympiad \citep{openai2025frontierscience} & \url{https://huggingface.co/datasets/openai/frontierscience} \\ \addlinespace
        ER-REASON \citep{mehandru2025erreason} & \url{https://physionet.org/content/er-reason/1.0.0/} \\
    \bottomrule
    \end{tabularx}
    \caption{List of models and datasets used in the LLM self-consistency experiments.}\label{tab:links}
\end{table}


\begin{thebibliography}{43}
\providecommand{\natexlab}[1]{#1}
\providecommand{\url}[1]{\texttt{#1}}
\expandafter\ifx\csname urlstyle\endcsname\relax
  \providecommand{\doi}[1]{doi: #1}\else
  \providecommand{\doi}{doi: \begingroup \urlstyle{rm}\Url}\fi

\bibitem[Achiam et~al.(2023)Achiam, Adler, Agarwal, Ahmad, Akkaya, Aleman,
  Almeida, Altenschmidt, Altman, Anadkat, et~al.]{achiam2023gpt}
Josh Achiam, Steven Adler, Sandhini Agarwal, Lama Ahmad, Ilge Akkaya,
  Florencia~Leoni Aleman, Diogo Almeida, Janko Altenschmidt, Sam Altman,
  Shyamal Anadkat, et~al.
\newblock Gpt-4 technical report.
\newblock \emph{arXiv preprint arXiv:2303.08774}, 2023.

\bibitem[Aggarwal et~al.(2023)Aggarwal, Madaan, Yang, et~al.]{aggarwal2023let}
Pranjal Aggarwal, Aman Madaan, Yiming Yang, et~al.
\newblock Let’s sample step by step: Adaptive-consistency for efficient
  reasoning and coding with llms.
\newblock In \emph{Proceedings of the 2023 Conference on Empirical Methods in
  Natural Language Processing}, pages 12375--12396, 2023.

\bibitem[Aghazadeh et~al.(2025)Aghazadeh, Ghasemi, Beyhaghi, and
  Pishro-Nik]{aghazadeh2025cges}
Ehsan Aghazadeh, Ahmad Ghasemi, Hedyeh Beyhaghi, and Hossein Pishro-Nik.
\newblock Cges: Confidence-guided early stopping for efficient and accurate
  self-consistency.
\newblock \emph{arXiv preprint arXiv:2511.02603}, 2025.

\bibitem[Berend and Kontorovich(2013)]{berendkontorovich2013missingmass}
Daniel Berend and Aryeh Kontorovich.
\newblock On the concentration of the missing mass.
\newblock \emph{Electronic Communications in Probability}, 18\penalty0
  (3):\penalty0 1--7, 2013.
\newblock \doi{10.1214/ECP.v18-2359}.

\bibitem[Brown et~al.(2024)Brown, Juravsky, Ehrlich, Clark, Le, R{\'e}, and
  Mirhoseini]{brown2025monkeys}
Bradley Brown, Jordan Juravsky, Ryan Ehrlich, Ronald Clark, Quoc~V Le,
  Christopher R{\'e}, and Azalia Mirhoseini.
\newblock Large language monkeys: Scaling inference compute with repeated
  sampling.
\newblock \emph{arXiv preprint arXiv:2407.21787}, 2024.

\bibitem[Brown et~al.(2020)Brown, Mann, Ryder, Subbiah, Kaplan, Dhariwal,
  Neelakantan, Shyam, Sastry, Askell, et~al.]{brown2020language}
Tom Brown, Benjamin Mann, Nick Ryder, Melanie Subbiah, Jared~D Kaplan, Prafulla
  Dhariwal, Arvind Neelakantan, Pranav Shyam, Girish Sastry, Amanda Askell,
  et~al.
\newblock Language models are few-shot learners.
\newblock \emph{Advances in neural information processing systems},
  33:\penalty0 1877--1901, 2020.

\bibitem[Chowdhery et~al.(2023)Chowdhery, Narang, Devlin, Bosma, Mishra,
  Roberts, Barham, Chung, Sutton, Gehrmann, et~al.]{chowdhery2023palm}
Aakanksha Chowdhery, Sharan Narang, Jacob Devlin, Maarten Bosma, Gaurav Mishra,
  Adam Roberts, Paul Barham, Hyung~Won Chung, Charles Sutton, Sebastian
  Gehrmann, et~al.
\newblock Palm: Scaling language modeling with pathways.
\newblock \emph{Journal of machine learning research}, 24\penalty0
  (240):\penalty0 1--113, 2023.

\bibitem[Cordero-Encinar and Duncan(2025)]{cordero2025certified}
Paula Cordero-Encinar and Andrew~B Duncan.
\newblock Certified self-consistency: Statistical guarantees and test-time
  training for reliable reasoning in llms.
\newblock \emph{arXiv preprint arXiv:2510.17472}, 2025.

\bibitem[Dekoninck et~al.(2026)Dekoninck, Jovanović, Gehrunger, Rögnvalddson,
  Petrov, Sun, and Vechev]{dekoninck2026matharena}
Jasper Dekoninck, Nikola Jovanović, Tim Gehrunger, Kári Rögnvalddson, Ivo
  Petrov, Chenhao Sun, and Martin Vechev.
\newblock Beyond {B}enchmarks: {M}ath{A}rena as an {E}valuation {P}latform for
  {M}athematics with {LLM}s.
\newblock 2026.
\newblock URL \url{https://arxiv.org/abs/2605.00674}.

\bibitem[Fu et~al.(2025)Fu, Wang, Tian, and Zhao]{fu2025deep}
Yichao Fu, Xuewei Wang, Yuandong Tian, and Jiawei Zhao.
\newblock Deep think with confidence.
\newblock \emph{arXiv preprint arXiv:2508.15260}, 2025.

\bibitem[Garivier and Kaufmann(2016)]{garivier2016optimal}
Aur{\'e}lien Garivier and Emilie Kaufmann.
\newblock Optimal best arm identification with fixed confidence.
\newblock In \emph{Conference on Learning Theory}, pages 998--1027. PMLR, 2016.

\bibitem[Good(1953)]{good1953population}
I.~J. Good.
\newblock The population frequencies of species and the estimation of
  population parameters.
\newblock \emph{Biometrika}, 40\penalty0 (3-4):\penalty0 237--264, 1953.
\newblock \doi{10.1093/biomet/40.3-4.237}.

\bibitem[Good and Toulmin(1956)]{goodtoulmin1956new}
I.~J. Good and G.~H. Toulmin.
\newblock The number of new species, and the increase in population coverage,
  when a sample is increased.
\newblock \emph{Biometrika}, 43\penalty0 (1-2):\penalty0 45--63, 1956.
\newblock \doi{10.1093/biomet/43.1-2.45}.

\bibitem[Gr{\"u}nwald et~al.(2024)Gr{\"u}nwald, de~Heide, and
  Koolen]{grunwald2024safe}
Peter Gr{\"u}nwald, Rianne de~Heide, and Wouter~M. Koolen.
\newblock Safe testing.
\newblock \emph{Journal of the Royal Statistical Society Series B: Statistical
  Methodology}, 86\penalty0 (5):\penalty0 1091--1128, 2024.

\bibitem[Howard et~al.(2020)Howard, Ramdas, McAuliffe, and
  Sekhon]{howard2020chernoff}
Steven~R Howard, Aaditya Ramdas, Jon McAuliffe, and Jasjeet Sekhon.
\newblock Time-uniform chernoff bounds via nonnegative supermartingales.
\newblock \emph{Probability Surveys}, 17, 2020.

\bibitem[Howard et~al.(2021)Howard, Ramdas, McAuliffe, and
  Sekhon]{howard2021confidence}
Steven~R. Howard, Aaditya Ramdas, Jon McAuliffe, and Jasjeet Sekhon.
\newblock Time-uniform, nonparametric, nonasymptotic confidence sequences.
\newblock \emph{The Annals of Statistics}, 49\penalty0 (2):\penalty0
  1055--1080, 2021.
\newblock \doi{10.1214/20-AOS1991}.

\bibitem[Kaufmann and Koolen(2021)]{kaufmann2021mixture}
Emilie Kaufmann and Wouter~M Koolen.
\newblock Mixture martingales revisited with applications to sequential tests
  and confidence intervals.
\newblock \emph{Journal of Machine Learning Research}, 22\penalty0
  (246):\penalty0 1--44, 2021.

\bibitem[Kaufmann et~al.(2016)Kaufmann, Capp{\'e}, and
  Garivier]{kaufmann2016complexity}
Emilie Kaufmann, Olivier Capp{\'e}, and Aur{\'e}lien Garivier.
\newblock On the complexity of best-arm identification in multi-armed bandit
  models.
\newblock \emph{The Journal of Machine Learning Research}, 17\penalty0
  (1):\penalty0 1--42, 2016.

\bibitem[Kim and Ramdas(2025)]{kim2025locally}
Ilmun Kim and Aaditya Ramdas.
\newblock Locally minimax optimal confidence sets for the best model.
\newblock \emph{arXiv preprint arXiv:2503.21639}, 2025.

\bibitem[Komiyama et~al.(2026)Komiyama, Oba, and
  Oyamada]{komiyama2026bestofinfinity}
Junpei Komiyama, Daisuke Oba, and Masafumi Oyamada.
\newblock Best-of-infinity: Asymptotic performance of test-time llm ensembling.
\newblock In \emph{The Fourteenth International Conference on Learning
  Representations}, 2026.

\bibitem[Li et~al.(2024)Li, Yuan, Feng, Pan, Wang, Sun, Wang, and
  Li]{li2024escape}
Yiwei Li, Peiwen Yuan, Shaoxiong Feng, Boyuan Pan, Xinglin Wang, Bin Sun, Heda
  Wang, and Kan Li.
\newblock Escape sky-high cost: Early-stopping self-consistency for multi-step
  reasoning.
\newblock In \emph{The Twelfth International Conference on Learning
  Representations}, 2024.

\bibitem[Lindon and Malek(2022)]{lindon2022anytime}
Michael Lindon and Alan Malek.
\newblock Anytime-valid inference for multinomial count data.
\newblock \emph{Advances in Neural Information Processing Systems},
  35:\penalty0 2817--2831, 2022.

\bibitem[McAllester and Ortiz(2003)]{mcallesterortiz2003missingmass}
David McAllester and Luis~E. Ortiz.
\newblock Concentration inequalities for the missing mass and for histogram
  rule error.
\newblock \emph{Journal of Machine Learning Research}, 4:\penalty0 895--911,
  2003.

\bibitem[Mehandru et~al.(2025)Mehandru, Golchini, Bamman, Zack, Molina, and
  Alaa]{mehandru2025erreason}
Nikita Mehandru, Niloufar Golchini, David Bamman, Travis Zack, Melanie~F.
  Molina, and Ahmed Alaa.
\newblock {ER-REASON}: {A} {B}enchmark {D}ataset for {LLM}-{B}ased {C}linical
  {R}easoning in the {E}mergency {R}oom.
\newblock \emph{arXiv preprint arXiv:2505.22919}, 2025.
\newblock URL \url{https://arxiv.org/abs/2505.22919}.

\bibitem[{OpenAI}(2025)]{openai2025gptoss}
{OpenAI}.
\newblock gpt-oss-120b \& gpt-oss-20b {M}odel {C}ard.
\newblock \emph{arXiv preprint arXiv:2508.10925}, 2025.
\newblock URL \url{https://arxiv.org/abs/2508.10925}.

\bibitem[Orlitsky et~al.(2016)Orlitsky, Suresh, and Wu]{orlitsky2016unseen}
Alon Orlitsky, Ananda~Theertha Suresh, and Yihong Wu.
\newblock Optimal prediction of the number of unseen species.
\newblock \emph{Proceedings of the National Academy of Sciences}, 113\penalty0
  (47):\penalty0 13283--13288, 2016.
\newblock \doi{10.1073/pnas.1607774113}.

\bibitem[Painsky(2025)]{Painsky02012025}
Amichai Painsky.
\newblock Confidence intervals for parameters of unobserved events.
\newblock \emph{Journal of the American Statistical Association}, 120\penalty0
  (549):\penalty0 226--236, 2025.

\bibitem[{Qwen Team}(2025)]{qwen2025qwen3}
{Qwen Team}.
\newblock Qwen3 {T}echnical {R}eport.
\newblock \emph{arXiv preprint arXiv:2505.09388}, 2025.
\newblock URL \url{https://arxiv.org/abs/2505.09388}.

\bibitem[Ramdas and Wang(2025)]{ramdas2025hypothesis}
Aaditya Ramdas and Ruodu Wang.
\newblock Hypothesis testing with e-values.
\newblock \emph{Foundations and Trends{\textregistered} in Statistics},
  1\penalty0 (1-2):\penalty0 1--390, 2025.

\bibitem[Ramdas et~al.(2023)Ramdas, Gr{\"u}nwald, Vovk, and
  Shafer]{ramdas2023gametheoretic}
Aaditya Ramdas, Peter Gr{\"u}nwald, Vladimir Vovk, and Glenn Shafer.
\newblock Game-theoretic statistics and safe anytime-valid inference.
\newblock \emph{Statistical Science}, 38\penalty0 (4):\penalty0 576--601, 2023.

\bibitem[Shafer(2021)]{shafer2021betting}
Glenn Shafer.
\newblock Testing by betting: A strategy for statistical and scientific
  communication (with discussion).
\newblock \emph{Journal of the Royal Statistical Society: Series A (Statistics
  in Society)}, 184\penalty0 (2):\penalty0 407--431, 2021.
\newblock \doi{10.1111/rssa.12647}.

\bibitem[Snell et~al.(2025)Snell, Lee, Xu, and Kumar]{snell2025scaling}
Charlie~Victor Snell, Jaehoon Lee, Kelvin Xu, and Aviral Kumar.
\newblock Scaling llm test-time compute optimally can be more effective than
  scaling parameters for reasoning.
\newblock In \emph{The Thirteenth International Conference on Learning
  Representations}, 2025.

\bibitem[Taubenfeld et~al.(2025)Taubenfeld, Sheffer, Ofek, Feder, Goldstein,
  Gekhman, and Yona]{taubenfeld2025confidence}
Amir Taubenfeld, Tom Sheffer, Eran Ofek, Amir Feder, Ariel Goldstein, Zorik
  Gekhman, and Gal Yona.
\newblock Confidence improves self-consistency in llms.
\newblock In \emph{Findings of the Association for Computational Linguistics:
  ACL 2025}, pages 20090--20111, 2025.

\bibitem[Touvron et~al.(2023)Touvron, Lavril, Izacard, Martinet, Lachaux,
  Lacroix, Rozi{\`e}re, Goyal, Hambro, Azhar, et~al.]{touvron2023llama}
Hugo Touvron, Thibaut Lavril, Gautier Izacard, Xavier Martinet, Marie-Anne
  Lachaux, Timoth{\'e}e Lacroix, Baptiste Rozi{\`e}re, Naman Goyal, Eric
  Hambro, Faisal Azhar, et~al.
\newblock Llama: Open and efficient foundation language models.
\newblock \emph{arXiv preprint arXiv:2302.13971}, 2023.

\bibitem[Ville(1939)]{ville1939}
Jean Ville.
\newblock \emph{{\'E}tude Critique de la Notion de Collectif}.
\newblock Gauthier-Villars, Paris, 1939.

\bibitem[Vovk and Wang(2021)]{vovkWang2021evalues}
Vladimir Vovk and Ruodu Wang.
\newblock E-values: Calibration, combination and applications.
\newblock \emph{The Annals of Statistics}, 49\penalty0 (3):\penalty0
  1736--1754, 2021.
\newblock \doi{10.1214/20-AOS2020}.

\bibitem[Wang et~al.(2025{\natexlab{a}})Wang, Wan, Sun, Chen, and
  Ar{\i}k]{wang2025dynscaling}
Fei Wang, Xingchen Wan, Ruoxi Sun, Jiefeng Chen, and Sercan~{\"O} Ar{\i}k.
\newblock Dynscaling: Efficient verifier-free inference scaling via dynamic and
  integrated sampling.
\newblock \emph{arXiv preprint arXiv:2506.16043}, 2025{\natexlab{a}}.

\bibitem[Wang et~al.(2026)Wang, Lin, Hu, Jiao, Chowdhury, Chang, and
  Patwardhan]{openai2025frontierscience}
Miles Wang, Robi Lin, Kat Hu, Joy Jiao, Neil Chowdhury, Ethan Chang, and Tejal
  Patwardhan.
\newblock Frontier{S}cience: {E}valuating {AI}'s {A}bility to {P}erform
  {E}xpert-{L}evel {S}cientific {T}asks.
\newblock \emph{arXiv preprint arXiv:2601.21165}, 2026.
\newblock URL \url{https://arxiv.org/abs/2601.21165}.

\bibitem[Wang et~al.(2025{\natexlab{b}})Wang, Feng, Li, Yuan, Zhang, Tan, Pan,
  Hu, and Li]{wang2025penny}
Xinglin Wang, Shaoxiong Feng, Yiwei Li, Peiwen Yuan, Yueqi Zhang, Chuyi Tan,
  Boyuan Pan, Yao Hu, and Kan Li.
\newblock Make every penny count: Difficulty-adaptive self-consistency for
  cost-efficient reasoning.
\newblock In \emph{Findings of the Association for Computational Linguistics:
  NAACL 2025}, pages 6904--6917, 2025{\natexlab{b}}.

\bibitem[Wang et~al.(2023)Wang, Wei, Schuurmans, Le, Chi, Narang, Chowdhery,
  and Zhou]{wang2023selfconsistency}
Xuezhi Wang, Jason Wei, Dale Schuurmans, Quoc~V Le, Ed~H. Chi, Sharan Narang,
  Aakanksha Chowdhery, and Denny Zhou.
\newblock Self-consistency improves chain of thought reasoning in language
  models.
\newblock In \emph{The Eleventh International Conference on Learning
  Representations}, 2023.
\newblock URL \url{https://openreview.net/forum?id=1PL1NIMMrw}.

\bibitem[Wei et~al.(2022)Wei, Wang, Schuurmans, Bosma, Xia, Chi, Le, Zhou,
  et~al.]{wei2022chain}
Jason Wei, Xuezhi Wang, Dale Schuurmans, Maarten Bosma, Fei Xia, Ed~Chi, Quoc~V
  Le, Denny Zhou, et~al.
\newblock Chain-of-thought prompting elicits reasoning in large language
  models.
\newblock \emph{Advances in neural information processing systems},
  35:\penalty0 24824--24837, 2022.

\bibitem[Yao et~al.(2023)Yao, Yu, Zhao, Shafran, Griffiths, Cao, and
  Narasimhan]{yao2023tree}
Shunyu Yao, Dian Yu, Jeffrey Zhao, Izhak Shafran, Tom Griffiths, Yuan Cao, and
  Karthik Narasimhan.
\newblock Tree of thoughts: Deliberate problem solving with large language
  models.
\newblock \emph{Advances in neural information processing systems},
  36:\penalty0 11809--11822, 2023.

\bibitem[Zhou et~al.(2023)Zhou, Sch{\"a}rli, Hou, Wei, Scales, Wang,
  Schuurmans, Cui, Bousquet, Le, et~al.]{zhou2023least}
Denny Zhou, Nathanael Sch{\"a}rli, Le~Hou, Jason Wei, Nathan Scales, Xuezhi
  Wang, Dale Schuurmans, Claire Cui, Olivier Bousquet, Quoc~V Le, et~al.
\newblock Least-to-most prompting enables complex reasoning in large language
  models.
\newblock In \emph{The Eleventh International Conference on Learning
  Representations}, 2023.

\end{thebibliography}
\end{document}